\newcommand{\mytitle}{
Learning Distributionally Robust First-Order Methods for Convex Optimization
}

\newenvironment{fittable}{\begin{adjustbox}{max width=\linewidth, max totalheight=0.85\textheight}}{\end{adjustbox}}

\newcommand{\codelink}{\url{https://github.com/stellatogrp/dro_pep}}

\documentclass[12pt]{article}

\usepackage{fullpage}

\let\oldparagraph\paragraph

\renewcommand{\paragraph}[1]{\oldparagraph{#1.}}

\usepackage{amsmath,amssymb,amsthm,amsfonts}
\usepackage{enumitem}
\usepackage{pgfplots,xcolor}
\usepgfplotslibrary{fillbetween,groupplots}
\pgfplotsset{compat=1.18}

\renewcommand\arraystretch{1.2}
\usepackage{algorithm}
\usepackage[noend]{algpseudocode} 

\usepackage{csvsimple}	\usepackage{booktabs} \usepackage{pgfplotstable}
\usepackage{etoolbox}
\AtBeginEnvironment{table}{\renewcommand\arraystretch{1.0}\small}
\usepackage{adjustbox}  

\usepackage[colorlinks=true, citecolor=cyan, linkcolor=cyan, urlcolor=cyan]{hyperref}
\usepackage[capitalize,noabbrev]{cleveref}

\usepackage{caption}
\newtheorem{theorem}{Theorem}
\newtheorem{lemma}{Lemma}
  \newtheorem{proposition}{Proposition}
\newtheorem{assumption}{Assumption}

\crefname{assumption}{Assumption}{Assumptions}
\Crefname{assumption}{Assumption}{Assumptions}
\crefname{theorem}{Theorem}{Theorems}
\Crefname{theorem}{Theorem}{Theorems}
\crefname{lemma}{Lemma}{Lemmas}
\Crefname{lemma}{Lemma}{Lemmas}
\crefname{proposition}{Proposition}{Propositions}
\Crefname{proposition}{Proposition}{Propositions}
\crefname{corollary}{Corollary}{Corollaries}
\Crefname{corollary}{Corollary}{Corollaries}

\theoremstyle{definition}
\newtheorem{definition}{Definition}

\theoremstyle{remark}
\newtheorem{remark}{Remark}
\newtheorem{example}{Example}
\newtheorem*{example*}{Example}

\allowdisplaybreaks

\usepackage[numbers,sort&compress]{natbib}
\newcommand{\eg}{{\it e.g.}}
\newcommand{\ie}{{\it i.e.}}

\newcommand{\reals}{{\mathbf R}}

\newcommand{\symm}{{\mathbf S}}         \newcommand{\ones}{{\mathbf 1}}                        \newcommand{\prob}{{\mathbf P}}         \newcommand{\probQ}{{\mathbf Q}}        \newcommand{\hprob}{\widehat{\prob}}          

\newcommand{\tpose}{T}                  \newcommand{\Tr}{\mathop{\bf tr}}       \newcommand{\diag}{\mathop{\bf diag}}   \newcommand{\Span}{\mathop{\bf span}}

\newcommand{\argmin}{\mathop{\rm argmin}}

\newcommand{\sign}{\mathop{\bf sign}}
\newcommand{\prox}{{\mathop{\bf prox}}}   

                    \newcommand{\diam}{\mathop{\bf diam}}            

\newcommand{\Expect}{\mathop{\bf E{}}}

\newcommand{\norm}[2][]{\left\| #2 \right\|_{#1}}

\newcommand{\risk}{\mathcal{R}}                         \newcommand{\suppset}{\Xi}              \newcommand{\probset}{\mathcal{Z}}      \newcommand{\supp}{\mathrm{supp}}                        \newcommand{\ambiset}{\mathcal{U}}            \newcommand{\data}{\widehat{\mathcal{D}}}

\newcommand{\hG}{\widehat{G}}

\newcommand{\hf}{\hat{f}}
\newcommand{\hF}{\widehat{F}}

\newcommand{\hx}{\hat{x}}

\newcommand{\hz}{\hat{z}}
\newcommand{\hZ}{\widehat{Z}}
\newcommand{\bZ}{\bar{Z}}
\newcommand{\tZ}{\widetilde{Z}}
\newcommand{\btheta}{\bar{\theta}}
  \newcommand{\xclass}{\mathcal{X}}       \newcommand{\fclass}{\mathcal{F}}                    \newcommand{\linear}{\mathcal{S}}

\newcommand{\Aobj}{A_{\mathrm{obj}}}
\newcommand{\bobj}{b_{\mathrm{obj}}}

\newcommand{\algo}{\mathcal{A}}

\begin{document}

 \title{\mytitle}
  \author{
    Vinit Ranjan\thanks{Sloan School of Management, MIT. Email: \texttt{vranjan1@mit.edu}.} \and
    Jisun Park\thanks{Department of Operations Research and Financial Engineering, Princeton University, and Research Institute of Mathematics, Seoul National University. Email: \texttt{jisunpark@princeton.edu}.} \and
    Bartolomeo Stellato\thanks{Department of Operations Research and Financial Engineering, Princeton University. Email: \texttt{bstellato@princeton.edu}.}
  }

\maketitle

\begin{abstract}
    We propose a distributionally robust approach for learning hyperparameters (\eg, step sizes and momentum coefficients) for first-order methods in convex optimization.
Given a dataset of problem instances, we minimize a Wasserstein distributionally robust version of the performance estimation problem~(PEP).
Our framework unifies two extremes.
    As the robustness radius vanishes, we recover the empirical risk minimization problem commonly encountered in learning to optimize~(L2O), and as it grows we recover the PEP worst-case algorithm design problem.
    We solve the resulting problem with stochastic gradient descent, differentiating through the solution of an inner semidefinite program at each step.
    We prove high-probability bounds showing that the true risk of the learned algorithm is at most the in-sample L2O optimum plus a slack that shrinks with the sample size, and is no worse than the worst-case PEP bound.
On logistic regression, LASSO, and linear programming examples, our learned algorithms achieve strong out-of-sample performance with certifiable robustness, outperforming both worst-case optimal and vanilla L2O baselines.
\end{abstract}

\section{Introduction}\label{sec:intro}

Applications in machine learning, signal processing, and optimal control increasingly rely on solving large-scale optimization problems using first-order methods for their low computational cost per iteration and modest memory requirements~\cite{beckFirstOrderMethodsOptimization2017,ryuLargeScaleConvexOptimization2022}.
Building on classical gradient descent~\cite{boydConvexOptimization2004}, modern variants of first-order methods apply proximal operators~\cite{parikhProximalAlgorithms2014} and operator splitting techniques~\cite{boydDistributedOptimizationStatistical2011,chambolleFirstOrderPrimalDualAlgorithm2011} to solve nonsmooth and constrained problems efficiently.
This has led to robust open-source solvers such as PDLP~\cite{applegatePracticalLargeScaleLinear2021,applegateFasterFirstorderPrimaldual2023} for linear programs (LPs), OSQP~\cite{stellatoOSQPOperatorSplitting2020a} for quadratic programs (QPs), and SCS~\cite{odonoghueConicOptimizationOperator2016,odonoghueOperatorSplittingHomogeneous2021} and COSMO~\cite{garstkaCOSMOConicOperator2019} for semidefinite programs (SDPs).
Yet a fundamental challenge persists: practical performance depends critically on hyperparameters like step sizes, and tuning them well remains largely a manual effort prone to trial and error.

Classical analysis derives worst-case convergence bounds over a given function class. 
The performance estimation problem~(PEP)~\cite{droriPerformanceFirstorderMethods2014,taylorSmoothStronglyConvex2017} automates this approach by solving semidefinite programs to obtain tight worst-case guarantees.
These bounds yield worst-case optimal constant step sizes~\cite{beckFirstOrderMethodsOptimization2017}, while recent work develops time-varying step sizes with improved worst-case rates~\cite{altschulerAccelerationstepsizeHedging2025}.
However, worst-case analysis is inherently pessimistic: it guards against pathological instances that rarely arise in practice, yielding conservative predictions of performance on typical problems.

Machine learning offers an alternative by exploiting structure in a dataset of problem instances.
Learning to optimize~(L2O)~\cite{andrychowiczLearningLearnGradient2016,liLearningOptimize2017,chenLearningOptimizePrimer2022} learns algorithm components from such data, spanning amortized optimization~\cite{amosTutorialAmortizedOptimization2023}, algorithm selection~\cite{riceAlgorithmSelectionProblem1976,kerschkeAutomatedAlgorithmSelection2019}, and algorithm configuration~\cite{hutterSequentialModelbasedOptimization2011,balcanDatadrivenAlgorithmDesign2021}, which tunes the hyperparameters of a fixed algorithmic structure.
All three minimize an average loss over the training instances, and often achieve substantial practical speedups over hand-tuned alternatives.
However, L2O provides no theoretical guarantees on out-of-distribution performance: learned algorithms may fail on problems outside the training distribution~\cite{chenLearningOptimizePrimer2022}.
Recent efforts enforce guarantees via nonlinear systems theory~\cite{martinLearningOptimizeConvergence2024,martinLearningOptimizeGuarantees2025,martinLearningAccelerateKrasnoselskiiMann2026} or statistical bounds, either PAC-Bayes~\cite{dziugaiteComputingNonvacuousGeneralization2017,sambharyaLearningWarmStartFixedPoint2024,suckerLearningtoOptimizePACBayesianGuarantees2025} or on held-out validation data~\cite{sambharyaLearningAlgorithmHyperparameters2026}.
The former reintroduces worst-case conservatism through dynamical systems analysis.
PAC-Bayes bounds certify either a Gibbs predictor or a fixed one with a margin penalty, and both kinds of statistical bound still target the training distribution.
Appendix~\ref{appendix:related} reviews these lines in detail.

In this work, we propose \emph{distributionally robust learning to optimize}~(DR-L2O), a framework that combines computer-assisted worst-case analysis with data-driven algorithm configuration.
Given a dataset of problem instances, we minimize the worst-case expected loss over algorithm parameters such as step sizes, where the worst case is taken over a Wasserstein ambiguity set centered at the empirical distribution.
This builds on the recent distributionally robust PEP of~\citet{parkDatadrivenAnalysisFirstOrder2025}, which evaluates this risk for a fixed algorithm.
We instead \emph{minimize} it, turning a performance certificate into a learning objective.
A single Wasserstein radius controls the trade-off between data-driven and worst-case design: as the radius vanishes we recover L2O; as it grows we recover worst-case optimal design via PEP.
We solve the resulting problem with stochastic gradient descent, where each iteration solves an inner SDP and differentiates through its solution via implicit differentiation of the KKT conditions (see Figure~\ref{fig:intro_losses} for trade-offs).
Our contributions are as follows:
\begin{enumerate}[itemsep=0pt]
    \item \textbf{DR-L2O framework.}
We formulate the distributionally robust learning to optimize problem~\eqref{prob:dr-l2o} for the scalar parameters of any first-order method satisfying the span condition of~\cref{assumption:first-order}, and show that it attains its minimum over the parameter set~(\cref{prop:attainment}).
        We solve it by stochastic gradient descent, differentiating through the inner SDP (Section~\ref{subsec:solution}).

    \item \textbf{Out-of-sample guarantees.}
        We prove that, for an appropriately chosen radius, the distributionally robust risk upper-bounds the true risk uniformly over algorithm parameters with high probability~(\cref{thm:finite_sample}).
        Combined with the interpolation result, the true risk of the learned algorithm is at most the in-sample L2O optimum plus a slack proportional to the radius, and at most the worst-case PEP bound~(\cref{thm:compare_risk}).

    \item \textbf{Numerical experiments.}
On logistic regression, LASSO, and linear programs, we show that DR-L2O learns algorithms with strong out-of-sample performance and certifiable robustness, outperforming both worst-case optimal and vanilla L2O approaches under distribution shift, while remaining competitive in-distribution.
\end{enumerate}

\begin{figure}
    \centering
    \includegraphics[width=0.8\linewidth]{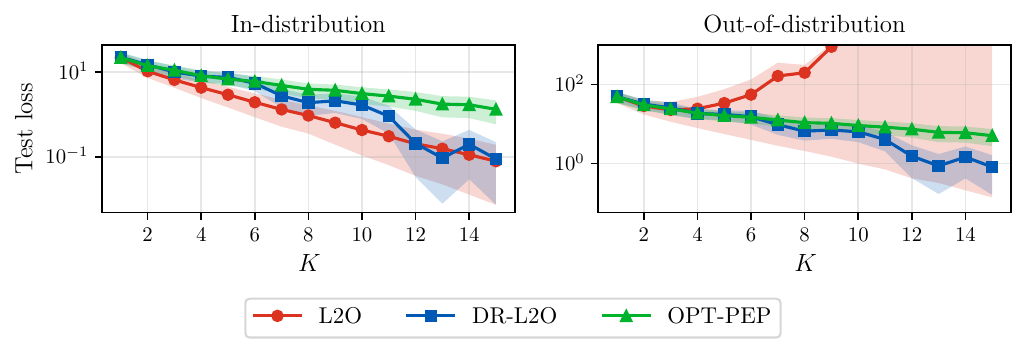}
    \caption{Schematic comparison of L2O~\eqref{prob:l2o}, the worst-case design~\eqref{prob:opt-pep}, and~\eqref{prob:dr-l2o} on a fixed problem class.
        Solid lines show mean loss across instances; shaded bands show the $10^\text{th}$ to $90^\text{th}$ quantiles.
        \emph{Left:} on the training distribution, L2O minimizes empirical loss, the worst-case design is conservative, and DR-L2O lies between (for most $K$).
        \emph{Right:} on an out-of-distribution test set, L2O degrades with high variance, the worst-case design remains slow but robust, and DR-L2O dominates both for higher $K$.
        Details of this experiment are provided in Section~\ref{subsec:lasso}.}
    \label{fig:intro_losses}
\end{figure}

\section{First-order algorithms and data-driven parameter tuning}\label{sec:problem_setup}
We consider a \emph{problem instance} $z = (f, x^0)$ consisting of a convex optimization problem to minimize $f(x)$,
where $f \colon \reals^d \to \reals \cup \{\infty\}$ is convex, proper, and lower semi-continuous, and $x^0 \in \reals^d$ is an initial point.
We assume that the problem has a solution $x^\star$ with optimal value $f^\star = f(x^\star)$.

\subsection{First-order algorithms}
\label{subsec:first-order}
An \emph{algorithm}~$\algo_\theta$ maps a problem instance to a sequence of approximate solutions $\algo_\theta(z) = \{x^k\}_{k=0,1,\ldots}$, where $\theta \in \Theta$ represents the algorithm parameters such as step sizes, also called hyperparameters in the learning-to-optimize literature.
We restrict attention to first-order algorithms, which use only gradient information of $f$.
\begin{assumption}[First-order algorithms]
    \label{assumption:first-order}
    For any $z = (f, x^0)$, the sequence $\{x^k\}_{k=0,1,\dots} = \algo_\theta (z)$ satisfies the \emph{span condition}
    \begin{equation*}
x^{k+1} \in x^0 + \Span \{ g^0, \dots, g^{k-1}, g^k \},
        \qquad k=0,1,\dots,
    \end{equation*}
    where~$g^i = \nabla f(x^i)$ for~$i=0,\dots,k$.
    We assume that there exist scalars~$\theta^k_i$, called the algorithm parameters, with~$x^{k+1} = x^k - \sum_{i=0}^{k} \theta^k_i g^i$ for~$k \ge 0$.
\end{assumption}

The span condition is the structural requirement of our framework, rather than a design choice, but still encompasses a broad class of first-order methods.
We impose the following assumptions on the instance set~$\probset$ and the parameter set~$\Theta$.

\begin{assumption}[Instance and parameter sets]
\label{assumption:compact-theta}
    \leavevmode
    \begin{enumerate}[label=(\roman*), itemsep=0pt]
        \item There exist constants~$L > 0$ and~$R_0 > 0$ such that, for every instance~$z = (f, x^0) \in \probset$ with solution~$x^\star$, the function~$f$ is $L$-smooth convex and $\|x^0 - x^\star\| \le R_0$.
        \item The parameter set~$\Theta$ is a box~$\Theta = \prod_{j=1}^p \, [l_j, u_j] \subset \reals^p$ with one interval per coordinate of~$\theta$ and $-\infty < l_j < u_j < \infty$ for every~$j = 1, \dots, p$.
    \end{enumerate}
\end{assumption}

We primarily focus on smooth unconstrained minimization in the main text.
    Our framework also applies to composite problems, which cover a broad range of convex optimization.
    \cref{appendix:composite} states the assumptions on the composite setting and the corresponding first-order methods.

\subsection{Data-driven parameter tuning}
\label{subsec:data-driven-design}
We measure algorithm performance on instance~$z$ by a \emph{loss function}~$\ell(\algo_\theta(z)) \in \reals_+$.
Given an iteration budget~$K$, common choices of~$\ell$ are the function-value gap $f(x^K) - f^\star$ and the distance to optimality $\|x^K - x^\star\|^2$, evaluated at the $K$-th iterate~$x^K$.

Let $\prob$ be a probability distribution over the set $\probset$ of problem instances.
Our goal is to \textbf{find the optimal algorithm parameter $\theta^\star$} minimizing the \emph{risk}, defined as the expected loss under $\prob$:
\begin{equation}\label{prob:risk_minimization}
    \begin{array}{ll}
        \underset{\theta \in \Theta}{\textnormal{minimize}} &
            \risk(\theta, \prob)
            :=
            \underset{z \sim \prob}{\Expect} \big( \ell(\algo_\theta(z)) \big).
    \end{array}
\end{equation}
Since the true distribution~$\prob$ is unknown, we approximate the risk in~\eqref{prob:risk_minimization} from a finite dataset.
Given i.i.d.\ samples~$\data_N = \{\hz_i\}_{i=1}^N \sim \prob^N$, the empirical distribution is~$\hprob_N = (1/N) \sum_{i=1}^N \delta_{\hz_i}$, where $\delta_z$ is the Dirac measure at $z \in \probset$.
The corresponding empirical risk is
\begin{equation*}
    \risk(\theta, \hprob_N)
    = (1/N) {\textstyle \sum_{i=1}^N} \ell \big( \algo_\theta(\hz_i) \big).
\end{equation*}
The risk minimization problem~\eqref{prob:risk_minimization} is the standard L2O formulation, which targets the expected loss under a single fixed but unknown distribution~$\prob$.
This does not guarantee robust performance of the algorithm with the optimal parameter~$\theta^\star$ on distributions other than the one that produced~$\data_N$.

\section{Performance guarantees under distribution shift and DR-L2O}\label{sec:perf_guarantee}
In this section, we highlight the differences between evaluating worst-case guarantees through the performance estimation problem (PEP)~\cite{droriPerformanceFirstorderMethods2014,taylorSmoothStronglyConvex2017} and its distributionally robust variant~\cite{parkDatadrivenAnalysisFirstOrder2025}.
We then define the distributionally robust learning to optimize (DR-L2O) framework and its solution method.

\subsection{Worst-case guarantee}\label{subsec:worst-case}
The worst-case loss~$\sup_{z \in \probset} \ell (\algo_\theta(z))$ upper-bounds the risk~$\risk(\theta, \prob)$ for any distribution $\prob$ supported on~$\probset$.
It is evaluated as the optimal value of the following performance estimation problem (PEP):
\begin{equation*}\label{prob:pep}
    \begin{array}[t]{ll}
        \textnormal{maximize} & \ell (\algo_\theta(z)) \\
        \textnormal{subject to}
        & z = (f, x^0) \in \probset
        \subseteq \fclass \times \xclass,
    \end{array}
    \tag{PEP}
\end{equation*}
where~$\fclass$ is the function class and~$\xclass$ is the set of initial iterates.
When the function class~$\fclass$ admits interpolation conditions, \eqref{prob:pep} admits a tractable formulation as an SDP via a lifted representation~\cite{droriPerformanceFirstorderMethods2014,taylorSmoothStronglyConvex2017,ryuOperatorSplittingPerformance2020}.
The function classes we consider are convex and~$L$-smooth~$\mu$-strongly convex functions~\cite[Theorems~1 and~4]{taylorSmoothStronglyConvex2017}, convex quadratic functions, and linear operators~\cite[Theorems~3.7 and~3.1]{bousselmiInterpolationConditionsLinear2024}.
Given iterates $\{x^k\}_{k=0}^K = \algo_\theta(f, x^0)$, define the Gram matrix
\begin{equation*}
    G = P^\tpose P \in \symm_+^{K+2} \text{ with } P = \begin{bmatrix} x^0 - x^\star & g^0 & \cdots & g^K \end{bmatrix} \in \reals^{d \times (K+2)}
    \text{ where } g^k = \nabla f(x^k),
\end{equation*}
and the function-value vector $F = (f(x^0) - f^\star, \dots, f(x^K) - f^\star) \in \reals^{K+1}$.
Each component of~\eqref{prob:pep} is linear in~$(G, F)$:
the loss is $\ell(G, F) = \Tr(\Aobj^\tpose G) + \bobj^\tpose F$,
the initial condition is $\Tr(A_0^\tpose G) + b_0^\tpose F + c_0 \le 0$,
and the interpolation constraints of~$\fclass$ are $\linear_\theta(G, F) \in \reals_+^M$ with
\begin{equation}\label{eq:lmi}
    \big(\linear_\theta (G, F)\big)_m = -\Tr\big(A_m(\theta)^\tpose G\big) - b_m(\theta)^\tpose F,
    \qquad m=1,\dots,M,
\end{equation}
where the LMI coefficients $A_m(\theta) \in \symm^{K+2}$ and $b_m(\theta) \in \reals^{K+1}$ depend on the algorithm parameters~$\theta$~\cite{kamriNumericalDesignOptimized2025}.
Denote by~$\suppset^\theta$ the feasible set of~$(G, F)$ pairs:
\begin{equation*}
    \suppset^\theta = \left\{
        (G, F) \in \symm_+^{K+2} \times \reals^{K+1}
        \left|\;
            \linear_\theta (G, F) \in \reals_+^M,\;
            \Tr(A_0^\tpose G) + b_0^\tpose F + c_0 \le 0
        \right.
    \right\}.
\end{equation*}
This set is closed and convex for each fixed~$\theta$.
Since every component of~\eqref{prob:pep} depends on the algorithm iterates only through their lifted representation~$(G, F)$, we identify the output~$\algo_\theta (z)$ with its lift and write $\algo_\theta (z) = (G, F) \in \suppset^\theta$.
Any norm applied to~$\algo_\theta(z)$ is then the pair norm $\|(G,F)\| = \sqrt{\|G\|_F^2 + \|F\|^2}$.
We write~$K_0$ for the number of columns of~$P$, which is~$K+2$ for the smooth lifting above and~$2K+3$ for the composite lifting of~\cref{appendix:composite}, and~$F$ has~$K_0 - 1$ entries in both cases.
These function classes admit \emph{exact} interpolation conditions, \ie, for every~$\theta \in \Theta$ and every~$(G, F) \in \suppset^\theta$ there exists~$z \in \probset$ with~$\algo_\theta (z) = (G, F)$, which we assume throughout and which holds when~$d \ge K_0$~\cite[Theorem~5]{taylorSmoothStronglyConvex2017}.
The explicit SDP formulation of~\eqref{prob:pep} and its Lagrangian dual are given in~\cref{appendix:pep}.

\subsection{Probabilistic guarantee from data}\label{subsec:data-driven}
With finitely many samples, the empirical risk~$\risk(\theta, \hprob_N)$ of~\cref{subsec:data-driven-design} may be a poor estimate of the true risk~$\risk(\theta, \prob)$.
To mitigate the bias of the empirical risk, \citet{parkDatadrivenAnalysisFirstOrder2025} propose a data-driven framework that yields a probabilistic performance guarantee robust to sampling variability.
Running $\algo_\theta$ on each sample $\hz_i \in \data_N$ yields a lifted representation $(\hG_i, \hF_i) \in \suppset^\theta$, inducing the lifted empirical distribution $\hprob_N^\theta = (1/N) \sum_{i=1}^N \delta_{(\hG_i, \hF_i)}$ over $\suppset^\theta$.
The guarantees are derived from a Wasserstein distributionally robust variant of the performance estimation problem~\eqref{prob:dro-pep}:
\begin{equation*}\label{prob:dro-pep}
    \!\risk_{\varepsilon} (\theta, \hprob_N) :=
    \begin{array}[t]{ll}
        \textnormal{maximize} & \risk(\theta, \probQ^\theta) \\
        \textnormal{subject to}
        & \probQ^\theta \in \ambiset_{\varepsilon} (\hprob_N^\theta),
    \end{array}
    \tag{DRO-PEP}
\end{equation*}
where the ambiguity set $\ambiset_{\varepsilon} (\hprob_N^\theta)$ is the $1$-Wasserstein ball of radius~$\varepsilon$ centered at~$\hprob_N^\theta$:
\begin{equation*}
    \ambiset_{\varepsilon} (\hprob_N^\theta)
    = \left\{
        \probQ^\theta \,\left|\;
            \supp\, \probQ^\theta \subseteq \suppset^\theta,\;
            W_1 (\hprob_N^\theta, \probQ^\theta) \le \varepsilon
    \right.\right\},
\end{equation*}
with $W_1$ induced by the norm~$\|(G, F)\| = \sqrt{\|{G}\|_F^2 + \|{F}\|^2}$.
We refer to $\risk_\varepsilon (\theta, \hprob_N)$ as the \emph{DRO risk}.

\paragraph{Tractable formulation}
By~\cite[Theorem~2]{parkDatadrivenAnalysisFirstOrder2025}, \eqref{prob:dro-pep} has a tractable formulation as a convex conic program; see~\cref{appendix:dro-pep-dual} for details.
For $\theta \in \Theta$ and $\varepsilon > 0$, \eqref{prob:dro-pep} also admits the following equivalent robust optimization formulation:
\begin{equation}\label{prob:dro-pep-primal}
    \begin{array}[t]{ll}
        \textnormal{maximize}
        &   (1/N) \sum_{i=1}^N
            \Tr(\Aobj^\tpose G_i) + \bobj^\tpose F_i \\
        \textnormal{subject to}
        &   (1/N) \sum_{i=1}^N \|{(G_i, F_i) - (\hG_i, \hF_i)}\| \le \varepsilon \\
        &   (G_i, F_i) \in \suppset^\theta,
            \hfill i=1,\dots,N.
    \end{array}
    \tag{DRO-PEP-P}
\end{equation}

The equivalence comes from the strong duality between the dual of~\eqref{prob:dro-pep} and its bi-dual, which is exactly \eqref{prob:dro-pep-primal}.
The proof of the strong duality result is deferred to \cref{appendix:dro-pep-strong-duality}.

\begin{proposition}\label{prop:strong-duality}
    Let $\theta \in \Theta$ and $\varepsilon > 0$.
    Then~\eqref{prob:dro-pep-primal} is a bi-dual of \eqref{prob:dro-pep}, and their optimal values coincide.
    Furthermore, \eqref{prob:dro-pep-primal} has an optimal solution.
\end{proposition}

\subsection{Distributionally robust learning to optimize}\label{subsec:dr-l2o}
We now minimize the DRO risk~$\risk_\varepsilon (\theta, \hprob_N)$ over~$\theta$ to learn optimal algorithm parameters, rather than evaluating it for a fixed~$\theta$ as in~\cite{parkDatadrivenAnalysisFirstOrder2025}.
When the learning objective is the worst-case loss, the corresponding learning problem is~\cite{kamriNumericalDesignOptimized2025}
\begin{equation}\label{prob:opt-pep}
    \begin{array}{ll}
        \underset{\theta \in \Theta}{\textnormal{minimize}} & \underset{z \in \probset}{\sup}\; \ell \big( \algo_\theta (z) \big).
    \end{array}
    \tag{OPT-PEP}
\end{equation}
However, worst-case optimal algorithms can underperform on typical instances.
At the other extreme, learning to optimize~(L2O) minimizes the empirical risk~\cite{chenLearningOptimizePrimer2022}:
\begin{equation}\label{prob:l2o}
    \begin{array}{ll}
        \underset{\theta \in \Theta}{\textnormal{minimize}} & \risk(\theta, \hprob_N)
        =
        (1/N) \sum_{i=1}^N \ell \big( \algo_\theta (\hz_i) \big).
    \end{array}
    \tag{L2O}
\end{equation}
When the loss~$\ell$ is differentiable, \eqref{prob:l2o} is solved efficiently with stochastic first-order methods.
However, it provides no guarantee under distribution shift: the learned algorithm~$\algo_{\theta^\star}$ may fail on instances from a different distribution.
To balance the two extremes, we minimize the DRO risk, yielding the learning problem we call \emph{distributionally robust learning to optimize} (DR-L2O):
\begin{equation}\label{prob:dr-l2o}
    \begin{array}{ll}
        \underset{\theta \in \Theta}{\textnormal{minimize}} & \risk_{\varepsilon} (\theta, \hprob_N).
    \end{array}
    \tag{DR-L2O}
\end{equation}
This learning problem admits a minimizer over~$\Theta$.
The proof of this fact is deferred to~\cref{appendix:well-posedness}.

\begin{proposition}[Existence of a minimizer]\label{prop:attainment}
    Fix~$\varepsilon > 0$.
    Suppose~\cref{assumption:first-order,assumption:compact-theta} hold, together with~\cref{assumption:composite} for proximal methods.
    Then~\eqref{prob:dr-l2o} attains its minimum over~$\Theta$.
\end{proposition}

\subsection{Solution method}\label{subsec:solution}
Problem~\eqref{prob:dr-l2o} is generally nonconvex in~$\theta$, since the entries of~$A_m(\theta)$ and~$b_m(\theta)$ are polynomial in~$\theta$ and those of~$\hG_i$ and~$\hF_i$ are nonlinear.
Still, the DRO risk~$\risk_{\varepsilon} (\theta, \hprob_N)$ is the optimal value of the convex conic program~\eqref{prob:dro-pep-dual} parametrized by~$\{ (A_m, b_m) \}_{m=1}^M$ and~$\{ (\hG_i, \hF_i) \}_{i=1}^N$, whose solution map admits a tractable derivative evaluation using the approach of~\cite{agrawalDifferentiatingConeProgram2019}.
We compute the gradient $d\risk_{\varepsilon} (\theta, \hprob_N) / d\theta$ via the chain rule and apply stochastic gradient descent to obtain~$\hat{\theta}_\varepsilon$, which may be only a local minimizer (see~\cref{appendix:solution}).
See~\cref{fig:solution_method} for a summary.

\begin{figure*}[t]
    \centering
    \includegraphics[width=.8\textwidth]{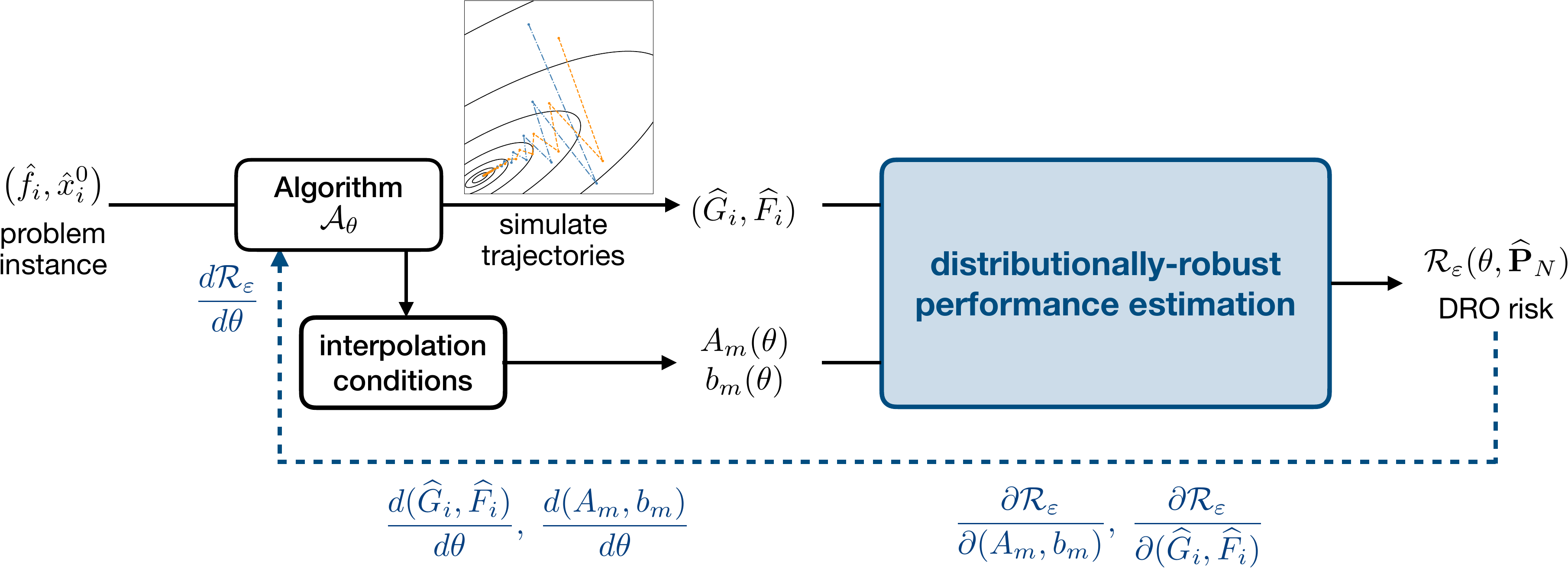}
    \caption{
        The diagram of the solution method for solving~\eqref{prob:dr-l2o} with stochastic gradient methods.
        Given~$\theta$, we have interpolation LMI coefficients~$A_m(\theta)$ and~$b_m(\theta)$, while a sample problem instance~$(\hf_i, \hx_i^0)$ is mapped to the Gram matrix and function-value representation~$(\hG_i, \hF_i) = \algo_\theta (\hf_i, \hx_i^0)$.
        We evaluate the gradient of the DRO risk $\risk_\varepsilon(\theta, \hprob_N)$ w.r.t.~$\theta$ through backpropagation.
    }
    \label{fig:solution_method}
\end{figure*}

\section{Theoretical guarantees of the learned optimizer}\label{sec:theory}
In this section, we show that~\eqref{prob:dr-l2o} not only interpolates between~\eqref{prob:opt-pep} and~\eqref{prob:l2o}, but also yields out-of-sample and out-of-distribution guarantees. 

\subsection{Robustness certificate}
We first show that our learned algorithm enjoys a robust out-of-sample guarantee, a consequence of the data-driven Wasserstein DRO framework~\cite{mohajerinesfahaniDatadrivenDistributionallyRobust2018}.
Here, we extend the finite-sample guarantee of~\cite[Section~3.2]{parkDatadrivenAnalysisFirstOrder2025} from a fixed~$\theta \in \Theta$ to a uniform statement over $\theta \in \Theta$.
The proof is provided in~\cref{appendix:finite_sample}.
\begin{theorem}\label{thm:finite_sample}
    Suppose~\cref{assumption:first-order,assumption:compact-theta} hold, together with~\cref{assumption:composite} for proximal methods.
    Let~$N$ and $K$ be fixed.
    For each~$\beta \in (0, 1)$,
    there exist $\overline{\varepsilon} (\beta) > 0$ and $\bar{L} > 0$ such that
    \begin{equation*}
        \prob^N \left(
\begin{array}[t]{l}
                \underset{z \sim \prob}{\Expect}
                \big( \ell (\algo_\theta (z)) \big)
            \end{array}
\le
                \begin{array}[t]{ll}
                    \underset{\probQ^\theta \in \ambiset_{\overline{\varepsilon} (\beta)} (\hprob_N^\theta)}{\sup}
                    & \underset{(G, F) \sim \probQ^\theta}{\Expect} \big( \ell (G, F) \big),
                \end{array}
            \quad \forall \theta \in \Theta
        \right) \ge 1 - \beta,
    \end{equation*}
    where~$q$ is the dimension of~$(G, F)$, namely~$q = K_0 (K_0+1)/2 + (K_0 - 1)$,
$\diam (\Theta) = \max_{\theta_1, \theta_2 \in \Theta} \|\theta_1 - \theta_2\|$ is the diameter of the parameter set,
    $\bar{L}$ is a Lipschitz constant of the algorithm mapping~$\theta \mapsto \algo_\theta$,
    and, for~$N$ sufficiently large,
    \begin{equation*}
        \overline{\varepsilon} (\beta)
        \lesssim
            \left(
                \frac{\log(1/\beta) + \dim(\Theta) \log \big( 1 + 2 N \diam (\Theta) \big)}{N}
            \right)^{1/q}
            + \frac{2 \bar{L}}{N}.
    \end{equation*}
\end{theorem}

\begin{remark}[Out-of-distribution (OOD) guarantee]\label{rem:ood_guarantee}
    The DRO risk upper-bounds the expected loss under every distribution in the ambiguity set, not only under the true distribution.
    The guarantee of~\cref{thm:finite_sample} therefore extends to test distributions that differ from the training distribution, as long as they are supported on the same class and remain within Wasserstein distance~$\overline{\varepsilon}(\beta)$ of~$\hprob_N^\theta$.
\end{remark}

\begin{remark}[Tightness of the Wasserstein radius]\label{rem:radius_tightness}
    Since $q = O(K^2)$, the bound on $\overline{\varepsilon}(\beta)$ decays as $N^{-1/q}$ and can be conservative for large~$K$, reflecting the Wasserstein curse of dimensionality~\cite{fournierRateConvergenceWasserstein2015}.
    It is a sufficient radius for the guarantee, whereas in our experiments the operating radius~$\varepsilon$ is selected by cross-validation (see~\cref{appendix:finite_sample} for sharper rates).
\end{remark}

\subsection{Interpolating behavior of (\ref{prob:dr-l2o})}
\label{subsec:interpolate}

We now show how~\eqref{prob:dr-l2o} interpolates between~\eqref{prob:l2o} and the worst-case optimal design~\eqref{prob:opt-pep} as the radius~$\varepsilon$ varies.
All proofs are presented in~\cref{appendix:interpolate}.

\begin{proposition}\label{prop:interpolate}
    Suppose~\cref{assumption:first-order,assumption:compact-theta} hold, together with~\cref{assumption:composite} for proximal methods.
    Let $\varepsilon > 0$,
    and let $\theta_{\varepsilon}^\star$ be an optimal solution of~\eqref{prob:dr-l2o}.
    Then the mapping~$\varepsilon \mapsto \risk_{\varepsilon} (\theta_{\varepsilon}^\star, \hprob_N)$ is monotonically nondecreasing.
    Furthermore,
    \begin{equation*}
        \lim_{\varepsilon \to 0^+} \risk_{\varepsilon} (\theta^\star_{\varepsilon}, \hprob_N)
            = \inf_{\theta \in \Theta} \risk (\theta, \hprob_N),
        \qquad
        \lim_{\varepsilon \to \infty} \risk_{\varepsilon} (\theta^\star_{\varepsilon}, \hprob_N)
            =   \underset{\theta \in \Theta}{\inf}\,
                \underset{z \in \probset}{\sup}\,
                \ell \big( \algo_{\theta} (z) \big).
    \end{equation*}
\end{proposition}

As~$\varepsilon \to 0^+$, problem~\eqref{prob:dr-l2o} loses its distributional robustness and reduces to~\eqref{prob:l2o}.
For large enough~$\varepsilon$, the optimizer learned via~\eqref{prob:dr-l2o} is robust to essentially all distributions supported on~$\suppset^{\theta}$, becoming worst-case optimal.
We now establish certified performance bounds for the true risk of our learned algorithm.

\begin{theorem}\label{thm:compare_risk}
    Suppose~\cref{assumption:first-order,assumption:compact-theta} hold, together with~\cref{assumption:composite} for proximal methods.
    Given~$\varepsilon > 0$,
    let~$\theta_{\varepsilon}^\star$ be an optimal solution of~\eqref{prob:dr-l2o}.
    For any~$\beta \in (0, 1)$ and~$\varepsilon = \overline{\varepsilon}(\beta)$ as in~\cref{thm:finite_sample}, the true risk~$\risk(\theta^\star_{\varepsilon}, \prob)$ of the optimizer learned via~\eqref{prob:dr-l2o} satisfies the following with probability at least~$1-\beta$:
    \begin{equation*}
            \risk (\theta^\star_{\varepsilon}, \prob)
            \le
                \inf_{\theta \in \Theta} \risk (\theta, \hprob_N)
                + \varepsilon\, \mathrm{Lip} (\ell),
        \qquad
            \risk (\theta^\star_{\varepsilon}, \prob)
            \le
                \inf_{\theta \in \Theta} \sup_{z \in \probset} \ell \big( \algo_{\theta} (z) \big),
    \end{equation*}
    where~$\mathrm{Lip}(\ell) = \|(\Aobj, \bobj)\|$ is a Lipschitz constant of the loss function~$\ell$.
\end{theorem}

The first bound places the true risk of the optimizer learned via~\eqref{prob:dr-l2o} within~$\varepsilon\,\mathrm{Lip}(\ell)$ of the in-sample optimum of~\eqref{prob:l2o}, uniformly over the ambiguity set.
The second bound shows that this true risk is no worse than the worst-case optimal value of~\eqref{prob:opt-pep}.
Neither bound implies strict empirical superiority over the baselines, which~\cref{sec:experiments} demonstrates in practice.

\section{Numerical experiments}\label{sec:experiments}
We now showcase the performance of our learned optimizers{} on logistic regression, LASSO, and total variation inpainting.
For each experiment, we report the out-of-sample test loss, the fraction of test instances solved at relative tolerance $\eta > 0$ (an instance $z = (f, x^0)$ is solved if $\ell(\algo_\theta(z)) \le \eta(1+|f^\star|)$), and out-of-distribution performance.
Implementation details and per-iteration wall-time tables are in Appendix~\ref{appendix:experiments}, while tables of all results are in Appendix~\ref{appendix:tables}.
We compare three learned optimizers with parameters $\theta_\varepsilon^\star$ from~\eqref{prob:dr-l2o}, $\theta_{\text{PEP}}^\star$ from~\eqref{prob:opt-pep}, and $\theta_{\text{L2O}}^\star$ from~\eqref{prob:l2o}.
All code to reproduce our experiments is available at 
\begin{center}
    \codelink.
\end{center}

\paragraph{Training and parameter selection}
We learn a schedule of step sizes (and, for logistic regression, momentum coefficients) $\theta = \{\theta^k\}_{k=0}^{K-1}$ for each $K$ (and learn a separate schedule for each $K$ value as the learning horizon) and each of~\eqref{prob:dr-l2o},~\eqref{prob:opt-pep}, and~\eqref{prob:l2o} using gradient-based methods.
We cross-validate the Wasserstein radius~$\varepsilon$, the learning rate, and the AdamW weight decay, and use a held-out test set for final loss evaluations.
Full details are in~\cref{appendix:solution}.

\subsection{Logistic regression}\label{subsec:logistic}
Consider an unregularized logistic regression problem
\begin{equation*}
    \begin{array}{ll}
        \underset{x}{\textnormal{minimize}} & (1/m) \sum_{i=1}^m \big( \log (1 + \exp(a_i^\tpose x)) - b_i a_i^\tpose x \big),
    \end{array}
\end{equation*}
where~$a_i^\tpose$ are the rows of the feature matrix~$A \in \reals^{m \times n}$ and~$b \in \{0, 1\}^m$ are the labels.
The objective is smooth and convex but not strongly convex.
We learn the step sizes and momentum coefficients~$\theta = \{(\tau^k, \rho^k)\}_{k=0}^{K-1} \in \reals^{2K}$ of Nesterov's fast gradient method (FGM)~\cite{nesterovMethodUnconstrainedConvex1983}:
\begin{equation*}
    x^{k+1} = y^k - \tau^k \nabla f(y^k), \qquad
    y^{k+1} = x^{k+1} + \rho^k (x^{k+1} - x^k),
\end{equation*}
with~$y^0 = x^0$.
We consider the performance loss $f(x^K) - f(x^\star)$.
Each instance is a uniform random subsample of the rows of the german.numer dataset from LIBSVM~\cite{changLIBSVMLibrarySupport2011}; the out-of-distribution shift scales the feature matrix, which increases the Lipschitz constant~$L$ (full parameters in Appendix~\ref{appendix:logistic}).

\paragraph{Results}
In Figure~\ref{fig:logreg_frac_solved} we show the fractions of problems solved by the learned schedules, and the losses are provided in Figure~\ref{fig:logreg_losses} (Appendix~\ref{appendix:logistic}).
In-distribution, we notice that~\eqref{prob:l2o} solves the most problems at every tolerance and~\eqref{prob:dr-l2o} lies between~\eqref{prob:l2o} and~\eqref{prob:opt-pep}.
Out-of-distribution, the~\eqref{prob:l2o} fraction solved at $\eta = 10^{-2}$ peaks at a small $K$ and then falls, and at the tighter tolerances it stays near zero for all $K$.
The~\eqref{prob:l2o} average loss does not decrease with $K$, so running more learned iterations degrades the out-of-distribution performance.
In contrast, both~\eqref{prob:dr-l2o} and~\eqref{prob:opt-pep} improve with $K$ overall, and~\eqref{prob:dr-l2o} solves every out-of-distribution instance at $\eta = 10^{-3}$ and solves the largest fraction at $\eta = 10^{-4}$, whereas~\eqref{prob:opt-pep} does not within the horizon.

\begin{figure}
    \centering
    \includegraphics[width=0.9\linewidth]{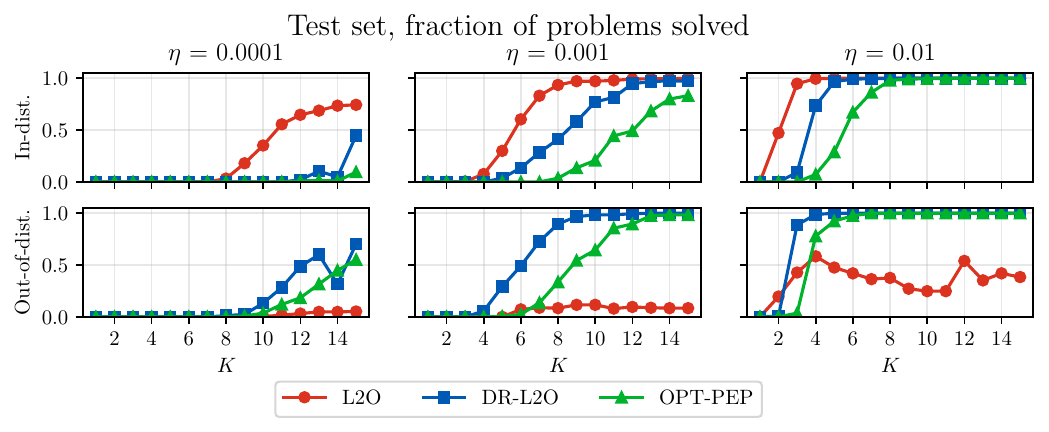}
    \caption{Fractions of logistic regression problems solved to different relative tolerances across horizon $K$. \emph{Top:} Fraction solved in-distribution. \emph{Bottom:} Fraction solved out-of-distribution.}
    \label{fig:logreg_frac_solved}
\end{figure}

\subsection{LASSO}\label{subsec:lasso}
Consider an~$\ell_1$-regularized least-squares problem:
\begin{equation*}
    \begin{array}{ll}
        \underset{x}{\textnormal{minimize}} & (1/2) \| Ax - b \|^2 + \lambda \|x\|_1,
    \end{array}
\end{equation*}
with regularization parameter~$\lambda>0$.
We learn the step sizes~$\{\theta^k\}_{k=0}^{K-1}$ of ISTA~\cite{beckFirstOrderMethodsOptimization2017}:
\begin{equation*}
    x^{k+1} = {\prox}_{\lambda \theta^k \|\cdot\|_1} \big(
        x^k - \theta^k A^\tpose (A x^k - b)
    \big),
\end{equation*}
where~$\theta = (\theta^k)_{k=0}^{K-1} \in \reals_+^K$.
We again consider the performance loss $f(x^K) - f(x^\star)$.

Specialized learned optimizers exist for LASSO, all based on empirical risk minimization.
Learned ISTA (LISTA)~\cite{gregorLearningFastApproximations2010} learns weight matrices to speed up the proximal updates over a distribution of instances.
For comparison, we label our default~\eqref{prob:l2o} scheme as L2O-ISTA and compare against analytic LISTA (ALISTA)~\cite{chenTheoreticalLinearConvergence2018,liuALISTAAnalyticWeights2018}, which matches the performance of LISTA with far fewer learned parameters and much cheaper training.
ALISTA restricts the structure of LISTA by making the update rule $x^{k+1} = {\prox}_{\gamma^k \|\cdot\|_1} \big( x^k - \theta^{k} W^\tpose (A x^k - b) \big)$, where $W$ is precomputed as the solution of a convex quadratic program~\cite[Equation 16]{liuALISTAAnalyticWeights2018}.
We label this baseline L2O-ALISTA.
We use a sparse coding setup, recovering sparse vectors $\tilde{x}$ from noisy measurements $b = A\tilde{x} + \delta$ with a fixed dictionary $A\in\reals^{m \times n}$~\cite{chenTheoreticalLinearConvergence2018}; the out-of-distribution shift increases the sparse-signal variance $\sigma_x$ (full parameters in Appendix~\ref{appendix:lasso}).

\paragraph{Results} We previewed the loss results of this experiment in Figure~\ref{fig:intro_losses}, and the full version with L2O-ALISTA is provided in Figure~\ref{fig:lasso_losses} (Appendix~\ref{appendix:lasso}), while the fractions of problems solved are provided in Figure~\ref{fig:lasso_frac_solved}.
Overall, we notice the performance of L2O-ISTA degrades out of distribution significantly.
Given that the $10^\text{th}$ quantile is so low and the schedule is able to solve many of the out-of-distribution problems, this implies that the schedule diverges on a small number of instances, while our~\eqref{prob:dr-l2o} learned schedule is robust against such outliers.

\begin{figure}
    \centering
    \includegraphics[width=0.9\linewidth]{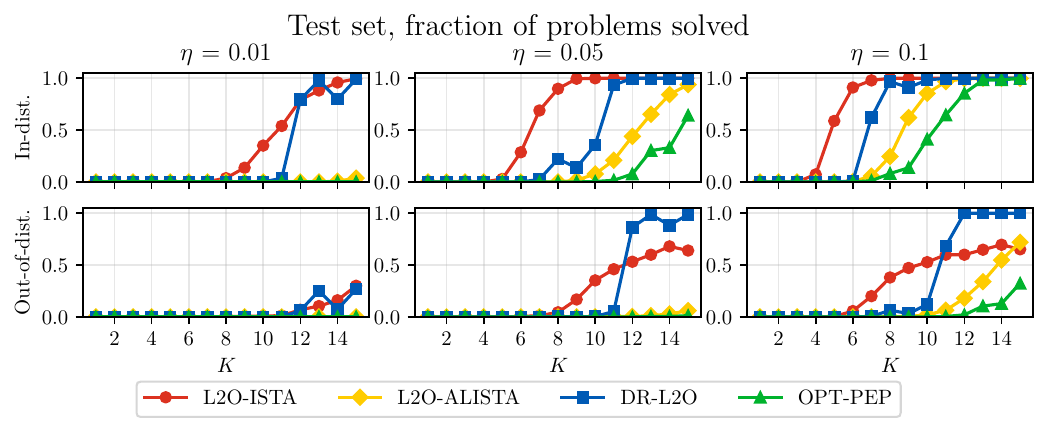}
    \caption{Fractions of LASSO problems solved to different relative tolerances across horizon $K$. \emph{Top:} Fraction solved in-distribution. \emph{Bottom:} Fraction solved out-of-distribution.}
    \label{fig:lasso_frac_solved}
\end{figure}

\subsection{Total variation (TV) inpainting}\label{subsec:lp}
We consider the $\ell_1$ total variation (TV) inpainting problem, which fills in an image that has been corrupted with some pixels blacked out~\cite{cvxpy_tv_inpainting}.
For a grayscale image represented as an $m \times n$ matrix $U^\text{orig}$ with values in $[0, 255]$ (scaled down to $[0,1]$ for algorithm stability) and the set $\mathcal{K}$ of known pixel locations, the goal is to reconstruct the image $U$ while matching the known pixels.
The corresponding linear program is
\begin{equation*}
    \begin{array}{llll}
        \underset{U}{\textnormal{minimize}} & \displaystyle \sum_{i=1}^{m-1} \sum_{j=1}^{n-1} \norm[1]{\begin{bmatrix}
            U_{i+1, j} - U_{i, j} \\ U_{i, j+1} - U_{i, j}
        \end{bmatrix}} &
        \textnormal{subject to} & U_{ij} = U^\text{orig}_{ij}, \quad (i,j) \in \mathcal{K} \\
        &&& U_{ij} \in [0, 1], \quad 1 \le i \le m,\, 1 \le j \le n.
    \end{array}
\end{equation*}

We solve this LP using primal-dual hybrid gradient (PDHG)~\cite{chambolleFirstOrderPrimalDualAlgorithm2011} adjusted for linear programs (PDLP)~\cite{applegateFasterFirstorderPrimaldual2023}.
We first cast the problem in standard LP form by introducing slack variables for the $\ell_1$ terms (see~\cref{appendix:lp}), and apply PDHG to its convex-concave saddle reformulation
\begin{equation*}
    \underset{x}{\textnormal{minimize}}~\underset{u}{\textnormal{maximize}}~~L(x, u) = f(x) + \langle u, Mx \rangle - g^*(u),
\end{equation*}
where $f$ encodes the LP's linear cost and box constraints, $M$ stacks the equality and inequality constraint matrices, and $g(Mx)$ is the indicator of the constraints; see~\cref{appendix:lp} for the explicit form and the corresponding PDHG iteration, whose step sizes $\theta = \{(\tau^k, \rho^k, \sigma^k)\}_{k=0}^{K-1} \in \reals_+^{3K}$ we learn.
For the performance loss, we use the Lagrangian duality gap $L(x, u^\star) - L(x^\star, u)$, and since this class falls outside the standard~\eqref{prob:pep} interpolation conditions, we extend them to general linear operators in~\cref{appendix:extended_dual_dro_pep} (see~\cref{rem:pdhg-extension} for guarantees).
For in-distribution we use Olivetti faces~\cite{samaria1994Faces,olivettifaces} with 10\% of pixels blacked out and for out-of-distribution we use color images from Tiny ImageNet~\cite{tinyimagenet,deng2009imagenet} (full parameters in Appendix~\ref{appendix:lp}).

\paragraph{Results}
We show both the fractions of problems solved and the losses in Figures~\ref{fig:pdlp_frac_solved} and~\ref{fig:pdlp_losses} in Appendix~\ref{appendix:lp}.
Here, in Figure~\ref{fig:pdlp_reconstructions}, we instead show the reconstructions of images from the two datasets with the different learned schedules.
The Olivetti reconstruction is not significantly different between~\eqref{prob:l2o} and~\eqref{prob:dr-l2o}.
However, the Tiny ImageNet reconstruction shows a difference where the learned schedules increasingly bring the image into focus and lessen the gray hue.
In Figure~\ref{fig:pdlp_more_reconstructions} (Appendix~\ref{appendix:lp}), we provide more reconstructions on extra images from Tiny ImageNet.

\begin{figure}
    \centering
    \includegraphics[width=\linewidth]{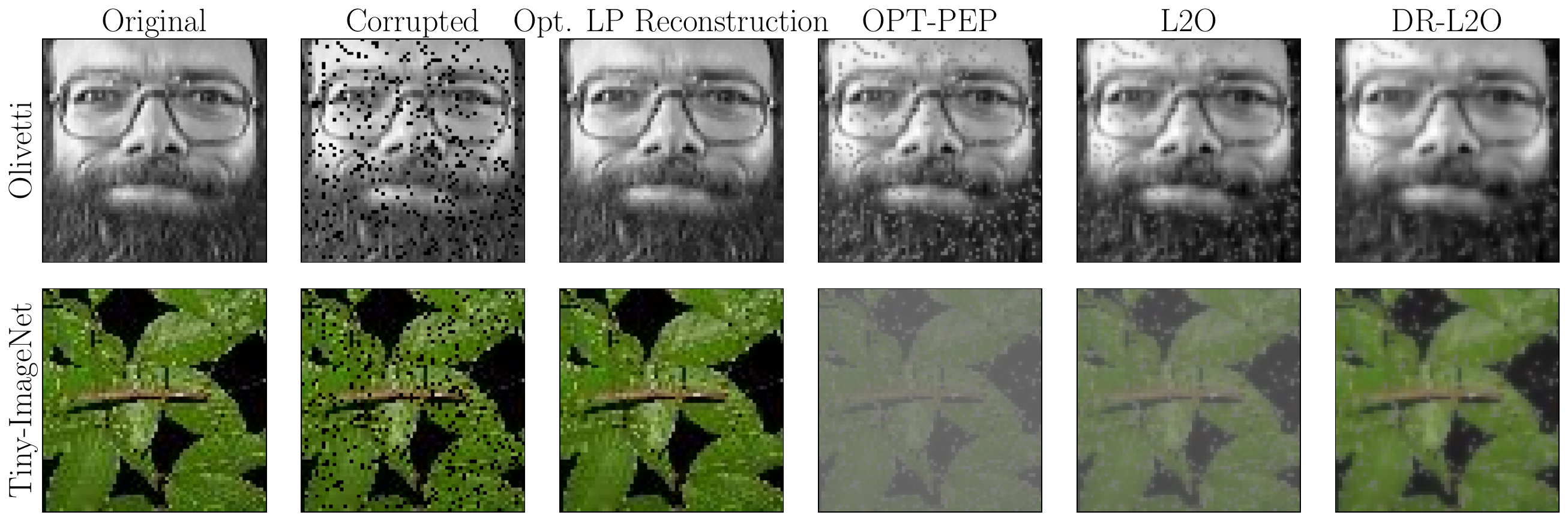}
    \caption{TV inpainting reconstructions ($K=10$). \emph{Left to right:} original, corrupted (10\% missing pixels), optimal LP reconstruction, and reconstructions from~\eqref{prob:opt-pep},~\eqref{prob:l2o}, and~\eqref{prob:dr-l2o}.}
    \label{fig:pdlp_reconstructions}
\end{figure}

\section{Conclusion}\label{sec:conclusion}
We presented a distributionally robust algorithm design framework for first-order methods.
By minimizing the distributionally robust risk, our learned algorithms enjoy out-of-sample guarantees and a robustness to distribution shift that empirical risk minimization alone does not provide.
The framework extends naturally to risk-averse training with other risk measures, such as conditional value-at-risk~\cite{uryasevConditionalValueatRiskOptimization2001}; the corresponding DRO-PEP formulation already appears in~\cite[Corollary~2.1]{parkDatadrivenAnalysisFirstOrder2025}.
A major limitation is scalability: unrolling algorithms for more iterations requires solving a much larger conic program at each gradient step.
A natural remedy is to use GPU-accelerated first-order solvers~\cite{luCuPDLPjlGPUImplementation2025,linPracticalGPUEnhancedMatrixFree2026,luCuPDLPxFurtherEnhanced2025,chenHPRLPImplementationHPR2025,chenHPRQPDualHalpern2025}.
A second limitation is scope, since only the scalar span coefficients of~\cref{assumption:first-order} receive a gradient, which leaves learned solution maps, learned projection and preconditioner matrices, and algorithm selection outside the framework.
PEP also covers more than our framework learns, since adaptive rules such as exact line search~\cite{deklerkWorstcaseComplexityGradient2017}, inexact gradients~\cite{deklerkWorstcaseConvergenceAnalysis2020}, and Polyak step sizes~\cite{barreComplexityGuaranteesPolyak2020} admit PEP representations but fix their step sizes from the iterates, leaving no free scalar coefficient.
We therefore do not claim that fixed learned schedules displace adaptive methods, line search, or silver step sizes~\cite{altschulerAccelerationstepsizeHedging2025} in general-purpose solving.

\newcommand{\myack}{Bartolomeo Stellato and Vinit Ranjan are supported by the NSF CAREER Award ECCS-2239771 and the ONR YIP Award N000142512147.
Jisun Park is supported by the National Research Foundation of Korea (NRF) grant funded by the Korea government (MSIT) (RS-2024-00353014) and the ONR YIP Award N000142512147.
The authors are pleased to acknowledge that the work reported on in this paper was substantially performed using Princeton University's Research Computing resources.}
\ifdefined\isaccepted \section*{Acknowledgements} \myack \fi
\ifdefined\ispreprint \section*{Acknowledgements} \myack \fi
 \newpage \section*{Acknowledgements} \myack

\newpage
\appendix
\section{Related works}\label{appendix:related}

\subsection{Worst-case analysis and its limitations}
The performance estimation problem~(PEP)~\cite{droriPerformanceFirstorderMethods2014,taylorSmoothStronglyConvex2017} and control-theoretic analysis~\cite{lessardAnalysisDesignOptimization2016} automate worst-case convergence analysis for first-order methods in convex optimization, while verification frameworks~\cite{ranjanVerificationFirstorderMethods2025,ranjanExactVerificationFirstOrder2025} extend these ideas to parametric quadratic and linear optimization.
These approaches have enabled the discovery of worst-case optimal algorithms, including the optimized gradient method~(OGM)~\cite{kimOptimizedFirstorderMethods2016}, OptISTA~\cite{jangComputerAssistedDesignAccelerated2026}, and the accelerated proximal point method~(APPM)~\cite{kimAcceleratedProximalPoint2021}, as well as accelerated gradient methods based on silver step-size schedules~\cite{dasguptaBranchandboundPerformanceEstimation2024,altschulerAccelerationstepsizeHedging2025}.
\citet{kamriNumericalDesignOptimized2025} design such methods numerically by differentiating the PEP worst-case bound with respect to the step sizes.
However, worst-case bounds are often far from typical performance.
Average-case analysis~\cite{smaleAverageNumberSteps1983,xiongHighProbabilityPolynomialTimeComplexity2025} offers tighter characterizations when problem distributions are known, while data-driven evaluation~\cite{sambharyaDataDrivenPerformanceGuarantees2025,huangDataDrivenPerformanceGuarantees2025,parkDatadrivenAnalysisFirstOrder2025} provides performance bounds on observed instances.
Yet worst-case design optimizes hyperparameters only for robustness, while average-case and data-driven analyses only evaluate the performance of a given method.

\subsection{Learning to optimize~(L2O)}
L2O~\cite{andrychowiczLearningLearnGradient2016,liLearningOptimize2017,chenLearningOptimizePrimer2022} learns algorithm components from a dataset of problem instances.
Its paradigms include amortized optimization~\cite{amosTutorialAmortizedOptimization2023}, algorithm selection~\cite{riceAlgorithmSelectionProblem1976,kerschkeAutomatedAlgorithmSelection2019}, and algorithm configuration~\cite{hutterSequentialModelbasedOptimization2011,balcanDatadrivenAlgorithmDesign2021}, the last of which learns hyperparameters within a fixed algorithmic structure and is the setting of this paper.
This has achieved practical success, from learned proximal methods that recover sparse signals in only a few iterations~\cite{liuALISTAAnalyticWeights2018} to scalable approaches for neural network training~\cite{chenScalableLearningOptimize2022}.
However, L2O is essentially empirical risk minimization, which alone does not ensure robustness to a new distribution.

Two strategies have emerged to address this gap.
The first imposes structural constraints that ensure convergence: \citet{chenTheoreticalLinearConvergence2018} identified algorithm structures that are both convergent and trainable, \citet{premont-schwarzSimpleGuardLearned2022} and \citet{heatonSafeguardedLearnedConvex2023} developed safeguarded formulations guaranteeing convergence for every instance, \citet{liuConstitutingMathematicalStructures2023} derived necessary conditions for convergent learned algorithms, and \citet{martinLearningOptimizeConvergence2024,martinLearningOptimizeGuarantees2025} developed parametrizations of learned algorithms with convergence guarantees,
including a complete characterization of convergence-preserving modifications for several classes of linearly convergent algorithms.
The second strategy incorporates distribution-aware objectives.
Some improve generalization through meta-learning~\cite{yangML2OGeneralizableLearningtoOptimize2022} or data alignment for out-of-distribution robustness~\cite{songRobustLearningOptimize2024}.
Others certify the learned algorithm, either by regularizing the learning objective with worst-case convergence rates~\cite{sambharyaLearningAccelerationAlgorithms2025} or through statistical bounds from PAC-Bayes~\cite{sambharyaLearningWarmStartFixedPoint2024,suckerLearningtoOptimizePACBayesianGuarantees2025} and held-out validation data~\cite{sambharyaLearningAlgorithmHyperparameters2026}.

However, these strategies have limitations.
Structural constraints reduce learning capacity, while distribution-aware objectives either reintroduce worst-case conservatism or certify a quantity other than the one deployed.
The PAC-Bayes bound of~\citet{sambharyaLearningWarmStartFixedPoint2024} holds for the fixed learned weights, at the price of a margin term that takes the worst case over a perturbation of the warm start.
The PAC-Bayes guarantee of~\citet{suckerLearningtoOptimizePACBayesianGuarantees2025} bounds the loss of a Gibbs predictor, namely hyperparameters drawn afresh from a learned posterior at every instance, whereas one deploys a single fixed choice.
More fundamentally, these bounds and the validation bound of~\citet{sambharyaLearningAlgorithmHyperparameters2026} certify performance only under the single distribution that generated the training instances, which is the target we replace by a Wasserstein ball around the empirical distribution.

A parallel line in data-driven algorithm design certifies empirical risk minimization itself in the standard PAC sense, bounding the gap between empirical and expected performance uniformly over the hyperparameters with pseudo-dimension~\cite{sakaueGeneralizationBoundLearning2024,nguyenProvablyDatadrivenProjection2026} or fat-shattering~\cite{bartlettGeneralizationBoundsDatadriven2022} arguments and Rademacher complexity with matching minimax lower bounds~\cite{leProvablyDatadrivenLagrangian2026}.
As with the bounds above, those works take the expectation at the training distribution.
This line therefore offers complementary frameworks rather than a bound established for our L2O baseline.
\citet{leProvablyDatadrivenLagrangian2026} name shifts in the problem distribution an important open problem, which is the problem we take up.

\subsection{Distributionally robust optimization~(DRO)}
DRO~\cite{kuhnDistributionallyRobustOptimization2025} handles uncertainty by optimizing over ambiguity sets, which are collections of distributions consistent with partial information.
When only data samples are available, data-driven DRO~\cite{wiesemannDistributionallyRobustConvex2014,kuhnDistributionallyRobustOptimization2025} constructs ambiguity sets as metric balls centered at the empirical distribution, using Wasserstein distance~\cite{mohajerinesfahaniDatadrivenDistributionallyRobust2018,kuhnWassersteinDistributionallyRobust2019} or divergences such as Kullback--Leibler~(KL) divergence~\cite{hongKullbackLeiblerDivergenceConstrained2012}.
This framework naturally quantifies distributional uncertainty from finite samples, providing the mechanism we exploit to interpolate between L2O and worst-case PEP.

\section{Additional details and proofs for~Section~\ref{sec:perf_guarantee}}
\label{appendix:perf_guarantee}

\paragraph{Risk on lifted distributions}
For a distribution~$\probQ$ over~$\probset$, let~$\probQ^\theta$ denote its pushforward under~$\algo_\theta$, which is a distribution over~$\suppset^\theta$.
The loss depends on~$z$ only through the lift~$\algo_\theta(z) = (G, F)$, so the risk~\eqref{prob:risk_minimization} satisfies~$\risk(\theta, \probQ) = \Expect_{(G, F) \sim \probQ^\theta} \big( \ell(G, F) \big)$.
We use the right-hand side to define~$\risk(\theta, \probQ^\theta)$ for any distribution~$\probQ^\theta$ supported on~$\suppset^\theta$, which is how~$\risk$ enters~\eqref{prob:dro-pep}.
The Wasserstein ball~$\ambiset_\varepsilon(\hprob_N^\theta)$ therefore lives in the lifted space~$\suppset^\theta$, which depends on the algorithm through~$\theta$.
The empirical distribution~$\hprob_N^\theta$ is the pushforward of~$\hprob_N$, so~$\risk_\varepsilon(\theta, \hprob_N)$ is a function of~$\hprob_N$ through~$\hprob_N^\theta$.

\subsection{Tractable convex conic formulation of (\ref{prob:pep})}\label{appendix:pep}

The performance estimation problem~\eqref{prob:pep} admits the following tractable formulation:
\begin{align*}
    \begin{array}[t]{ll}
        \underset{z \in \probset}{\sup}
        & \ell(\algo_\theta (z))
    \end{array}
    &=
    \begin{array}[t]{ll}
        \textnormal{maximize}
        &   \ell(G, F) \left( := \Tr (\Aobj^\tpose G) + \bobj^\tpose F \right) \\
        \textnormal{subject to}
        &   \linear_\theta (G, F) \in \reals_+^M \\
        &   \Tr (A_0^\tpose G) + b_0^\tpose F + c_0 \le 0 \\
        &   G \in \symm_+^{K+2},\, F \in \reals^{K+1}.
    \end{array}
\end{align*}
One instance of the initial condition is the bound of~\cref{assumption:compact-theta}($i$).
The choice $A_0 = e_1 e_1^\tpose$, $b_0 = 0$, and~$c_0 = -R_0^2$, with~$e_1$ being the first standard basis vector of~$\reals^{K+2}$, makes~$\Tr(A_0^\tpose G) + b_0^\tpose F + c_0 \le 0$ equivalent to~$\|x^0 - x^\star\|^2 - R_0^2 \le 0$.
The Lagrangian dual of the problem above writes as
\begin{align*}
    \begin{array}[t]{ll}
        \textnormal{minimize}
        &   - c_0 \tau \\
        \textnormal{subject to}
        &   - \linear^*_\theta(y) + \tau (A_0, b_0) - (\Aobj, \bobj) \in \symm_+^{K+2} \times \{0\} \\
        &   \tau \ge 0,\, y \in \reals_+^M,
    \end{array}
\end{align*}
where~$\linear^*_\theta \colon \reals^M \to \symm^{K+2} \times \reals^{K+1}$ is an adjoint mapping of~$\linear_\theta$ defined as~$\linear^*_\theta (y) = - \sum_{m=1}^M y_m (A_m, b_m)$.
For most first-order methods, strong duality holds true from~\cite[Theorem~6]{taylorSmoothStronglyConvex2017}.

\subsection{Tractable form of~(\ref{prob:dro-pep}) with expected loss}\label{appendix:dro-pep-dual}
The problem~\eqref{prob:dro-pep} with expected loss admits a tractable convex conic formulation~\cite{parkDatadrivenAnalysisFirstOrder2025}:
\begin{equation}\label{prob:dro-pep-dual}
    \begin{array}{l}
        \begin{array}[t]{ll}
            \textnormal{minimize}
            &   (1/N) \sum_{i=1}^N s_i \\
            \textnormal{subject to}
            &   - c_0 \tau_i - \big\langle (X_i, Y_i),\, (\hG_i, \hF_i) \big\rangle + \lambda \varepsilon \le s_i \\
            &   - \linear^*_\theta (y_i) - (X_i, Y_i) + \tau_i (A_0, b_0) - (\Aobj, \bobj) \in \symm_+^{K+2} \times \{0\} \\
            &   \norm{(X_i, Y_i)} \le \lambda,
                \hfill i=1,\dots,N,
        \end{array}
        \tag{DRO-PEP-D}
    \end{array}
\end{equation}
with optimization variables~$s_i \in \reals$, $\tau_i \in \reals_+$, $(X_i, Y_i) \in \symm^{K+2} \times \reals^{K+1}$, $y_i \in \reals_+^M$ for $i=1,\dots,N$, and~$\lambda \in \reals_+$,
where~$\linear^*_\theta$ is an adjoint mapping of~$\linear_\theta$.

\subsection{Composite minimization}
\label{appendix:composite}

In this subsection, we provide a detailed treatment of the \emph{composite minimization} setting, where the objective function is split into a smooth part~$f$ and a nonsmooth part~$h$:
\begin{equation}
    \label{prob:composite}
    \begin{array}{ll}
        \underset{x}{\textnormal{minimize}} & f(x) + h(x),
    \end{array}
\end{equation}
where~$f$ is~$L$-smooth convex as in~\cref{assumption:compact-theta}($i$) and $h \colon \reals^d \to \reals \cup \{\infty\}$ is convex, proper, and lower semi-continuous.
The set of instances~$\probset$ is now a set of triples~$(f, h, x^0)$, where~$x^0$ is the initial point.

Given that~\eqref{prob:composite} attains its minimum, for any solution~$x^\star$ we have $0 \in \nabla f(x^\star) + \partial h(x^\star)$, where~$\partial h(x^\star)$ is the subdifferential of~$h$ at~$x^\star$.
The gradient~$\nabla f(x^\star)$ therefore need not vanish, which is in contrast to the unconstrained minimization setting where~$h \equiv 0$ and~$\nabla f(x^\star) = 0$.

\paragraph{Proximal-gradient methods}
We consider the class of proximal-gradient methods, which is defined by the sequence of iterations generated by the following update rules:
\begin{equation}
    \label{algo:prox-grad}
    \begin{array}{rcl}
        y^k &=& x^k - \sum_{i=0}^{k-1} \phi^k_i \nabla f(y^i) - \sum_{i=0}^{k-1} \psi^k_i s^i, \\
        x^{k+1} &=& \prox_{\theta^k_\mathrm{prox} h} \left( y^k - \sum_{i=0}^{k} \theta^k_i \nabla f(y^i) \right),
    \end{array}
\end{equation}
for~$k=0,\dots,K-1$.
This algorithm uses two oracle calls: the gradient of smooth~$f$ and the proximal operator of nonsmooth~$h$.
Denoting by~$w^k$ the intermediate point
\begin{equation*}
    w^k = y^k - \sum_{i=0}^{k} \theta^k_i \nabla f(y^i),
    \qquad k=0,1,\dots,K-1,
\end{equation*}
at which the proximal operator is evaluated, we may extract the subgradient~$s^k$ as
\begin{equation}
    \label{eq:prox-subgradient}
    s^k = \frac{w^k - x^{k+1}}{\theta^k_\mathrm{prox}} \in \partial h(x^{k+1}),
    \qquad k=0,1,\dots,K-1.
\end{equation}
The empty sums make~$y^0 = x^0$.
The parameter~$\theta$ collects the scalars~$\phi^k_i$, $\psi^k_i$, $\theta^k_i$, and~$\theta^k_\mathrm{prox}$, so every coefficient in~\eqref{algo:prox-grad} is a coordinate of~$\theta$.

Note that this proximal-gradient method definition includes the standard proximal-gradient method or ISTA, and also its accelerated variant FISTA~\cite{beckFirstOrderMethodsOptimization2017}.
Setting~$\phi^k_i$ and~$\psi^k_i$ to zero gives the single-sequence iteration~$y^k = x^k$, which is ISTA when~$\theta^k_k = \theta^k_\mathrm{prox}$ and the other~$\theta^k_i$ are zero.
The FISTA iteration extrapolates~$y^k = x^k + \beta_k (x^k - x^{k-1})$ before querying the gradient~\cite[Chapter~10]{beckFirstOrderMethodsOptimization2017}, and the difference~$x^k - x^{k-1}$ is a linear combination of the oracle outputs already produced, so this~$y^k$ is also of the form~\eqref{algo:prox-grad}.

\begin{assumption}[Proximal parameter and composite instances]
\label{assumption:composite}
    Suppose that the algorithm~$\algo_\theta$ is of the form~\eqref{algo:prox-grad}.
    \begin{enumerate}[label=(\roman*)]
        \item There exist constants~$0 < \theta_{\min} \le \theta_{\max} < \infty$ such that the proximal parameters satisfy~$\theta^k_\mathrm{prox} \in [\theta_{\min}, \theta_{\max}]$ for every~$\theta \in \Theta$ and every~$k = 0, \dots, K-1$.
        \item There exists a constant~$R_\nabla$ such that
        \begin{equation}\label{eq:composite-gradient-bound}
            \norm{\nabla f(x^\star)} \le R_\nabla
        \end{equation}
        for every composite instance~$z = (f, h, x^0) \in \probset$ with solution~$x^\star$.
    \end{enumerate}
\end{assumption}

Part~($i$) gives a uniform lower bound on the proximal parameters, leaving the gradient step sizes free as in~\cref{assumption:compact-theta}($ii$).
The proximal-residual bound needed below follows from part~($ii$) rather than requiring an independent assumption.
Since~$x^\star$ solves~\eqref{prob:composite}, we have~$-\nabla f(x^\star) \in \partial h(x^\star)$, or equivalently, $\prox_{\eta h} \big( x^\star - \eta \nabla f(x^\star) \big) = x^\star$ for every~$\eta > 0$.
Nonexpansiveness of the proximal operator therefore gives
\begin{equation}\label{eq:prox-residual-bound}
    \norm{x^\star - \prox_{\eta h} (x^\star)}
    \le \eta \norm{\nabla f(x^\star)}
    \le \theta_{\max} R_\nabla
    =: R_\mathrm{prox},
    \qquad
    \eta \in [\theta_{\min}, \theta_{\max}].
\end{equation}

\begin{remark}[Necessity of the proximal lower bound]\label{rem:prox-lower-bound}
    The lower bound of~\cref{assumption:composite}($i$) is not an artifact of the arguments below.
    Take~$h = |\cdot|$ on~$\reals$, whose proximal operator is scalar soft thresholding.
    The subgradient~\eqref{eq:prox-subgradient} contributed by a proximal step at input~$v$ and parameter~$\alpha$ is~$\sign(v) \min \{ |v|/\alpha, 1 \}$.
    The derivative of this expression in~$\alpha$ is~$-v/\alpha^2$ on~$|v| < \alpha$, whose supremum over~$v$ grows like~$1/\alpha$.
    This subgradient is therefore not Lipschitz in~$\theta$ uniformly over instances as the proximal parameter approaches zero.
\end{remark}

\begin{remark}\label{rem:instance-constants}
    One way to verify~\cref{assumption:composite}($ii$) is to compare the composite solution with a minimizer of the smooth part.
    Suppose in addition that~$f$ attains its minimum at~$x_f^\star$.
    Then
    \begin{equation*}
        \norm{\nabla f(x^\star)}
        = \norm{\nabla f(x^\star) - \nabla f(x_f^\star)}
        \le L \norm{x^\star - x_f^\star},
    \end{equation*}
    by~\cref{assumption:compact-theta}($i$), so it suffices that one can choose~$x_f^\star$ within a uniformly bounded distance of~$x^\star$ over~$z \in \probset$.
\end{remark}

We now define the lifted representation~$(G, F)$ of a composite instance.

\begin{definition}[Lifted representation of a composite instance]\label{def:composite-lift}
    Let~$z = (f, h, x^0) \in \probset$ be a composite instance with solution~$x^\star$, let~$\{x^k\}_{k=0}^K = \algo_\theta (z)$ be the first~$K+1$ iterates, and let~$\{y^k\}_{k=0}^{K-1}$ be the gradient-query points of~\eqref{algo:prox-grad}.
    The \emph{lifted representation} of~$z$ under~$\algo_\theta$ is the pair~$(G, F) \in \symm_+^{2K+3} \times \reals^{2K+2}$, where~$G = P^\tpose P$ is the Gram matrix of
    \begin{equation}\label{eq:P-columns}
        P = \begin{bmatrix}
            x^0 - x^\star &
            \nabla f(y^0) & \cdots & \nabla f(y^{K-1}) & \nabla f(x^K) &
            s^\star & s^0 & \cdots & s^{K-1}
        \end{bmatrix}
        \in \reals^{d \times (2K+3)},
    \end{equation}
    with~$s^\star = -\nabla f(x^\star) \in \partial h(x^\star)$ the optimality column and~$s^k$ the subgradient~\eqref{eq:prox-subgradient} contributed by the~$k$-th proximal step, and where
    \begin{equation}\label{eq:F-columns}
        \begin{aligned}
            F = \big(
                &f(y^0) - (f+h)^\star,\; \ldots,\; f(y^{K-1}) - (f+h)^\star,\;
                f(x^K) - (f+h)^\star,\;
            \\
                &h(x^\star),\;
                h(x^1) - (f+h)^\star,\; \ldots,\; h(x^K) - (f+h)^\star
            \big) \in \reals^{2K+2},
        \end{aligned}
    \end{equation}
    with~$(f+h)^\star = f(x^\star) + h(x^\star)$ the optimal value of~\eqref{prob:composite}, normalized to~$f(x^\star) = h(x^\star) = 0$ so that the entry~$h(x^\star)$ vanishes.
    This is without loss of generality, since replacing~$(f, h)$ by~$(f - f(x^\star), h - h(x^\star))$ changes no iterate, oracle output, or gap loss, and without it~$\suppset^\theta$ is unbounded along the shift~$(f - c, h + c)$.
\end{definition}

The smooth part is interpolated over the gradient-query points~$y^0, \dots, y^{K-1}$ together with the final iterate~$x^K$ and the solution~$x^\star$, and the nonsmooth part over~$x^1, \dots, x^K$ and~$x^\star$.
The matrix~$P$ therefore has~$K_0 = 2K+3$ columns here, compared with~$K_0 = K+2$ for the smooth lifting of~\cref{subsec:worst-case}, and~$F$ has~$K_0 - 1$ entries in both cases.
In both cases~$P$ carries one column per interpolation point together with the initial-distance column~$x^0 - x^\star$, while~$F$ carries one entry per interpolation point, each value measured relative to the optimal value~$(f+h)^\star$ of the sum.
Smooth minimization is the case~$h \equiv 0$, where the optimality column, the proximal columns, and the~$h$-value entries are all absent, so that~\eqref{eq:P-columns} and~\eqref{eq:F-columns} reduce to the smooth lifting.
Both the interpolation conditions and the loss of~\eqref{prob:pep} are linear in this pair, so the tractable formulations of~\cref{appendix:pep} apply with~$K_0 = 2K+3$ in place of~$K+2$ and~$F \in \reals^{2K+2}$ in place of~$F \in \reals^{K+1}$~\cite[Appendix~A]{parkDatadrivenAnalysisFirstOrder2025}.
The support set~$\suppset^\theta$ also contains the equality~$F_{K+2} = 0$ and the bound~\eqref{eq:composite-gradient-bound} as~$G_{K+3, K+3} \le R_\nabla^2$, each handled like the initial condition.
Our implementation omits the latter bound, which enlarges the support set, keeps the certificate of~\cref{thm:finite_sample} valid, and is needed only for the compactness of~\cref{lem:supp-compact} and the large-radius limit of~\cref{prop:interpolate}.

Neither the iterates nor the gradient-query points are themselves columns of~$P$.
The update rules~\eqref{algo:prox-grad} write each of them as an affine combination of the columns of~\eqref{eq:P-columns}.
Only the~$K+1$ gradient columns and the~$K$ proximal columns depend on~$\theta$, since~$x^0$, $x^\star$, and~$s^\star$ belong to the instance~$z$.

\subsection{Proof of Proposition~\ref{prop:strong-duality}}
\label{appendix:dro-pep-strong-duality}

\begin{proof}[Proof of \cref{prop:strong-duality}]
    The dual derivation in~\cite[Theorem~2]{parkDatadrivenAnalysisFirstOrder2025} shows that~\eqref{prob:dro-pep-primal} is a dual of~\eqref{prob:dro-pep-dual}, and hence a bi-dual of~\eqref{prob:dro-pep}.
    It remains to prove that~\eqref{prob:dro-pep} and~\eqref{prob:dro-pep-primal} have the same optimal value and that the latter attains its optimum.

    Write~$Z = (G,F)$ and~$\widehat Z_i = (\hG_i,\hF_i)$, and denote the linear loss by
    \begin{equation*}
        \ell(Z) = \Tr(\Aobj^\tpose G) + \bobj^\tpose F.
    \end{equation*}
    Recall that~$\suppset^\theta$ is closed and convex and that~$\widehat Z_i \in \suppset^\theta$ for every~$i=1,\dots,N$.

    First, let~$\probQ^\theta \in \ambiset_\varepsilon(\hprob_N^\theta)$ be feasible for~\eqref{prob:dro-pep}.
    Since the first marginal is the empirical distribution~$\hprob_N^\theta = (1/N)\sum_{i=1}^N \delta_{\widehat Z_i}$, an optimal coupling between~$\hprob_N^\theta$ and~$\probQ^\theta$ can be written as
    \begin{equation*}
        \pi
        = \frac{1}{N} \sum_{i=1}^N \delta_{\widehat Z_i} \otimes \probQ_i,
        \qquad
        \probQ^\theta = \frac{1}{N} \sum_{i=1}^N \probQ_i,
    \end{equation*}
    where each~$\probQ_i$ is supported on~$\suppset^\theta$.
    Define the barycenter
    \begin{equation*}
        \overline Z_i = \int_{\suppset^\theta} Z \, d\probQ_i(Z).
    \end{equation*}
    Since~$\suppset^\theta$ is closed and convex, the point~$\overline Z_i$ belongs to~$\suppset^\theta$.
    Jensen's inequality gives
    \begin{align*}
        \frac{1}{N} \sum_{i=1}^N \|\overline Z_i - \widehat Z_i\|
        &\leq
        \frac{1}{N} \sum_{i=1}^N
        \int_{\suppset^\theta} \|Z - \widehat Z_i\| \, d\probQ_i(Z) \\
        &= W_1(\hprob_N^\theta,\probQ^\theta)
        \leq \varepsilon.
    \end{align*}
    Thus,~$(\overline Z_1,\dots,\overline Z_N)$ is feasible for~\eqref{prob:dro-pep-primal}.
    Since~$\ell$ is linear,
    \begin{align*}
        \underset{Z \sim \probQ^\theta}{\Expect} \big( \ell(Z) \big)
        &= \frac{1}{N} \sum_{i=1}^N
        \int_{\suppset^\theta} \ell(Z) \, d\probQ_i(Z) \\
        &= \frac{1}{N} \sum_{i=1}^N \ell(\overline Z_i).
    \end{align*}
    Every feasible distribution for~\eqref{prob:dro-pep} therefore induces a feasible solution of~\eqref{prob:dro-pep-primal} with the same objective value, so
    \begin{equation}\label{eq:dro-pep-upper-by-primal}
        \operatorname{val}(\text{DRO-PEP})
        \leq
        \operatorname{val}(\text{DRO-PEP-P}).
    \end{equation}

    Conversely, let~$(Z_1,\dots,Z_N)$ be feasible for~\eqref{prob:dro-pep-primal}, and define
    \begin{equation*}
        \probQ^\theta = \frac{1}{N} \sum_{i=1}^N \delta_{Z_i}.
    \end{equation*}
    The coupling~$\pi = (1/N)\sum_{i=1}^N \delta_{(\widehat Z_i,Z_i)}$ satisfies
    \begin{equation*}
        W_1(\hprob_N^\theta,\probQ^\theta)
        \leq
        \frac{1}{N} \sum_{i=1}^N \|Z_i-\widehat Z_i\|
        \leq \varepsilon.
    \end{equation*}
    Thus,~$\probQ^\theta$ is feasible for~\eqref{prob:dro-pep}, and
    \begin{equation*}
        \underset{Z \sim \probQ^\theta}{\Expect} \big( \ell(Z) \big)
        = \frac{1}{N} \sum_{i=1}^N \ell(Z_i).
    \end{equation*}
    This construction gives the reverse inequality to~\eqref{eq:dro-pep-upper-by-primal}, and hence the two optimal values coincide.

    It remains to prove attainment.
    Every feasible point of~\eqref{prob:dro-pep-primal} satisfies
    \begin{equation*}
        \|Z_i - \widehat Z_i\| \leq N\varepsilon,
        \qquad i=1,\dots,N.
    \end{equation*}
    Its feasible set is therefore a closed subset of
    \begin{equation*}
        \prod_{i=1}^N
        \left(
            \suppset^\theta
            \cap
            \{Z \mid \|Z-\widehat Z_i\| \leq N\varepsilon\}
        \right),
    \end{equation*}
    which is compact because~$\suppset^\theta$ is closed and the lifted space is finite-dimensional.
    The feasible set is nonempty because~$Z_i = \widehat Z_i$ for every~$i$ is feasible.
    The objective is continuous, so~\eqref{prob:dro-pep-primal} attains its optimal value.
\end{proof}

\subsection{Extended dual formulation of (DRO-PEP) with linear operator interpolations}
\label{appendix:extended_dual_dro_pep}

Consider the case where the support set~$\suppset$ is defined with an additional positive semidefinite constraint as follows.
\begin{equation*}
    \suppset = \left\{
        (G, F) \in \symm^{K+2}_+ \times \reals^{K+1}
        \left|
        \begin{array}{l}
            \Tr(A_0^\tpose G) + b_0^\tpose F + c_0 \le 0 \\
            \Tr(A_m^\tpose G) + b_m^\tpose F \le 0,\;
                m=1,\dots,M, \\
            H \in \symm_+^{K+1}
                \text{ where }
                (H)_{kl} = \Tr (C_{kl}^\tpose G) + d_{kl}^\tpose F
        \end{array}
        \right.
    \right\}.
\end{equation*}
In this subsection, $\suppset$ denotes a support set of this form.
The classes of linear operators with bounded spectral norm and convex quadratic functions are such cases, as the following example illustrates.

\begin{example}[Interpolation condition of linear operators]
\label{eq:linear_supp}
    We state~\cite[Theorem~3.1]{bousselmiInterpolationConditionsLinear2024}, which provides interpolation conditions for linear operators, along with convex quadratic functions as a special case.
    Let~$X \in \reals^{n \times N_1}$, $Y \in \reals^{m \times N_1}$, $U \in \reals^{m \times N_2}$, $V \in \reals^{n \times N_2}$ and $L \ge 0$.
    Then there exists~$M \in \reals^{m \times n}$ with $\sigma_{\max} (M) \le L$ such that
    \begin{equation*}
        Y = M X,
        \qquad
        V = M^\tpose U,
    \end{equation*}
    if and only if
    \begin{equation*}
        X^\tpose V = Y^\tpose U,
        \qquad
            Y^\tpose Y \preceq L^2 X^\tpose X,
        \qquad
            V^\tpose V \preceq L^2 U^\tpose U.
    \end{equation*}
    Note that the first condition can be written as $\Tr(A_m^\tpose G) + b_m^\tpose F \le 0$ and $- \Tr(A_m^\tpose G) - b_m^\tpose F \le 0$ by writing each entry of the matrix~$X^\tpose V - Y^\tpose U$ as the left-hand side of these inequalities.
    For the second constraint, let~$H = L^2 X^\tpose X - Y^\tpose Y$ and $H_2 = L^2 U^\tpose U - V^\tpose V$.
    Then $H \succeq 0$ and each entry of $H$ is a linear combination of bilinear terms of $(X, Y)$.
    Therefore, it is of the form
    \begin{equation*}
        (H)_{kl} = \Tr(C_{kl}^\tpose G) + d_{kl}^\tpose F,
        \qquad 1 \le k,l \le K+1,
    \end{equation*}
    where~$C_{kl} \in \symm^{K+2}$ and~$d_{kl} \in \reals^{K+1}$.
    We can do similarly for the last constraint.

    The conditions above can be specialized to a symmetric positive semidefinite matrix~$Q$ such that~$\mu I \preceq Q \preceq L I$,
    in the sense that $g^k = Q x^k$ and $f^k = (1/2) (x^k)^\tpose Q x^k$ for $k=0,\dots,K$ if and only if,
    for $X = (x^0, \dots, x^K) \in \reals^{d \times (K+1)}$, $P = (g^0, \dots, g^K) \in \reals^{d \times (K+1)}$, and $F = (f^0, \dots, f^K) \in \reals^{K+1}$,
    $(X, P, F)$ is an element of the following set:
    \begin{equation}\label{eq:quad_supp}
        \mathcal{Q}_{\mu, L} = \left\{
            (X, P, F) \in \reals^{d \times (K+1)} \times \reals^{d \times (K+1)} \times \reals^{K+1}
            \left|
                \begin{array}{ll}
                    X^\tpose P = P^\tpose X \\
                    (P - \mu X)^\tpose (LX - P) \succeq 0 \\
                    f^i = \tfrac{1}{2} (x^i)^\tpose g^i, \; i=0,\dots,K
                \end{array}
            \right.
        \right\}.
    \end{equation}
\end{example}

\begin{remark}[Extension to PDHG]\label{rem:pdhg-extension}
    The PDHG iteration of~\cref{appendix:lp} is not of the form~\eqref{algo:prox-grad}, but the proofs of~\cref{lem:supp-compact,lem:algo-map-lipschitz} carry over as follows, with the lifted representation of the Chambolle--Pock method in~\cite[Section~4.2]{bousselmiInterpolationConditionsLinear2024}.
    Suppose~$\|M\|_2 \le M_{\max}$ and~$\|(M^\tpose u^\star, M x^\star)\| \le R_\nabla$ for every instance.
    The oracles~$M$ and~$M^\tpose$ are then Lipschitz with constant~$M_{\max}$ and are applied to affine combinations of past columns, so~$M_{\max}$ replaces~$L$.
    The two proximal maps are the blocks of the proximal map of the separable function~$(x, u) \mapsto f(x) + g^*(u)$ with block parameters~$(\tau^k, \sigma^k)$, so~$\min \{\tau_{\min}, \sigma_{\min}\}$ replaces~$\theta_{\min}$ and the extrapolation~$\bar{x}^{k+1}$ plays the role of the second sequence.
    The functions~$f$ and~$g^*$ are convex, and~$\|(x^0 - x^\star, u^0 - u^\star)\| \le R_0$ replaces the initial condition.
    The interpolation conditions of~\cref{eq:linear_supp} combined with those of~$f$ and~$g^*$ are exact when the dimensions~$n$ and~$m$ are large enough, since~\cite[Theorem~3.1]{bousselmiInterpolationConditionsLinear2024} is an equivalence and the convex interpolation of~$f$ and~$g^*$ constrains only their own triples~\cite[Theorem~4]{taylorSmoothStronglyConvex2017}.
    Therefore, both lemmas hold for PDHG with~$(\tau, \rho, \sigma)$ in a box with~$\tau_{\min}, \sigma_{\min} > 0$, and so do~\cref{thm:finite_sample,prop:attainment,prop:interpolate,thm:compare_risk}.
\end{remark}

\begin{lemma}
    Let~$\suppset$ be of the form above.
    The worst-case expected loss
    \begin{equation*}
        \begin{array}{ll}
            \textnormal{maximize}
            &   \underset{z \sim \probQ}{\Expect} \big( \ell(\algo_\theta(z)) \big) \\
            \textnormal{subject to}
            &   \probQ \in \ambiset_{\varepsilon} (\hprob_N), 
            \quad
                \mathrm{supp}\, \probQ \subseteq \suppset,
        \end{array}
    \end{equation*}
    is upper-bounded by the optimal value of the problem
    \begin{equation*}
        \begin{array}{ll}
            \textnormal{minimize} & (1/N) \sum_{i=1}^N s_i \\
            \textnormal{subject to}
            &   - c_0 \tau_i
                - \Tr(X_i^\tpose \hG_i)
                - Y_i^\tpose \hF_i + \lambda \varepsilon \le s_i \\
            &   \tau_i A_0
                + \sum_{m=1}^M (y_i)_m A_m
                - \sum_{k,l=1}^{K+1} (\widetilde{H}_i)_{kl} C_{kl}
                - \Aobj - X_i
                \in \symm_+^{K+2} \\
            &   \tau_i b_0
                + \sum_{m=1}^M (y_i)_m b_m
                - \sum_{k,l=1}^{K+1} (\widetilde{H}_i)_{kl} d_{kl}
                - \bobj - Y_i
                = 0 \\
            &   \norm{(X_i, Y_i)} \le \lambda,
                \hfill i=1,\dots,N,
        \end{array}
    \end{equation*}
    with variables~$s_i \in \reals$, $\tau_i \in \reals_+$, $y_i \in \reals_+^M$, $X_i \in \symm^{K+2}$, $Y_i \in \reals^{K+1}$, $\widetilde{H}_i \in \symm_+^{K+1}$, and~$\lambda \in \reals_+$.
    Equality holds if~$\suppset$ contains a point with~$G \succ 0$ and~$H \succ 0$.
\end{lemma}

\begin{proof}
    The worst-case expected loss given can be written as
    \begin{equation*}
        \begin{array}[t]{ll}
            \underset{\probQ \in \ambiset_{\varepsilon} (\hprob_N)}{\sup}
            &   \underset{z \sim \probQ}{\Expect} 
                \big( \ell (\algo_\theta(z)) \big)
        \end{array}
        =
        \begin{array}[t]{ll}
            \underset{\probQ^\theta \in \ambiset_{\varepsilon} (\hprob_N^\theta)}{\sup} & \underset{(G, F) \sim \probQ^\theta}{\Expect} \big( \ell (G, F) \big).
        \end{array}
    \end{equation*}
    Circling back to the proof of~\cite[Theorem~2]{parkDatadrivenAnalysisFirstOrder2025},
    this is equivalent to
    \begin{align*}
        \begin{array}{ll}
            \textnormal{minimize} & (1/N) \sum_{i=1}^N s_i \\
            \textnormal{subject to}
            &   (- \ell)^* \big( (X_i, Y_i) - (U_i, V_i) \big)
                + \sigma_{\suppset} (U_i, V_i)
                - \Tr(X_i^\tpose \hG_i) - Y_i^\tpose \hF_i
                + \lambda \varepsilon \le s_i \\
            &   \norm{(X_i, Y_i)} \le \lambda,
            \hfill i=1,\dots,N,
        \end{array}
    \end{align*}
    with variables~$s_i \in \reals$, $\lambda \in \reals_+$, $X_i, U_i \in \symm^{K+2}$, and~$Y_i, V_i \in \reals^{K+1}$ for $i=1,\dots,N$.
    First of all,
    \begin{align*}
        (- \ell)^* (X, Y)
        &=
            \sup_{(G, F) \in \symm^{K+2} \times \reals^{K+1}} \left(
                \Tr (X^\tpose G) + Y^\tpose F
                + \ell (G, F)
            \right)
        \\&=
            \sup_{(G, F) \in \symm^{K+2} \times \reals^{K+1}} \left(
                \Tr \big( (X + \Aobj)^\tpose G \big)
                + (Y + \bobj)^\tpose F
            \right)
        \\&=
            \begin{cases}
                0 & \text{if } (X, Y) + (\Aobj, \bobj) = 0 \\
                \infty & \text{otherwise}.
            \end{cases}
    \end{align*}
    Furthermore, the support function~$\sigma_\suppset$ of $\suppset$ is
    \begin{align*}
        \sigma_{\suppset}(U, V)
        &=
        \begin{array}[t]{ll}
            \sup & \Tr(U^\tpose G) + V^\tpose F \\
            \mathrm{s.t.}
            &   (G, F) \in \suppset
        \end{array}
        \\&=
        \begin{array}[t]{ll}
            \sup & \Tr(U^\tpose G) + V^\tpose F \\
            \mathrm{s.t.}
            &   \Tr(A_0^\tpose G) + b_0^\tpose F + c_0 \le 0 \\
            &   \Tr(A_m^\tpose G) + b_m^\tpose F \le 0,
                \quad m=1,\dots,M \\
            &   H \succeq 0, \quad (H)_{kl} = \Tr(C_{kl}^\tpose G) + d_{kl}^\tpose F \\
            &   (G, F) \in \symm_+^{K+2} \times \reals^{K+1}
        \end{array}
        \\&\le
        \begin{array}[t]{ll}
            \inf
            &   - c_0 \tau \\
            \mathrm{s.t.}
            &   \tau A_0
                + \sum_{m=1}^M y_m A_m
                - \sum_{k,l=1}^{K+1} C_{kl} \widetilde{H}_{kl}
                - U \in \symm_+^{K+2} \\
            &   \tau b_0
                + \sum_{m=1}^M y_m b_m
                - \sum_{k,l=1}^{K+1} d_{kl} \widetilde{H}_{kl}
                - V = 0,
\end{array}
    \end{align*}
    with variables~$\tau \in \reals_+$, $y \in \reals_+^M$, and~$\widetilde{H} \in \symm_+^{K+1}$.
    
    Weak duality gives the inequality above, so the minimization problem below upper-bounds the worst-case expected loss, which suffices for the certificate of~\cref{thm:finite_sample} and for training on an upper bound.
    Equality holds if~$\suppset$ contains a point with~$G \succ 0$ and~$H \succ 0$, since the remaining constraints are affine in~$(G, F)$ and the refined Slater condition applies~\cite[Section~5.2.3]{boydConvexOptimization2004}.

    Combining the results above, the worst-case expected loss is upper-bounded by the optimal value of the problem
    \begin{equation*}
        \begin{array}{ll}
            \textnormal{minimize} & (1/N) \sum_{i=1}^N s_i \\
            \textnormal{subject to}
            &   - c_0 \tau_i
                - \Tr(X_i^\tpose \hG_i)
                - Y_i^\tpose \hF_i + \lambda \varepsilon \le s_i \\
            &   \tau_i A_0
                + \sum_{m=1}^M (y_i)_m A_m
                - \sum_{k,l=1}^{K+1} (\widetilde{H}_i)_{kl} C_{kl}
                - \Aobj - X_i
                \in \symm_+^{K+2} \\
            &   \tau_i b_0
                + \sum_{m=1}^M (y_i)_m b_m
                - \sum_{k,l=1}^{K+1} (\widetilde{H}_i)_{kl} d_{kl}
                - \bobj - Y_i
                = 0 \\
            &   \norm{(X_i, Y_i)} \le \lambda,
                \hfill i=1,\dots,N,
        \end{array}
    \end{equation*}
    with variables~$s_i \in \reals$, $\tau_i \in \reals_+$, $y_i \in \reals_+^M$, $X_i \in \symm^{K+2}$, $Y_i \in \reals^{K+1}$, $\widetilde{H}_i \in \symm_+^{K+1}$, and~$\lambda \in \reals_+$,
    and equality under the Slater condition.
\end{proof}

\begin{remark}[Slater point for convex quadratic functions]
    The condition of the lemma holds for~$\mathcal{Q}_{\mu, L}$ with~$\mu < L$ and~$d \ge K+2$ whenever~$\theta^k_k \neq 0$ for every~$k$ in~\cref{assumption:first-order}.
    Take~$Q = \diag(\lambda_1, \dots, \lambda_d)$ with~$\mu < \lambda_i < L$ and distinct~$\lambda_1, \dots, \lambda_{K+2}$, and take~$x^0$ with nonzero first~$K+2$ entries and zero remaining entries, so that~$x^\star = 0$.
    Then~$x^k = p_k(Q) x^0$ for a polynomial~$p_k$ of degree exactly~$k$, and~$g^k = Q p_k(Q) x^0$ has degree exactly~$k+1$.
    The first~$K+2$ rows of~$P = (x^0, g^0, \dots, g^K)$ are therefore the product of the diagonal matrix of the nonzero entries of~$x^0$, a Vandermonde matrix with distinct nodes, and an upper triangular matrix with nonzero diagonal, so~$P$ has full column rank and~$G = P^\tpose P \succ 0$.
    For the same reason~$X = (x^0, \dots, x^K)$ has full column rank, and~$H = X^\tpose (Q - \mu I)(L I - Q) X \succ 0$ since~$\mu I \prec Q \prec L I$.
\end{remark}

\section{Proofs for Sections~\ref{sec:perf_guarantee} and~\ref{sec:theory}}
\label{appendix:theory}

\subsection{Proof of Theorem~\ref{thm:finite_sample}}
\label{appendix:finite_sample}

Define the union of the support sets~$\suppset^\theta$ over~$\Theta$ as
\begin{equation*}
    \suppset^\Theta = \bigcup_{\theta \in \Theta} \suppset^\theta.
\end{equation*}
Throughout this section, $(G, F) = \algo_\theta(z)$ is the lifted representation of~\cref{def:composite-lift}, so that~$G = P^\tpose P \in \symm^{K_0}_+$ is the Gram matrix of the~$K_0$ columns~\eqref{eq:P-columns}.
We write~$G = G(\theta)$, $F = F(\theta)$, $P = P(\theta)$, $x^i = x^i(\theta)$, $y^i = y^i(\theta)$, $w^i = w^i(\theta)$, and~$s^k = s^k(\theta)$ to emphasize the~$\theta$-dependency of each term, where~$y^i$ and~$w^i$ are the gradient-query point and the proximal input of~\eqref{algo:prox-grad}.
The smooth part is interpolated at~$y^0, \dots, y^{K-1}$ and at~$x^K$, and we write~$u^i$ for the~$i$-th of these points, so that~$u^i = y^i$ for~$i \le K-1$ and~$u^K = x^K$.
Without loss of generality, let~$x^\star = 0$.
Both proofs below use the largest parameter norm
\begin{equation}\label{eq:M-theta}
    M_\Theta
    := \max_{\theta \in \Theta} \norm{\theta}
    = \sqrt{ {\textstyle \sum_{j=1}^p} \max \{ l_j^2, u_j^2 \} },
\end{equation}
which is finite and attained because~$\Theta$ is a box by~\cref{assumption:compact-theta}($ii$).

We first bound the points that the algorithm produces, uniformly over~$\theta \in \Theta$.
These bounds depend on the structure of the algorithm, and the induction that establishes them also produces the constants used in the proof of~\cref{lem:algo-map-lipschitz}.
The compactness of~$\suppset^\Theta$ then follows from them, together with the closedness that~\cref{assumption:compact-theta}($ii$) gives on its own.

\begin{lemma}[Compactness of the lifted support set]\label{lem:supp-compact}
    Let~$K$ be fixed and suppose~\cref{assumption:first-order,assumption:compact-theta} hold.
    If a proximal step is present, suppose in addition that~$z$ is a composite instance and that~\cref{assumption:composite} holds.
    Then there exist finite constants~$R_K$, $M_h$, and~$M_P$ such that
    \begin{equation}\label{eq:column-bounds}
        \begin{aligned}
            \sup_{\theta \in \Theta} \max_{0 \le i \le K} \max \big\{
                \norm{x^i(\theta) - x^\star},\,
                \norm{y^i(\theta) - x^\star},\,
                \norm{w^i(\theta) - x^\star}
            \big\} &\le R_K,
            \\
            \sup_{\theta \in \Theta} \max_{0 \le k \le K-1} \norm{s^k(\theta)} &\le M_h,
        \end{aligned}
    \end{equation}
    for every instance~$z \in \probset$ with solution~$x^\star$, and every column of~$P(\theta)$ has norm at most~$M_P$ for every~$\theta \in \Theta$.
    Here~$y^i$ and~$w^i$ are read as~$x^i$ for~$i = K$, where the algorithm takes no further step.
    Consequently, the set~$\suppset^\Theta$ is compact.
\end{lemma}

\begin{proof}
    First of all, the set~$\{ (\theta, G,F) \mid \theta \in \Theta,\; (G,F) \in \suppset^\theta \}$ is closed from the fact that~$\Theta$ is closed and the inequalities defining~$\suppset^\theta$ are jointly continuous in~$(\theta, G, F)$, since the entries of~$A_m(\theta)$ and~$b_m(\theta)$ are polynomial in~$\theta$.
    The second sequence of~\eqref{algo:prox-grad} does not affect this, since every coefficient there is a coordinate of~$\theta$, so unrolling the two sequences leaves each entry polynomial in~$\theta$ and only raises its degree.
    Then as~$\Theta$ is compact, the projection mapping~$(\theta, G,F) \mapsto (G,F)$ is closed, so the image of the set above, which is~$\suppset^\Theta$, is closed.
    It remains to bound the set, which is where the structure of the algorithm enters.

    The points the algorithm produces are bounded by a recursion on their joint radius
    \begin{equation*}
        R_k = \sup_{\theta \in \Theta} \max \left\{
            \max_{i \le k} \norm{x^i(\theta) - x^\star},\;
            \max_{j \le k-1} \norm{y^j(\theta) - x^\star},\;
            \max_{j \le k-1} \norm{w^j(\theta) - x^\star}
        \right\},
    \end{equation*}
    which collects the iterates, the gradient-query points, and the proximal inputs of~\eqref{algo:prox-grad}.
    The~$k$-th step produces~$y^k(\theta)$, then~$w^k(\theta)$, then~$x^{k+1}(\theta)$, and we bound each in turn by an affine function of~$R_k$.
    A pure gradient step is the case~$h \equiv 0$, where the proximal operator is the identity, and a pure proximal step is the case of vanishing gradient coefficients.
    Note that we impose the initial condition that~$x^0(\theta) - x^\star$ is uniformly bounded with~$\theta \in \Theta$, so that~$R_0$ is the constant of~\cref{assumption:compact-theta}($i$).
    Also, from $L$-Lipschitzness of~$\nabla f$ and~\eqref{eq:composite-gradient-bound}, every gradient the algorithm queries obeys
    \begin{equation}\label{eq:grad-bound}
        \norm{\nabla f(u)}
        \le \norm{\nabla f(u) - \nabla f(x^\star)} + \norm{\nabla f(x^\star)}
        \le L \norm{u - x^\star} + R_\nabla,
    \end{equation}
    so it is bounded whenever its query point is.
    The Cauchy--Schwarz inequality bounds the coefficient sums as
    \begin{equation}\label{eq:coeff-sum}
        \left| \sum_{i=0}^k \theta^k_i \right|
        \le \sqrt{k+1} \sqrt{ \sum_{i=0}^k (\theta^k_i)^2 }
        \le \sqrt{K+1}\, M_\Theta,
        \qquad k = 0, \dots, K-1.
    \end{equation}
The subgradients already produced are controlled by~$R_k$ as well.
    For~$i \le k-1$, the definition~\eqref{eq:prox-subgradient} and the lower bound~$\theta^i_\mathrm{prox} \ge \theta_{\min}$ of~\cref{assumption:composite}($i$) give
    \begin{equation}\label{eq:subgrad-radius}
        \norm{s^i(\theta)}
        = \frac{\norm{w^i(\theta) - x^{i+1}(\theta)}}{\theta^i_\mathrm{prox}}
        \le \frac{\norm{w^i(\theta) - x^\star} + \norm{x^{i+1}(\theta) - x^\star}}{\theta_{\min}}
        \le \frac{2 R_k}{\theta_{\min}}.
    \end{equation}

    For the extrapolation, we have
    \begin{align*}
        \norm{y^k(\theta) - x^\star}
        &\le
            \norm{x^k(\theta) - x^\star}
            + \norm{ \sum_{i=0}^{k-1} \phi^k_i \nabla f(y^i(\theta)) }
            + \norm{ \sum_{i=0}^{k-1} \psi^k_i s^i(\theta) }
        \\&\le
            R_k
            + M_\Theta \sqrt{ \sum_{i=0}^{k-1} \norm{\nabla f(y^i(\theta))}^2 }
            + M_\Theta \sqrt{ \sum_{i=0}^{k-1} \norm{s^i(\theta)}^2 }
        \\&\le
            R_k
            + \sqrt{K}\, M_\Theta \left( L R_k + R_\nabla \right)
            + \sqrt{K}\, M_\Theta \, \frac{2 R_k}{\theta_{\min}}.
    \end{align*}
    The second inequality is the Cauchy--Schwarz inequality together with~$\sqrt{\sum_i (\phi^k_i)^2} \le \|\theta\| \le M_\Theta$ and the same bound on the~$\psi^k_i$, and the third uses~\eqref{eq:grad-bound} and~\eqref{eq:subgrad-radius} over the at most~$K$ terms of each sum.
    Writing the result as
    \begin{equation}\label{eq:extrapolation-bound}
        \norm{y^k(\theta) - x^\star} \le a R_k + b,
        \qquad
        a = 1 + \sqrt{K}\, M_\Theta \left( L + \frac{2}{\theta_{\min}} \right),
        \qquad
        b = \sqrt{K}\, M_\Theta R_\nabla,
    \end{equation}
    we have~$a \ge 1$ and~$b \ge 0$, so~\eqref{eq:extrapolation-bound} bounds~$\|y^i(\theta) - x^\star\|$ for every~$i \le k$.
    A single-sequence method has~$y^k = x^k$, where the two extrapolation sums are absent and~\eqref{eq:extrapolation-bound} holds with~$a = 1$ and~$b = 0$.

    For the gradient part, we have
    \begin{align*}
        &\norm{
            w^k (\theta)
            - x^\star
        }
        \\&=
            \norm{
                \big( y^k(\theta) - x^\star \big)
                - \sum_{i=0}^k \theta^k_i \left(
                    \nabla f(y^i(\theta))
                    - \nabla f(x^\star)
                \right)
                - \sum_{i=0}^k \theta^k_i \nabla f(x^\star)
            }
        \\&\le
            \norm{y^k(\theta) - x^\star}
            + \norm{
                \sum_{i=0}^k \theta^k_i \big( \nabla f(y^i(\theta)) - \nabla f(x^\star) \big)
            }
            + \left| \sum_{i=0}^k \theta^k_i \right|
            \norm{\nabla f(x^\star)}
        \\&\le
            \norm{y^k(\theta) - x^\star}
            + \sqrt{\sum_{i=0}^k (\theta^k_i)^2} \sqrt{
                \sum_{i=0}^k
                \norm{
                    \nabla f(y^i(\theta))
                    - \nabla f(x^\star)
                }^2
            }
            +   \left| \sum_{i=0}^k \theta^k_i \right|
                \norm{\nabla f(x^\star)}
        \\&\le
            \norm{y^k(\theta) - x^\star}
            +   M_\Theta
                \sqrt{
                    \sum_{i=0}^k
                    \norm{\nabla f(y^i(\theta)) - \nabla f(x^\star)}^2
                }
            +   \left| \sum_{i=0}^k \theta^k_i \right|
                \norm{\nabla f(x^\star)}
        \\&\le
            \norm{y^k(\theta) - x^\star}
            +   L\, M_\Theta
                \sqrt{
                    \sum_{i=0}^k
                    \norm{y^i(\theta) - x^\star}^2
                }
            +   \left| \sum_{i=0}^k \theta^k_i \right|
                \norm{\nabla f(x^\star)}.
    \end{align*}
    The third inequality uses $\sqrt{\sum_{i=0}^k (\theta^k_i)^2} \le \|\theta\| \le M_\Theta$, and the last uses~$L$-Lipschitzness of~$\nabla f$ over the query points.
    Bounding every~$\|y^i(\theta) - x^\star\|$ by~\eqref{eq:extrapolation-bound} and the coefficient sum by~\eqref{eq:coeff-sum} gives
    \begin{equation}\label{eq:prox-input-bound}
        \norm{w^k(\theta) - x^\star} \le c \left( a R_k + b \right) + d,
        \qquad
        c = 1 + \sqrt{K+1}\, L\, M_\Theta,
        \qquad
        d = \sqrt{K+1}\, M_\Theta R_\nabla .
    \end{equation}
    For the proximal part,
    \begin{align*}
        \norm{
            x^{k+1} (\theta)
            - x^\star
        }
        &=
            \norm{
                \big(
                    \prox_{\theta^k_\mathrm{prox} h} (w^k(\theta))
                    - \prox_{\theta^k_\mathrm{prox} h} (x^\star)
                \big)
                + \big(
                    \prox_{\theta^k_\mathrm{prox} h} (x^\star)
                    - x^\star
                \big)
            }
        \\&\le
            \norm{
                \prox_{\theta^k_\mathrm{prox} h} (w^k(\theta))
                - \prox_{\theta^k_\mathrm{prox} h} (x^\star)
            }
            +  \norm{
                \prox_{\theta^k_\mathrm{prox} h} (x^\star)
                - x^\star
            }
        \\&\le
            \norm{w^k(\theta) - x^\star}
            +  \norm{\prox_{\theta^k_\mathrm{prox} h} (x^\star) - x^\star},
    \end{align*}
    where the inequality comes from~$\prox_{\theta^k_\mathrm{prox} h}$ being a firmly nonexpansive operator.
    The proximal part is closed by~$\|\prox_{\theta^k_\mathrm{prox} h} (x^\star) - x^\star\| \le R_\mathrm{prox}$ from~\eqref{eq:prox-residual-bound}, so that~$\|x^{k+1}(\theta) - x^\star\| \le \|w^k(\theta) - x^\star\| + R_\mathrm{prox}$.
    Chaining this with~\eqref{eq:prox-input-bound} bounds each of the three points the~$k$-th step produces by an affine function of~$R_k$, and the iterate just displayed is the largest of the three.
    This gives the affine recursion
    \begin{equation}\label{eq:radius-recursion}
        R_{k+1} \le \nu R_k + \omega_R,
        \qquad
        \nu = c \, a,
        \qquad
        \omega_R = c \, b + d + R_\mathrm{prox},
    \end{equation}
    with~$a$, $b$ from~\eqref{eq:extrapolation-bound} and~$c$, $d$ from~\eqref{eq:prox-input-bound}, which gives~$R_K \le \nu^K R_0 + \omega_R (\nu^K - 1)/(\nu - 1)$.
    The radius grows exponentially in the horizon~$K$, as one expects of a composition of~$K$ Lipschitz maps, and is finite for every fixed~$K$.
    A single-sequence method has~$a = 1$ and~$b = 0$, where the recursion reduces to~$\nu = c$ and~$\omega_R = d + R_\mathrm{prox}$.
    The second sequence therefore costs a constant factor in the radius, and that factor degrades as~$\theta_{\min} \to 0$ through the~$a$ of~\eqref{eq:extrapolation-bound}.

    The subgradient columns~$s^k(\theta)$ of~\eqref{eq:P-columns} are bounded as well.
    Taking~$k = K-1$ in~\eqref{eq:subgrad-radius} and bounding~$R_{K-1}$ by~$R_K$ gives
    \begin{equation}\label{eq:subgrad-bound}
        M_h
        = \sup_{\theta \in \Theta} \max_{k} \norm{s^k(\theta)}
        \le \frac{2 R_K}{\theta_{\min}}
        < \infty.
    \end{equation}
    Every column of~$P(\theta)$ is therefore bounded, with
    \begin{equation*}
        M_P \le \max \big\{ R_0, \; L R_K + R_\nabla, \; M_h \big\},
    \end{equation*}
    the middle entry coming from~\eqref{eq:grad-bound} together with the fact that~$R_K$ bounds every gradient-query point.
    The optimality column needs no entry of its own, since $\|s^\star\| = \|\nabla f(x^\star)\| \le R_\nabla$.
    In particular, the column bound~$M_P$ is finite.
    The entries of~$F$ are bounded as well.
    Convexity gives~$|f(u^i(\theta)) - f(x^\star)| \le \max \{ \|\nabla f(u^i(\theta))\|, \|\nabla f(x^\star)\| \} \, \|u^i(\theta) - x^\star\| \le M_P R_K$, and likewise~$|h(x^k(\theta)) - h(x^\star)| \le M_P R_K$ with the subgradients~$s^{k-1}(\theta)$ and~$s^\star$.
    Hence~$\|F(\theta)\| \le \sqrt{K_0 - 1}\, M_P R_K$.
    Since $G = P^\tpose P$ with $P$ stacking these bounded columns, the lifted outputs $\algo_\theta (z)$ are bounded over $\theta \in \Theta$ and $z \in \probset$.
    By exactness~(\cref{subsec:worst-case}), every element of~$\suppset^\Theta$ is such an output, so~$\suppset^\Theta$ is bounded.
    Being closed and bounded, the set~$\suppset^\Theta$ is compact.
\end{proof}

We now show the Lipschitz property of the algorithm mapping $\theta \mapsto \algo_\theta$.

\begin{lemma}\label{lem:algo-map-lipschitz}
    Suppose the hypotheses of~\cref{lem:supp-compact} hold.
    Then the mapping $\theta \mapsto \algo_\theta$ is Lipschitz, \ie, there exists $\bar{L} > 0$ such that
    \begin{equation*}
        \norm{
            \algo_{\theta_1} (z)
            - \algo_{\theta_2} (z)
        }
        \le
            \bar{L} \norm{\theta_1 - \theta_2},
    \end{equation*}
    for all $\theta_1, \theta_2 \in \Theta$ and $z \in \probset$, where the norm on the left is the pair norm~$\|(G,F)\| = \sqrt{\|G\|_F^2 + \|F\|^2}$ of~\cref{subsec:worst-case} applied to the lifted representation~$\algo_\theta(z) = (G, F)$ of~\cref{def:composite-lift}, whose matrix~$P(\theta)$ lies in~$\reals^{d \times K_0}$.
\end{lemma}

\begin{proof}
    We use the notation introduced before~\cref{lem:supp-compact}, together with its constants~$R_K$, $M_h$, and~$M_P$ of~\eqref{eq:column-bounds}.
    Since
    \begin{equation*}
        G(\theta_1) - G(\theta_2)
        = P_1^\tpose P_1 - P_2^\tpose P_2
        = \frac{1}{2} \big(
            (P_1 - P_2)^\tpose (P_1 + P_2)
            + (P_1 + P_2)^\tpose (P_1 - P_2)
        \big),
    \end{equation*}
    we have
    \begin{align*}
        &\norm[F]{G(\theta_1) - G(\theta_2)}
        \\&\le \norm[F]{P(\theta_1) - P(\theta_2)} \norm[2]{P(\theta_1) + P(\theta_2)}
        \\&\le
            2 \sqrt{K_0} \norm[F]{P(\theta_1)- P(\theta_2)}
            \max_{\theta \in \{\theta_1,\theta_2\}}
            \max_{i, k} \Big\{
                \norm{x^0 - x^\star},\,
                \norm{s^\star},\,
                \norm{\nabla f(u^i(\theta))},\,
                \norm{s^k(\theta)}
            \Big\}.
    \end{align*}
    The last factor in the bound above is a column norm of~$P(\theta_1)$ or~$P(\theta_2)$, hence at most the column bound~$M_P$ of~\cref{lem:supp-compact}.

    Regarding the~$F$-vector, its~$f$-entries sit at the points~$u^i$ by~\eqref{eq:F-columns}.
    The convexity of~$f$ implies
    \begin{equation*}
        \nabla f(u^i(\theta_2))^\tpose \big( u^i(\theta_1) - u^i(\theta_2) \big)
        \le f(u^i(\theta_1)) - f(u^i(\theta_2))
        \le \nabla f(u^i(\theta_1))^\tpose \big( u^i(\theta_1) - u^i(\theta_2) \big),
    \end{equation*}
    so the Cauchy--Schwarz inequality gives
    \begin{equation*}
        \big| f(u^i(\theta_1)) - f(u^i(\theta_2)) \big|
        \le
            \max_{\theta \in \{\theta_1, \theta_2\}}
            \norm{\nabla f(u^i(\theta))}\,
            \norm{u^i(\theta_1) - u^i(\theta_2)}.
    \end{equation*}
    The two components are therefore controlled by the same two quantities, the column bound~$M_P$ and the distance between the points at which the two parameters evaluate the oracles.

    It remains to show that~$\theta \mapsto P(\theta)$ is Lipschitz continuous.
    We use induction to prove that,
    if~$\|x^k(\theta_1) - x^k(\theta_2)\| \le L_k \|\theta_1 - \theta_2\|$ for all~$\theta_1, \theta_2 \in \Theta$ and any choice of~$f$ and~$h$, then there exists~$L_{k+1}>0$ such that~$\|x^{k+1}(\theta_1) - x^{k+1}(\theta_2)\| \le L_{k+1} \|\theta_1 - \theta_2\|$ for all~$\theta_1, \theta_2 \in \Theta$ and any~$f$ and~$h$ as well.
    The induction starts at~$L_0 = 0$, since the initial iterate~$x^0$ belongs to the instance~$z$ and carries no~$\theta$-dependency.

    The~$k$-th step of~\eqref{algo:prox-grad} produces three points and one subgradient column, so the induction carries four constants,
    \begin{equation*}
        \norm{y^k(\theta_1) - y^k(\theta_2)} \le L_{y,k} \norm{\theta_1 - \theta_2},
        \qquad
        \norm{w^k(\theta_1) - w^k(\theta_2)} \le L_{w,k} \norm{\theta_1 - \theta_2},
    \end{equation*}
    together with~$L_{k+1}$ for the iterate and~$L_{s,k}$ for the column~$s^k$.
    Step~$k$ uses only~$L_k$ and the constants~$L_{y,i}$, $L_{s,i}$ of the earlier steps~$i \le k-1$, so the four are well defined by simultaneous induction.
    All four are nondecreasing in~$k$, which we use to replace sums of squares by their largest term.

    For the extrapolation, splitting each sum into a coefficient difference against a bounded column and a bounded coefficient against a column difference gives
    \begin{align*}
        &\norm{y^k(\theta_1) - y^k(\theta_2)}
        \\&\le
            \norm{x^k(\theta_1) - x^k(\theta_2)}
            + \norm{ \sum_{i=0}^{k-1} \big( (\phi_1)^k_i - (\phi_2)^k_i \big) \nabla f(y^i(\theta_1)) }
        \\&\quad
            + \norm{ \sum_{i=0}^{k-1} (\phi_2)^k_i \big( \nabla f(y^i(\theta_1)) - \nabla f(y^i(\theta_2)) \big) }
            + \norm{ \sum_{i=0}^{k-1} \big( (\psi_1)^k_i - (\psi_2)^k_i \big) s^i(\theta_1) }
        \\&\quad
            + \norm{ \sum_{i=0}^{k-1} (\psi_2)^k_i \big( s^i(\theta_1) - s^i(\theta_2) \big) }
        \\&\le
            \Big(
                L_k
                + \sqrt{K} \left( M_P + M_h \right)
                + \sqrt{K}\, M_\Theta \left( L\, L_{y,k-1} + L_{s,k-1} \right)
            \Big) \norm{\theta_1 - \theta_2}.
    \end{align*}
    Here~$(\phi_1)^k_i$ is the coordinate~$\phi^k_i$ of~$\theta_1$, following the convention used for~$(\theta_1)^k_i$ below.
    The coefficient differences are bounded by~$\|\theta_1 - \theta_2\|$ and the coefficients themselves by~$M_\Theta$, both by the Cauchy--Schwarz inequality, while~$\|\nabla f(y^i(\theta_1))\| \le M_P$ and~$\|s^i(\theta_1)\| \le M_h$ hold by~\eqref{eq:column-bounds}, since both are columns of~$P(\theta_1)$.
    The column differences use~$L$-Lipschitzness of~$\nabla f$ with the induction hypothesis~$L_{y,i}$, and the induction hypothesis~$L_{s,i}$ directly.
    This defines~$L_{y,k}$ as the bracket above, with~$L_{y,0} = L_0 = 0$ because the two sums are empty.
    A single-sequence method has~$L_{y,k} = L_k$.

    For the gradient part, we have
    \begin{align*}
        &w^k (\theta_1) - w^k (\theta_2)
        \\&=
        y^k (\theta_1) - y^k (\theta_2)
        - \sum_{i=0}^k \Big(
            (\theta_1)^k_i \nabla f(y^i(\theta_1))
            - (\theta_2)^k_i \nabla f(y^i(\theta_2))
        \Big)
        \\&=
        y^k (\theta_1) - y^k (\theta_2)
        - \sum_{i=0}^k  \big(
            (\theta_1)^k_i
            - (\theta_2)^k_i
        \big) \nabla f(y^i(\theta_1))
        - \sum_{i=0}^k (\theta_2)^k_i \big(
            \nabla f(y^i(\theta_1))
            - \nabla f(y^i(\theta_2))
        \big)
        \\&=
        y^k (\theta_1) - y^k (\theta_2)
        - \sum_{i=0}^k  \big(
            (\theta_1)^k_i
            - (\theta_2)^k_i
        \big) \big(
            \nabla f(y^i(\theta_1)) - \nabla f(x^\star) 
        \big)
        \\&\quad
        - \Big(
            \sum_{i=0}^k \big( (\theta_1)^k_i - (\theta_2)^k_i \big) 
        \Big) \nabla f(x^\star)
        - \sum_{i=0}^k (\theta_2)^k_i \big(
            \nabla f(y^i(\theta_1))
            - \nabla f(y^i(\theta_2))
        \big).
    \end{align*}
    Therefore,
    \begin{align*}
        \norm{ w^k (\theta_1) - w^k (\theta_2) }
        &\le
            \norm{ y^k (\theta_1) - y^k (\theta_2) }
        \\&\quad+
           \left\|
                \sum_{i=0}^k
                \big( (\theta_1)^k_i - (\theta_2)^k_i \big)
                \big( \nabla f(y^i(\theta_1)) - \nabla f(x^\star) \big)
            \right\|
        \\&\quad
        + \left\|
            \sum_{i=0}^k \big(
                (\theta_1)^k_i - (\theta_2)^k_i
            \big)
            \nabla f(x^\star)
        \right\|
        \\&\quad
        +   \left\|
                \sum_{i=0}^k (\theta_2)^k_i 
                \big(
                    \nabla f(y^i(\theta_1))
                    - \nabla f(y^i(\theta_2))
                \big)
            \right\|
        \\&\le
            \norm{ y^k (\theta_1) - y^k (\theta_2) }
        \\&\quad
        +   \sqrt{
                \sum_{i=0}^k \big(
                    (\theta_1)^k_i
                    - (\theta_2)^k_i
                \big)^2
            } \sqrt{
                \sum_{i=0}^k \norm{\nabla f(y^i(\theta_1)) - \nabla f(x^\star)}^2
            }
        \\&\quad
            + \sqrt{K+1} \sqrt{
                \sum_{i=0}^k \big(
                    (\theta_1)^k_i
                    - (\theta_2)^k_i
                \big)^2
            }
            \left\| \nabla f(x^\star) \right\|
        \\&\quad
        +   \sqrt{
                \sum_{i=0}^k
                ((\theta_2)^k_i)^2
            } \sqrt{
                \sum_{i=0}^k
                \norm{
                    \nabla f(y^i(\theta_1))
                    - \nabla f(y^i(\theta_2))
                }^2
            }
        \\&\le
        \norm{ y^k (\theta_1) - y^k (\theta_2) }
        +   L\, \norm{\theta_1 - \theta_2}
            \sqrt{
                \sum_{i=0}^k \norm{y^i(\theta_1) - x^\star}^2
            }
        \\&\quad
        +   \sqrt{K+1}\, \|\theta_1 - \theta_2\|
            \|\nabla f(x^\star)\|
        +   L\, M_\Theta
            \sqrt{
                \sum_{i=0}^k \norm{y^i(\theta_1) - y^i(\theta_2)}^2
            }.
    \end{align*}
    The third term involves the sum of the coefficient differences rather than their norm, so it costs an extra factor~$\sqrt{K+1}$ by the same Cauchy--Schwarz bound as~\eqref{eq:coeff-sum}.

    For the proximal part, we have
    \begin{align*}
    &\norm{x^{k+1}(\theta_1) - x^{k+1}(\theta_2)}
    \\&=
        \norm{
            \prox_{(\theta_1)^k_\mathrm{prox} h}(w^k(\theta_1))
            - \prox_{(\theta_2)^k_\mathrm{prox} h}(w^k(\theta_2))
        }
    \\&\le
        \norm{
            \prox_{(\theta_1)^k_\mathrm{prox} h}(w^k(\theta_1))
            - \prox_{(\theta_1)^k_\mathrm{prox} h}(w^k(\theta_2))
        }
    \\&\quad
        +
        \norm{
            \prox_{(\theta_1)^k_\mathrm{prox} h}(w^k(\theta_2))
            - \prox_{(\theta_2)^k_\mathrm{prox} h}(w^k(\theta_2))
        }
    \\&\le
        \norm{w^k(\theta_1) - w^k(\theta_2)}
        + \norm{
            \prox_{(\theta_1)^k_\mathrm{prox} h}(w^k(\theta_2))
            - \prox_{(\theta_2)^k_\mathrm{prox} h}(w^k(\theta_2))
        }.
    \end{align*}
    The second inequality uses nonexpansiveness of the proximal operator.
    For the remaining term, we have
    \begin{align*}
        &\norm{
            \prox_{(\theta_1)^k_\mathrm{prox} h}(w^k(\theta_2))
            - \prox_{(\theta_2)^k_\mathrm{prox} h}(w^k(\theta_2))
        }
        \\&\qquad
        \le |(\theta_1)^k_\mathrm{prox} - (\theta_2)^k_\mathrm{prox}| 
            \left\|
                \frac{
                    w^k (\theta_2) - \prox_{(\theta_2)^k_\mathrm{prox} h} (w^k (\theta_2))
                }{(\theta_2)^k_\mathrm{prox}}
            \right\|.
    \end{align*}
    The last factor is the norm of the subgradient column~$s^k(\theta_2)$ of~\eqref{eq:P-columns}, at most~$M_h$ by~\eqref{eq:column-bounds}, and the coordinate difference satisfies~$|(\theta_1)^k_\mathrm{prox} - (\theta_2)^k_\mathrm{prox}| \le \|\theta_1 - \theta_2\|$.
    Therefore,
    \begin{equation*}
        \norm{x^{k+1}(\theta_1) - x^{k+1}(\theta_2)}
        \le
            \norm{w^k(\theta_1) - w^k(\theta_2)}
            + M_h \norm{\theta_1 - \theta_2}.
    \end{equation*}
    We may conclude that there exists~$L_{k+1} > 0$ such that, for all~$\theta_1, \theta_2 \in \Theta$ and functions~$f$, $h$,
    \begin{equation*}
        \norm{x^{k+1} (\theta_1) - x^{k+1} (\theta_2)} \le L_{k+1} \norm{\theta_1 - \theta_2},
    \end{equation*}
    given~$\|x^k(\theta_1) - x^k(\theta_2)\| \le L_k \|\theta_1 - \theta_2\|$.
    Collecting the three displays gives the per-step recursion
    \begin{equation}\label{eq:lip-recursion}
        L_{w,k} \le c\, L_{y,k} + \sqrt{K+1} \left( L\, R_K + R_\nabla \right),
        \qquad
        L_{k+1} \le L_{w,k} + M_h,
    \end{equation}
    with the multiplier~$c = 1 + \sqrt{K+1}\, L\, M_\Theta$ of~\eqref{eq:prox-input-bound}.
    The gradient part uses~$\|y^i(\theta_1) - x^\star\| \le R_K$ from~\eqref{eq:column-bounds} and~$\sqrt{\sum_{i=0}^k L_{y,i}^2} \le \sqrt{K+1}\, L_{y,k}$, the latter from the sequence~$(L_{y,i})$ being nondecreasing in~$i$.
    The proximal part adds~$M_h$ on top of it, by the same argument as in the boundedness induction in the proof of~\cref{lem:supp-compact}.
    Together with the extrapolation constant~$L_{y,k}$ above and the column constant~$L_{s,k}$ below, this closes the induction, and every constant is finite for a fixed horizon~$K$.
    A single-sequence method has~$L_{y,k} = L_k$, where~\eqref{eq:lip-recursion} collapses to the single recursion~$L_{k+1} \le c\, L_k + \sqrt{K+1} ( L R_K + R_\nabla ) + M_h$.

    We now prove the Lipschitz continuity of the mapping~$\theta \mapsto P(\theta)$ regarding~\eqref{eq:P-columns}.
    This reduces to proving that each column is Lipschitz continuous, since the squared Frobenius norm of a matrix is bounded by the sum of the squares of its columns.
    First of all, columns~$x^0 - x^\star$ and~$s^\star$ are~$\theta$-independent.
    Regarding gradient columns, the first~$K$ sit at the gradient-query points and the last at~$x^K$, so
    \begin{equation*}
        \norm{\nabla f(y^i(\theta_1)) - \nabla f(y^i(\theta_2))}
        \le L \norm{y^i(\theta_1) - y^i(\theta_2)}
        \le L L_{y,i} \norm{\theta_1 - \theta_2},
        \qquad i = 0, \dots, K-1,
    \end{equation*}
    and the same bound with~$L_K$ in place of~$L_{y,i}$ for the column at~$x^K$.
    We write~$L_{y,K} = L_K$ so that the two cases read alike.
    Now, regarding the subgradient columns related to proximal operators, we have
    \begin{equation*}
        s^k(\theta_1) - s^k(\theta_2)
        = \frac{
            \big( w^k(\theta_1) - x^{k+1}(\theta_1) \big)
            - \big( w^k(\theta_2) - x^{k+1}(\theta_2) \big)
        }{(\theta_1)^k_\mathrm{prox}}
        + s^k(\theta_2) \, \frac{
            (\theta_2)^k_\mathrm{prox} - (\theta_1)^k_\mathrm{prox}
        }{(\theta_1)^k_\mathrm{prox}}.
    \end{equation*}
    Bounding the two denominators by~$\theta_{\min}$ and the second numerator by~$M_h \|\theta_1 - \theta_2\|$ gives
    \begin{equation*}
        \norm{s^k(\theta_1) - s^k(\theta_2)}
        \le L_{s,k} \norm{\theta_1 - \theta_2},
        \qquad
        L_{s,k} = \frac{L_{w,k} + L_{k+1} + M_h}{\theta_{\min}},
        \qquad k = 0, \dots, K-1,
    \end{equation*}
    which is the fourth constant of the induction, with~$L_{w,k}$ from~\eqref{eq:lip-recursion}.
    This is the only place where~$\theta_{\min}$ enters the proof.
    Summing the squares of the column bounds gives
    \begin{equation}\label{eq:barLP}
        \norm[F]{P(\theta_1) - P(\theta_2)} \le \bar{L}_P \norm{\theta_1 - \theta_2},
        \qquad
        \bar{L}_P = \sqrt{
            L^2 {\textstyle \sum_{i=0}^{K}} L_{y,i}^2
            + {\textstyle \sum_{k=0}^{K-1}} L_{s,k}^2
        },
    \end{equation}
    the second term vanishes for the smooth case.

    The $F$-component also uses a similar induction step.
    The gradients~$\nabla f(u^i(\theta))$ are columns of~$P(\theta)$ and are therefore bounded by~$M_P$, and the induction above gives~$\|u^i(\theta_1) - u^i(\theta_2)\| \le L_{y,i} \|\theta_1 - \theta_2\|$, with the convention~$L_{y,K} = L_K$ of the column bounds.
    The convexity bound at the beginning then writes as
    \begin{equation*}
        \big| f(u^i(\theta_1)) - f(u^i(\theta_2)) \big|
        \le M_P L_{y,i} \norm{\theta_1 - \theta_2},
        \qquad i = 0,\dots,K.
    \end{equation*}
    Similarly, using~\eqref{eq:P-columns} and convexity of~$h$, whose entries sit at the iterates,
    \begin{equation*}
        \big| h(x^i(\theta_1)) - h(x^i(\theta_2)) \big|
        \le M_h L_i \norm{\theta_1 - \theta_2},
        \qquad i = 1,\dots,K.
    \end{equation*}
    Summing the squares over all entries and using~$M_h \le M_P$ together with~$L_i \le L_{y,i}$ gives
    \begin{equation*}
        \norm{F(\theta_1) - F(\theta_2)}
        \le
            \sqrt{2}\, M_P \sqrt{ {\textstyle \sum_{i=0}^K} L_{y,i}^2 }\;
            \norm{\theta_1 - \theta_2}.
    \end{equation*}
    These result in
    \begin{align*}
        \norm{
            \algo_{\theta_1}(z) - \algo_{\theta_2}(z)
        }
        &=
            \sqrt{
                \norm[F]{G(\theta_1) - G(\theta_2)}^2
                + \norm{F(\theta_1) - F(\theta_2)}^2
            }
        \\&\le
            M_P \sqrt{
                4 K_0 \bar{L}_P^2
                + 2 {\textstyle \sum_{i=0}^K} L_{y,i}^2
            }\;
            \norm{\theta_1 - \theta_2}.
    \end{align*}
    The gradient columns alone make~$\bar{L}_P^2 \ge L^2 \sum_{i=0}^K L_{y,i}^2$ in~\eqref{eq:barLP}, so the claim holds with
    \begin{equation}\label{eq:barL}
        \bar{L} = M_P \bar{L}_P \sqrt{4 K_0 + 2 L^{-2}},
    \end{equation}
    which is finite for every fixed horizon~$K$.
\end{proof}

The lemma covers logistic regression, which takes no proximal step and uses the second sequence, and LASSO, which is single-sequence with a real-valued~$h$.
The lower bound~$\theta_{\min}$ is a live hypothesis for LASSO, since the subgradient column constant~$L_{s,k}$ divides by it, and~\cref{rem:prox-lower-bound} shows that this dependence is unavoidable.
The PDHG iteration of~\cref{subsec:lp} is not of the form~\eqref{algo:prox-grad}, and~\cref{rem:pdhg-extension} extends both lemmas to it.

The constant~$\bar{L}$ enters the finite-sample radius~$\overline{\varepsilon}(\beta)$ of~\cref{thm:finite_sample} only through its finiteness.
A larger~$\bar{L}$ therefore enlarges~$\overline{\varepsilon}(\beta)$ by a constant factor, given~$\theta_{\min}$, and leaves the~$N$-dependence of~\cref{thm:finite_sample,thm:compare_risk} unchanged.

We follow the measure concentration result of~\cite{fournierRateConvergenceWasserstein2015} to derive the radius~$\varepsilon$ guaranteeing that DRO loss upper-bounds true risk with probability at least~$1-\beta$, as in~\cite{parkDatadrivenAnalysisFirstOrder2025}.
Choose~$\varepsilon = \varepsilon(\beta, \theta)$ as follows:
\begin{equation}\label{eq:finite-sample-radius}
    \varepsilon(\beta, \theta) = 
    \begin{cases}
        \left( \frac{\log(c_1 \beta^{-1})}{c_2 N} \right)^{1/q} & N \ge \frac{\log(c_1 \beta^{-1})}{c_2} \\
        \left( \frac{\log(c_1 \beta^{-1})}{c_2 N} \right)^{1/a} & N < \frac{\log(c_1 \beta^{-1})}{c_2},
    \end{cases}
\end{equation}
where~$c_1>0$ and~$c_2>0$ are constants that depend on the pushforward~$\prob^\theta$ of~\cref{appendix:perf_guarantee} only through~$q$ and an exponential moment of~$\|(G, F)\|$, which is bounded uniformly in~$\theta$ because~$\prob^\theta$ is supported in the compact set~$\suppset^\Theta$ of~\cref{lem:supp-compact}.
Hence~$\varepsilon(\beta, \theta) = \varepsilon(\beta)$ does not depend on~$\theta$.
Here~$q = K_0(K_0+1)/2 + (K_0-1)$ is the dimension of the lifted space~$\symm^{K_0} \times \reals^{K_0-1}$, which is~$(K+2)(K+3)/2 + (K+1)$ for the smooth lifting of~\cref{subsec:worst-case}, and~$a$ is the light-tail exponent of the concentration result.
From this, we get the following measure concentration inequality for each $\theta \in \Theta$:
\begin{lemma}\label{lem:theta-radius}
    Let $\beta \in (0, 1)$.
    For each~$\theta \in \Theta$, there exists $\varepsilon (\beta, \theta) > 0$ as in \eqref{eq:finite-sample-radius} such that
    \begin{equation*}
        \prob^N \left(
            \prob^\theta
            \in \ambiset_{\varepsilon \left( \beta, \theta \right)} \big( \hprob_N^\theta \big)
        \right) \ge 1 - \beta.
    \end{equation*}
\end{lemma}

\begin{proof}
    The proof follows directly from \cite[Theorem~2]{fournierRateConvergenceWasserstein2015}.
\end{proof}

We now prove~\cref{thm:finite_sample} using the covering number argument on compact~$\Theta$ to extract a $\theta$-independent radius~$\varepsilon$ of the Wasserstein ambiguity set.

\begin{proof}[Proof of~\cref{thm:finite_sample}]
    For any $\delta > 0$, there exists a finite covering $C_\delta$ of the compact set $\Theta$ in the sense that $|C_\delta| < \infty$ and for any $\theta \in \Theta$, there exists $\theta^{(j)} \in C_\delta$ such that $\|\theta - \theta^{(j)}\| \le \delta$.
    According to \cref{lem:theta-radius}, the radius $\varepsilon (\beta / |C_{\delta}|) > 0$ satisfies
    \begin{equation*}
        \prob^N \left(
            \prob^{\theta^{(j)}} \in \ambiset_{\varepsilon(\beta/|C_\delta|)} \big( \hprob_N^{\theta^{(j)}} \big)
        \right) \ge 1 - \beta / |C_{\delta}|,
        \qquad \forall \theta^{(j)} \in C_{\delta}.
    \end{equation*}
    From the union bound, we have
    \begin{equation*}
        \prob^N \left(
            \prob^{\theta^{(j)}} \in \ambiset_{\varepsilon(\beta/|C_\delta|)} \big( \hprob_N^{\theta^{(j)}} \big),
            \; \forall \theta^{(j)} \in C_{\delta}
        \right) \ge 1 - \beta.
    \end{equation*}
    Now, define
    \begin{equation*}
        \overline{\varepsilon} (\beta, \delta) =
            \varepsilon (\beta / |C_\delta|)
            + 2 \bar{L} \delta.
    \end{equation*}
    For any $\theta \in \Theta$, there exists $\theta^{(j)} \in C_\delta$ such that $\|\theta - \theta^{(j)}\| \le \delta$.
    From~\cref{lem:algo-map-lipschitz}, we have
    \begin{equation*}
        W_1 \left(
            \hprob_N^\theta,\,
            \hprob_N^{\theta^{(j)}}
        \right)
            \le \bar{L} \norm{\theta - \theta^{(j)}}
            \le \bar{L} \delta,
    \end{equation*}
    and
    \begin{equation*}
        W_1 \left(
            \prob^\theta,\,
            \prob^{\theta^{(j)}}
        \right)
            \le \bar{L} \norm{\theta - \theta^{(j)}}
            \le \bar{L} \delta.
    \end{equation*}
    From the triangle inequality of the Wasserstein distance,
    \begin{align*}
        &W_1 \left(
            \hprob_N^\theta,\,
            \prob^\theta
        \right)
        \\&\le
            W_1 \left(
                \hprob_N^\theta,\,
                \hprob_N^{\theta^{(j)}}
            \right)
            + W_1 \left( 
                \hprob_N^{\theta^{(j)}},\,
                \prob^{\theta^{(j)}}
            \right)
            + W_1 \left(
                \prob^{\theta^{(j)}},\,
                \prob^\theta
            \right)
        \\&\le
            \bar{L}\delta + \varepsilon (\beta / |C_\delta|) + \bar{L}\delta
            = \overline{\varepsilon} (\beta, \delta).
    \end{align*}
    Therefore, we get
    \begin{equation*}
        \prob^N \left(
            \prob^\theta \in \ambiset_{\overline{\varepsilon} (\beta, \delta)} (\hprob_N^\theta),
            \; \forall \theta \in \Theta
        \right) \ge 1 - \beta.
    \end{equation*}
    Finally, for any $\delta > 0$, we get
    \begin{equation*}
        \prob^N \left(
            \prob^\theta \in \ambiset_{\overline{\varepsilon} (\beta)} \big( \hprob_N^\theta \big),
            \; \forall \theta \in \Theta
        \right) \ge
        \prob^N \left(
            \prob^{\theta^{(j)}}
            \in \ambiset_{\varepsilon(\beta/|C_\delta|)} \big( \hprob_N^{\theta^{(j)}} \big),
            \; \forall \theta^{(j)} \in C_{\delta}
        \right) \ge 1 - \beta,
    \end{equation*}
    for $\overline{\varepsilon} (\beta, \delta) = \varepsilon (\beta / |C_\delta|) + 2 \bar{L} \delta$.
    The radius of the theorem is~$\overline{\varepsilon} (\beta) := \overline{\varepsilon} (\beta, \delta)$ for a single fixed~$\delta > 0$, chosen below.

    From~\cite[Corollary~4.2.13]{vershyninHighDimensionalProbabilityIntroduction2018}, the covering number $|C_\delta|$ of $\Theta$ is upper bounded as
    \begin{equation*}
        |C_\delta| \le B_\delta := \left( \frac{2\, \diam (\Theta)}{\delta} + 1 \right)^{\dim (\Theta)}.
    \end{equation*}
    By~\eqref{eq:finite-sample-radius} the map~$\alpha \mapsto \varepsilon(\alpha)$ is decreasing, and $\beta / |C_\delta| \ge \beta / B_\delta$, so
    \begin{equation*}
        \overline{\varepsilon} (\beta, \delta)
        \le
            \varepsilon \left( \beta / B_\delta \right) + 2 \bar{L} \delta.
    \end{equation*}
    This holds for every fixed~$\delta > 0$, so any choice of~$\delta$ gives a valid radius, and the~$\delta$ minimizing the right-hand side gives the tightest one.
    Substituting the first expression of~\eqref{eq:finite-sample-radius}, valid for~$N$ large enough, and absorbing~$c_1$, $c_2$, and the factor~$\sqrt{2}$ between the pair norm and the Euclidean norm on~$\reals^q$,
    \begin{equation*}
        \overline{\varepsilon} (\beta, \delta)
        \lesssim \left( \frac{\log(1/\beta) + \dim(\Theta) \log(2\, \diam(\Theta)/\delta + 1) }{N} \right)^{1/q} + 2 \bar{L} \delta,
    \end{equation*}
    where~$q$ is the dimension of~$(G, F)$.
    Naively choosing~$\delta = 1/N$ yields
    \begin{equation*}
        \overline{\varepsilon} (\beta, 1/N) \lesssim
            \left( \frac{\log(1/\beta) + \dim(\Theta)
                \log \big( 1 + 2 N \diam(\Theta) \big)
            }{N} \right)^{1/q} + \frac{2\bar{L}}{N}.
    \end{equation*}
    The first branch of~\eqref{eq:finite-sample-radius} applies whenever~$c_2 N \ge \log(c_1 B_{1/N} / \beta)$, which holds for all sufficiently large~$N$ because the right-hand side grows only logarithmically in~$N$.
\end{proof}

The rate~$N^{-1/q}$ with $q = O(K^2)$ is conservative for large~$K$ and moderate~$N$.
Sharper rates can depend on the intrinsic dimension of the lifted distribution~\cite{weedSharpAsymptoticFinitesample2019}, but establishing such rates uniformly over~$\Theta$ is beyond the scope of this paper.

\subsection{Well-posedness of~(\ref{prob:dr-l2o})}\label{appendix:well-posedness}

We now prove~\cref{prop:attainment}, \ie, the learning problem~\eqref{prob:dr-l2o} attains its minimum over~$\Theta$.
The parameter set~$\Theta$ is compact by~\cref{assumption:compact-theta}($ii$), so the only requirement is continuity of the DRO risk~$\theta \mapsto \risk_\varepsilon (\theta, \hprob_N)$.

Write~$Z = (G, F)$ for a lifted pair and~$\hZ_i (\theta) = \algo_\theta (\hz_i)$ for the~$i$-th lifted sample.
Given~$\varepsilon > 0$ and~$\theta \in \Theta$, the feasible set of~\eqref{prob:dro-pep-primal} is defined as
\begin{equation}\label{eq:primal-feasible-set}
    C(\theta)
    = \left\{
        Z = (Z_1, \dots, Z_N) \in (\suppset^\theta)^N
        \;\left|\;
            (1/N) {\textstyle \sum_{i=1}^N} \norm{Z_i - \hZ_i (\theta)} \le \varepsilon
    \right.\right\},
\end{equation}
and~\cref{prop:strong-duality} identifies the DRO risk with the following optimal value:
\begin{equation}\label{eq:risk-as-primal-value}
    \risk_\varepsilon (\theta, \hprob_N)
    =
        \max_{Z \in C(\theta)}\; (1/N) {\textstyle \sum_{i=1}^N} \ell (Z_i).
\end{equation}

Continuity of~$\theta \mapsto \risk_\varepsilon (\theta, \hprob_N)$ is not a trivial fact.
The objective of~\eqref{eq:risk-as-primal-value} is linear in~$Z$ and free of~$\theta$, but the feasible set~$C(\theta)$ moves with~$\theta$ through both the support set~$\suppset^\theta$ and the lifted samples~$\hZ_i (\theta)$.
Therefore, when the parameter~$\theta$ changes, the feasible set~$C(\theta)$ can jump, and the optimal value may do so.
We therefore consider~$C$ as a set-valued mapping and appeal to Berge's maximum theorem~\cite[Chapter~6]{bergeTopologicalSpacesIncluding1963};
this gives continuity of the optimal value whenever the objective is continuous and the mapping is nonempty-valued, compact-valued, upper hemicontinuous, and lower hemicontinuous.

\begin{proposition}[Well-posedness of~\eqref{prob:dr-l2o}]\label{prop:well-posedness}
    Fix~$\varepsilon > 0$.
    Suppose the hypotheses of~\cref{lem:supp-compact} hold.
    Then the map~$\theta \mapsto \risk_\varepsilon (\theta, \hprob_N)$ is continuous on~$\Theta$, and~\eqref{prob:dr-l2o} attains its minimum over~$\Theta$.
\end{proposition}

\begin{proof}
    The objective of~\eqref{eq:risk-as-primal-value} is linear in~$Z$, hence continuous, and does not depend on~$\theta$.
    We now prove the required properties of the mapping~$\theta \mapsto C(\theta)$ for Berge's maximum theorem to hold.

    Let~$\theta \in \Theta$.
    The lifted samples satisfy~$\hZ_i (\theta) \in \suppset^\theta$, so the tuple~$(\hZ_1(\theta), \dots, \hZ_N(\theta))$ lies in~$C(\theta)$.
    Each~$\suppset^\theta$ is closed and~$Z \mapsto (1/N) \sum_{i=1}^N \|Z_i - \hZ_i(\theta)\|$ is continuous, so~$C(\theta)$ is closed.
    Furthermore, $C(\theta)$ is contained in~$(\suppset^\Theta)^N$, which is compact by~\cref{lem:supp-compact}; therefore, it is a compact set.

    We now show the upper hemicontinuity of~$C$.
    All values of~$C$ lie in the fixed compact set~$(\suppset^\Theta)^N$, so it suffices to show that any~$\theta_n \to \btheta$ in~$\Theta$ and any~$Z^n \in C(\theta_n)$ admit a subsequence of~$(Z^n)$ converging to a point of~$C(\btheta)$.
    First, compactness gives a subsequence~$Z^{n_k} \to \bZ$ with~$Z_i^{n_k} \to \bZ_i$ for each~$i$.
    Second, the graph of~$\theta \mapsto \suppset^\theta$ is closed by the proof of~\cref{lem:supp-compact}, so each limit~$\bZ_i$ lies in~$\suppset^{\btheta}$.
    Finally, since~$\theta \mapsto \hZ_i (\theta) = \algo_\theta (\hz_i)$ is Lipschitz by~\cref{lem:algo-map-lipschitz}, the mapping~$(\theta, Z) \mapsto (1/N) \sum_{i=1}^N \|Z_i - \hZ_i (\theta)\|$ is jointly continuous, so the Wasserstein inequality is preserved in the limit.
    Therefore, $\bZ$ lies in~$C(\btheta)$.

    It remains to show the lower hemicontinuity of~$C$.
    Let~$\theta_n \to \btheta$ in~$\Theta$ and~$\bZ \in C(\btheta)$.
    By the exactness of the interpolation conditions~(\cref{subsec:worst-case}), pick~$z_i \in \probset$ with~$\bZ_i = \algo_{\btheta} (z_i)$ and define~$Z_i^n = \algo_{\theta_n} (z_i) \in \suppset^{\theta_n}$.
    \cref{lem:algo-map-lipschitz} gives~$Z_i^n \to \bZ_i$ and~$\hZ_i (\theta_n) \to \hZ_i (\btheta)$, so
    \begin{equation*}
        d_n := (1/N) {\textstyle \sum_{i=1}^N} \norm{Z_i^n - \hZ_i (\theta_n)}
        \;\longrightarrow\;
        \bar{d} = (1/N) {\textstyle \sum_{i=1}^N} \norm{\bZ_i - \hZ_i (\btheta)}
        \le \varepsilon.
    \end{equation*}
    The tuple~$Z^n = (Z_1^n, \dots, Z_N^n)$ may exceed the Wasserstein budget~$\varepsilon$, so we pull it back toward the lifted samples.
    Set~$t_n = 1$ if~$d_n \le \varepsilon$ and~$t_n = \varepsilon / d_n$ otherwise, and let~$\tZ_i^n = (1 - t_n) \hZ_i (\theta_n) + t_n Z_i^n$.
    Each~$\suppset^{\theta_n}$ is convex, so~$\tZ_i^n \in \suppset^{\theta_n}$, and
    \begin{equation*}
        (1/N) {\textstyle \sum_{i=1}^N} \norm{\tZ_i^n - \hZ_i (\theta_n)}
        = t_n d_n \le \varepsilon,
    \end{equation*}
    which gives~$\tZ^n \in C(\theta_n)$.
    From~$d_n \to \bar{d} \le \varepsilon$ we get~$t_n \to 1$, and the differences~$\hZ_i (\theta_n) - Z_i^n$ are bounded, so
    \begin{equation*}
        \tZ_i^n - \bZ_i
        = (1 - t_n) \big( \hZ_i (\theta_n) - Z_i^n \big) + \big( Z_i^n - \bZ_i \big)
        \longrightarrow 0.
    \end{equation*}
    Hence every~$\bZ \in C(\btheta)$ is the limit of a sequence~$\tZ^n \in C(\theta_n)$, proving lower hemicontinuity.

    Berge's maximum theorem now gives continuity of the optimal value of~\eqref{eq:risk-as-primal-value}, which is exactly~$\risk_\varepsilon (\theta, \hprob_N)$.
    The parameter set~$\Theta$ is compact by~\cref{assumption:compact-theta}($ii$), so a continuous function attains its minimum over it.
\end{proof}

A large enough radius makes the Wasserstein constraint in~\eqref{eq:primal-feasible-set} inactive, which reduces~$C$, with~$N = 1$, to the correspondence~$\theta \mapsto \suppset^\theta$ and turns the optimal value into the worst-case loss.
The following is therefore a special case of~\cref{prop:well-posedness}.

\begin{proposition}[Continuity of the worst-case loss]\label{prop:worst-case-continuous}
    Under the hypotheses of~\cref{prop:well-posedness}, the worst-case loss
    \begin{equation*}
        \theta \mapsto \sup_{z \in \probset} \ell \big( \algo_\theta (z) \big)
        = \sup_{(G, F) \in \suppset^\theta} \ell (G, F),
    \end{equation*}
    is continuous on~$\Theta$, and~\eqref{prob:opt-pep} attains its minimum over~$\Theta$.
\end{proposition}

\begin{proof}
    The set~$\suppset^\Theta$ is compact by~\cref{lem:supp-compact}, so its diameter is finite.
    We apply the proof of~\cref{prop:well-posedness} with~$N = 1$ and any radius~$\varepsilon > \diam (\suppset^\Theta)$.
    Both~$Z_1$ and~$\hZ_1 (\theta)$ lie in~$\suppset^\theta \subseteq \suppset^\Theta$, so the budget constraint of~\eqref{eq:primal-feasible-set} is inactive and~$C(\theta) = \suppset^\theta$ for every~$\theta \in \Theta$.
    Then~\eqref{eq:risk-as-primal-value} becomes~$\risk_\varepsilon (\theta, \hprob_1) = \max_{(G, F) \in \suppset^\theta} \ell (G, F)$, which is continuous in~$\theta$ and attains its minimum over~$\Theta$ from~\cref{prop:well-posedness}.
\end{proof}

\subsection{Proof of Proposition~\ref{prop:interpolate}}\label{appendix:interpolate}

Throughout this subsection, we omit~$\theta$-dependency and write~$\prob^\theta$ and~$\probQ^\theta$ as~$\prob$ and~$\probQ$, and~$\risk$ and~$\risk_\varepsilon$ are evaluated on lifted distributions as in~\cref{appendix:perf_guarantee}.
First, we state the proposition from~\cite{kuhnDistributionallyRobustOptimization2025}, which will be useful in our proofs.
\begin{proposition}[Proposition~8.5 of~\cite{kuhnDistributionallyRobustOptimization2025}]\label{prop:dro_regularization}
    Consider the $1$-Wasserstein ambiguity set
    \begin{equation*}
        \ambiset_{\varepsilon} (\hprob_N) = \left\{
            \probQ \mid W_1 (\probQ, \hprob_N) \le \varepsilon
        \right\},
    \end{equation*}
    where~$\probQ$ is a probability measure supported on a closed set~$\suppset$.
    Then, if~$\ell$ is a Lipschitz continuous function with~$\Expect_{Y \sim \hprob_N} ( |\ell(Y)| ) < \infty$,
    \begin{equation*}
        \begin{array}[t]{ll}
            \underset{\probQ \in \ambiset_{\varepsilon} (\hprob_N)}{\sup}
            &   \underset{Y \sim \probQ}{\Expect}
                \big( \ell(Y) \big)
        \end{array}
        \le
            \underset{Y \sim \hprob_N}{\Expect} \big( \ell(Y) \big)
            + \varepsilon \, \mathrm{Lip} (\ell).
    \end{equation*}
\end{proposition}

Note that from~\cite{parkDatadrivenAnalysisFirstOrder2025}, DRO loss interpolates~\eqref{prob:opt-pep} and~\eqref{prob:l2o} in the following sense: given a fixed~$\theta \in \Theta$, the risk~$\risk_{\varepsilon} (\theta, \hprob_N)$ tends to the worst-case loss of~\eqref{prob:opt-pep} when~$\varepsilon \to \infty$ and equals the empirical risk of~\eqref{prob:l2o} when~$\varepsilon = 0$.
The following lemma states the monotonicity and upper bound, and the proof of~\cref{prop:interpolate} recovers the two extremes as~$\varepsilon \to 0^+$ and~$\varepsilon \to \infty$.

\begin{lemma}\label{lem:compare_risk}
    Suppose the hypotheses of~\cref{lem:supp-compact} hold.
    For any algorithm parameter~$\theta \in \Theta$ and probability distribution~$\probQ$ supported on~$\suppset^\theta$,
    \begin{equation*}
        \risk (\theta, \probQ)
        \le
            \risk_{\varepsilon} (\theta, \probQ)
        \le
            \risk_{\varepsilon^+} (\theta, \probQ)
        \le
            \sup_{(G, F) \in \suppset^\theta}\, \ell (G, F),
    \end{equation*}
    for any~$\varepsilon^+ \ge \varepsilon \ge 0$.
\end{lemma}

\begin{proof}
    From $\{\probQ\} \subseteq \ambiset_{\varepsilon} (\probQ) \subseteq \ambiset_{\varepsilon^+} (\probQ)$, we get
    \begin{equation*}
        \risk (\theta, \probQ)
        \le
            \risk_{\varepsilon} (\theta, \probQ) \le
            \risk_{\varepsilon^+} (\theta, \probQ).
    \end{equation*}
    The set~$\suppset^\Theta$ is compact by~\cref{lem:supp-compact}, and~$\suppset^\theta$ is a closed subset of it, so~$\suppset^\theta$ is compact as well.
    The loss~$\ell$ is linear and hence continuous, so its supremum over~$\suppset^\theta$ is finite and attained.
    Therefore, for every~$\varepsilon > 0$,
    \begin{equation*}
        \risk_{\varepsilon} (\theta, \probQ) = \sup_{\widetilde{\probQ} \in \ambiset_{\varepsilon} (\probQ)} \Expect_{(G, F) \sim \widetilde{\probQ}} \big( \ell (G, F) \big)
        \le \sup_{\widetilde{\probQ} \in \ambiset_{\varepsilon} (\probQ)}  \sup_{(G, F) \in \suppset^\theta} \ell (G, F)
        = \sup_{(G, F) \in \suppset^\theta} \ell (G, F).
    \end{equation*}
\end{proof}

We now prove~\cref{prop:interpolate}.

\begin{proof}[Proof of~\cref{prop:interpolate}]
    First, from~\cref{lem:compare_risk},
    \begin{equation*}
        \risk (\theta, \hprob_N)
        \le
            \risk_{\varepsilon} (\theta, \hprob_N)
        \le
            \risk_{\varepsilon^+} (\theta, \hprob_N)
        \le
            \sup_{z \in \probset}\, \ell \big( \algo_\theta(z) \big),
    \end{equation*}
    for any~$\theta \in \Theta$ and~$\varepsilon^+ \ge \varepsilon > 0$.
    Each term in the inequality above is continuous in~$\theta$.
    The empirical risk is the linear loss~$\ell$ composed with the algorithm map, which is Lipschitz by~\cref{lem:algo-map-lipschitz}, the two DRO risks are continuous by~\cref{prop:well-posedness}, and the worst-case loss is continuous by~\cref{prop:worst-case-continuous}.
    Since~$\Theta$ is compact, each term attains its minimum over~$\Theta$.
    First of all,
    \begin{equation*}
        \inf_{\theta \in \Theta} \risk (\theta, \hprob_N)
        \le
            \risk (\theta, \hprob_N)
        \le
            \risk_{\varepsilon} (\theta, \hprob_N)
        \le
            \risk_{\varepsilon^+} (\theta, \hprob_N),
        \qquad \forall \theta \in \Theta.
    \end{equation*}
    For~$\theta_{\varepsilon}^\star \in \argmin_{\theta \in \Theta} \risk_{\varepsilon} (\theta, \hprob_N)$, we have
    \begin{equation*}
        \inf_{\theta \in \Theta} \risk (\theta, \hprob_N)
        \le
            \risk (\theta_\varepsilon^\star, \hprob_N)
        \le
            \risk_\varepsilon (\theta_\varepsilon^\star, \hprob_N)
            = \inf_{\theta \in \Theta} \risk_{\varepsilon} (\theta, \hprob_N)
        \le
            \risk_{\varepsilon} (\theta, \hprob_N),
        \qquad \forall \theta \in \Theta.
    \end{equation*}
    Similarly, with respect to~$\theta_{\varepsilon^+}^\star \in \argmin_{\theta \in \Theta} \risk_{\varepsilon^+} (\theta, \hprob_N)$,
    \begin{equation*}
        \inf_{\theta \in \Theta} \risk (\theta, \hprob_N)
        \le
            \risk_\varepsilon (\theta_\varepsilon^\star, \hprob_N)
        \le
            \risk_{\varepsilon^+} (\theta_{\varepsilon^+}^\star, \hprob_N)
        \le
            \risk_{\varepsilon^+} (\theta, \hprob_N),
        \qquad \forall \theta \in \Theta.
    \end{equation*}
    Lastly,
    \begin{equation*}
        \inf_{\theta \in \Theta} \risk (\theta, \hprob_N)
        \le
            \risk_{\varepsilon} (\theta_{\varepsilon}^\star, \hprob_N)
        \le
            \risk_{\varepsilon^+} (\theta_{\varepsilon^+}^\star, \hprob_N)
        \le
            \inf_{\theta \in \Theta} \sup_{(G, F) \in \suppset^\theta}\, \ell (G, F),
    \end{equation*}
    we observe that~$\varepsilon \mapsto \risk_{\varepsilon} (\theta^\star_{\varepsilon}, \hprob_N)$ is monotonically nondecreasing and upper-bounded by~$\inf_{\theta \in \Theta} \sup_{z \in \probset} \ell \big( \algo_\theta (z) \big)$.

    The hypotheses of~\cref{prop:dro_regularization} hold for every~$\theta \in \Theta$.
    The lifted loss~$\ell (G, F) = \Tr (\Aobj^\tpose G) + \bobj^\tpose F$ is linear, so Cauchy--Schwarz in the pair norm~$\|(G, F)\| = \sqrt{\|G\|_F^2 + \|F\|^2}$ inducing~$W_1$ gives the Lipschitz constant~$\mathrm{Lip} (\ell) = \|(\Aobj, \bobj)\|$.
    The empirical distribution~$\hprob_N^\theta$ has finite support, so~$\Expect_{Y \sim \hprob_N^\theta} |\ell(Y)| < \infty$.
    The upper bound in the following display then follows by applying~\cref{prop:dro_regularization} to each~$\theta$ and taking the infimum over~$\Theta$,
    \begin{align*}
        \inf_{\theta \in \Theta} \risk (\theta, \hprob_N)
        &\le
            \risk_{\varepsilon} (\theta_{\varepsilon}^\star, \hprob_N) \le
        \inf_{\theta \in \Theta} \risk (\theta, \hprob_N) + \varepsilon \, \mathrm{Lip} (\ell).
    \end{align*}
    Applying~$\varepsilon \to 0^+$, we get the desired result.

    For~$\varepsilon \ge \diam (\suppset^\Theta)$, which is finite by~\cref{lem:supp-compact}, every distribution on~$\suppset^\theta$ lies in~$\ambiset_\varepsilon (\hprob_N^\theta)$, so~$\risk_\varepsilon (\theta, \hprob_N) = \sup_{z \in \probset} \ell (\algo_\theta (z))$ for all~$\theta \in \Theta$ by~\cref{lem:compare_risk}.
    Taking the infimum over~$\Theta$ gives the second limit, which is attained for every such~$\varepsilon$.
\end{proof}

\subsection{Proof of Theorem~\ref{thm:compare_risk}}\label{appendix:compare_risk}

We now prove~\cref{thm:compare_risk}.
\begin{proof}[Proof of~\cref{thm:compare_risk}]
    According to~\cref{lem:compare_risk}, 
    \begin{equation*}
        \risk (\theta, \hprob_N)
        \le
        \risk_{\varepsilon} (\theta, \hprob_N)
        \le
        \sup_{z \in \probset}\, \ell \big( \algo_\theta (z) \big),
    \end{equation*}
    holds true for all~$\theta \in \Theta$.
    The minimizer~$\theta_{\varepsilon}^\star \in \argmin_{\theta \in \Theta} \risk_{\varepsilon} (\theta, \hprob_N)$ exists by~\cref{prop:well-posedness}, and~$\theta_{\text{PEP}}^\star \in \argmin_{\theta \in \Theta} \sup_{z \in \probset} \ell (\algo_\theta (z))$ exists by~\cref{prop:worst-case-continuous}.
    From these,
    \begin{equation*}
        \risk_{\varepsilon} (\theta_\varepsilon^\star, \hprob_N)
        \le
            \risk_{\varepsilon} (\theta_{\text{PEP}}^\star, \hprob_N)
        \le
            \sup_{z \in \probset}\, \ell \big( \algo_{\theta^\star_{\text{PEP}}} (z) \big)
        =
            \inf_{\theta \in \Theta} \sup_{z \in \probset} \ell \big( \algo_\theta (z) \big).
    \end{equation*}
    Taking~$\inf_{\theta \in \Theta}$ consecutively from smaller to larger terms, we get
    \begin{equation*}
        \inf_{\theta \in \Theta} \risk(\theta, \hprob_N)
        \le
            \inf_{\theta \in \Theta} \risk_{\varepsilon} (\theta, \hprob_N)
        =
            \risk_{\varepsilon} (\theta_{\varepsilon}^\star, \hprob_N)
        \le
            \inf_{\theta \in \Theta} \sup_{z \in \probset} \ell \big( \algo_\theta (z) \big).
    \end{equation*}
    The hypotheses of~\cref{prop:dro_regularization} hold here, as verified in~\cref{appendix:interpolate}.
    Applying~\cref{prop:dro_regularization} with~$Y = \algo_\theta (z)$,
    \begin{equation*}
        \risk_{\varepsilon} (\theta, \hprob_N)
        =
        \sup_{\probQ \in \ambiset_{\varepsilon} (\hprob_N)} \Expect_{z \sim \probQ} \left( \ell (\algo_{\theta} (z)) \right)
        \le
            \Expect_{z \sim \hprob_N} \big( \ell(\algo_{\theta} (z)) \big)
            + \varepsilon \, \mathrm{Lip} (\ell),
    \end{equation*}
    holds for every~$\theta \in \Theta$.
    Take~$\inf_{\theta \in \Theta}$ on the left-hand side of the inequality to have
    \begin{equation*}
        \risk_{\varepsilon} (\theta^\star_{\varepsilon}, \hprob_N)
        =
        \inf_{\theta \in \Theta} \risk_{\varepsilon} (\theta, \hprob_N)
        \le
            \Expect_{z \sim \hprob_N} \big( \ell(\algo_{\theta} (z)) \big)
            + \varepsilon \, \mathrm{Lip} (\ell),
    \end{equation*}
    and also take~$\inf_{\theta \in \Theta}$ on the right-hand side to get
    \begin{equation*}
        \risk_{\varepsilon} (\theta^\star_{\varepsilon}, \hprob_N)
        \le
            \inf_{\theta \in \Theta}
            \Expect_{z \sim \hprob_N} \big( \ell(\algo_{\theta} (z)) \big)
            + \varepsilon \, \mathrm{Lip} (\ell).
\end{equation*}
    We conclude the proof by upper-bounding the true risk~$\risk (\theta_{\varepsilon}^\star, \prob)$ by the DRO risk~$\risk_{\varepsilon} (\theta_{\varepsilon}^\star, \hprob_N)$ with probability at least~$1-\beta$, using~\cref{thm:finite_sample}.
\end{proof}

\section{Additional details of Section~\ref{subsec:solution}}\label{appendix:solution}

\cref{alg:dr-l2o} describes the solution method for the learning problem~\eqref{prob:dr-l2o}.

\begin{algorithm}[H]
    \caption{Stochastic Gradient Method for Distributionally Robust L2O (DR-L2O)}
    \label{alg:dr-l2o}
    \begin{algorithmic}
        \State {\bfseries Input:}
            training data~$\data_N$,
            parameter initialization~$\theta^0$,
            mini-batch size $n$,
            iteration counter~$k = 0$.
        \Repeat
        \State Sample mini-batch~$\data_n$ from training data~$\data_N$
        \State Define mini-batched empirical distribution~$\hprob_n$ of $\data_n$
        \State Evaluate~$d\risk_{\varepsilon} (\theta, \hprob_n)/d\theta$ using~\cref{lem:chain_rule}.
        \State Update~$\theta^{k+1} \leftarrow \verb|AdamWUpdate|(\theta^k)$.
        \State $k \leftarrow k+1$.
        \Until{$\theta^k$ converges}
    \end{algorithmic}
\end{algorithm}

Since~\eqref{prob:dr-l2o} is nonconvex, the parameter~$\hat{\theta}_\varepsilon$ returned by~\cref{alg:dr-l2o} may be a local minimizer rather than a global minimizer~$\theta^\star_\varepsilon$.
\Cref{thm:finite_sample} holds uniformly over~$\Theta$, so the DRO risk~$\risk_\varepsilon(\hat{\theta}_\varepsilon, \hprob_N)$ still upper-bounds the true risk~$\risk(\hat{\theta}_\varepsilon, \prob)$ with probability at least~$1-\beta$.
\Cref{thm:compare_risk} uses global optimality only through~$\risk_\varepsilon(\theta^\star_\varepsilon, \hprob_N) = \inf_{\theta \in \Theta} \risk_\varepsilon(\theta, \hprob_N)$.
For~$\hat{\theta}_\varepsilon$, both of its bounds therefore hold with the optimality gap~$\risk_\varepsilon(\hat{\theta}_\varepsilon, \hprob_N) - \risk_\varepsilon(\theta^\star_\varepsilon, \hprob_N) \ge 0$ added to the right-hand side.

In order to evaluate the gradient of~$\theta \mapsto \risk_{\varepsilon} (\theta, \hprob_n)$ given a dataset~$\data_n$ of~$n$ samples,
we use the approach of~\cite{agrawalDifferentiatingConeProgram2019,agrawalDifferentiableConvexOptimization2019,healeyDifferentiatingQuadraticCone2026} to differentiate the solution mapping of the conic linear problem in terms of problem parameters.
Let~$\mathcal{T} \big( \{ (A_m, b_m) \}_{m=0}^M, \{ (\hG_i, \hF_i) \}_{i=1}^n \big)$ be a solution mapping of the problem~\eqref{prob:dro-pep-dual} with~$N$ replaced by~$n$, and let~$\mathcal{L}$ be its objective function.
Then the DRO risk~$\risk_{\varepsilon} (\theta, \hprob_n)$ is composed of the solution mapping and $\mathcal{L}$ as follows:
\begin{equation}\label{eq:risk-composition}
    \risk_{\varepsilon} (\theta, \hprob_n)
    =
        \mathcal{L} \Big( \mathcal{T}  \big( \linear_\theta, \algo_\theta (\data_n) \big) \Big).
\end{equation}
We finally apply the chain rule to obtain the gradient with respect to~$\theta$, as follows:
\begin{lemma}\label{lem:chain_rule}
    Given~$\data_n$, the gradient of~$\theta \mapsto \risk_{\varepsilon} (\theta, \hprob_n)$ evaluates as
    \begin{align*}
        &\frac{d \risk_{\varepsilon} (\theta, \hprob_n)}{d\theta}
        =
        \left. \frac{d\mathcal{L}}{dt} \right|_{t = \mathcal{T} (\linear_\theta, \algo_\theta (\data_n))} \\
        &\quad\times
        \Bigg(
            \sum_{m=0}^M
                \left. \frac{\partial \mathcal{T}}{\partial (A_m, b_m)} \right|_{( \linear_\theta, \algo_\theta(\data_n) )}
                \frac{d (A_m(\theta), b_m(\theta))}{d\theta}
            +
            \sum_{i=1}^n
                \left. \frac{\partial \mathcal{T}}{\partial (\hG_i, \hF_i)} \right|_{ ( \linear_\theta, \algo_\theta (\data_n) ) }
            \frac{d \algo_\theta(\hz_i)}{d\theta}
        \Bigg).
    \end{align*}
\end{lemma}

\begin{proof}    
    We interpret the risk~$\risk_{\varepsilon} (\theta, \hprob_n)$ as the parametric conic program~\eqref{prob:dro-pep-dual} with parameters~$(\hG_i, \hF_i)$ for~$i=1,\dots,n$ and~$(A_m, b_m)$ for~$m=0,1,\dots,M$.
    Applying the chain rule to the composition~\eqref{eq:risk-composition} gives the result.
\end{proof}

\begin{remark}[Differentiability of the solution map]\label{rem:diff-solution-map}
    \cref{lem:chain_rule} requires the solution map~$\mathcal{T}$ to be differentiable at the parameters it is evaluated at.
    The solution map of a semidefinite program is differentiable at parameters whose primal and dual solutions are unique and satisfy strict complementarity together with primal and dual nondegeneracy.
    These conditions make the Jacobian of the optimality conditions nonsingular, so the implicit function theorem applies.
    They also hold for generic problem data, which leaves the parameters where~$\mathcal{T}$ fails to be differentiable in a thin set.
    At such parameters the approach of~\cite{agrawalDifferentiatingConeProgram2019} returns a heuristic quantity in place of the derivative.
\end{remark}

\paragraph{Training procedure}
For each experiment, we fix the algorithm structure and learn the schedule of step sizes (and, for logistic regression, momentum coefficients) $\theta = \{\theta^k\}_{k=0}^{K-1}$ for each horizon $K$ and each of~\eqref{prob:dr-l2o},~\eqref{prob:opt-pep}, and~\eqref{prob:l2o}.
For~\eqref{prob:dr-l2o} and~\eqref{prob:l2o}, we use mini-batched stochastic gradient descent with the AdamW optimizer~\cite{kingmaAdamMethodStochastic2017,loshchilovDecoupledWeightDecay2018}, with weight decay cross-validated over $\{0, 10^{-5}, 10^{-4}, 10^{-3}\}$, combined with a linear warm-up followed by a cosine annealing learning rate schedule~\cite{loshchilovSGDRStochasticGradient2017}.
We acknowledge that, since the DRO risk is not an average over samples, the mini-batch gradient is not an unbiased estimate of the gradient of~\eqref{prob:dr-l2o} on the full training set.
For~\eqref{prob:opt-pep}, we follow the approach of~\cite{kamriNumericalDesignOptimized2025}: their~\cite[Theorem~3.2]{kamriNumericalDesignOptimized2025} gives an explicit gradient of the worst-case rate with respect to the step size, and we apply gradient descent with the same cosine annealing schedule.
All experiments run for 500 iterations, with 10\% of the iterations used for a linear warm-up to the maximum learning rate, which is cross-validated from $\{10^{-5}, 10^{-4}, 10^{-3}\}$.

When learning the schedule, we do not use the performance loss at the last iterate as the training objective as this can hinder gradient information flow to earlier step sizes~\cite{chenLearningOptimizePrimer2022}.
Instead, we use an exponentially decaying weighted summation for the training loss.
Letting $\ell^k$ represent the loss at iterate $k$, we minimize the sample analogue of the surrogate objective:
\begin{equation*}
    \underset{z \sim \prob}{\Expect} \left( \sum_{k=1}^K w_k \ell^k (\mathcal{A}_\theta(z)) \right).
\end{equation*}
We choose $w_k = 0.9^{K-k}$, and show a plot of the exponential decay in Figure~\ref{fig:loss_iter_weights}.
Since $w_K = 1$ and every $\ell^k$ is nonnegative, the final-iterate loss~$\ell^K$ reported in our figures and tables is bounded by this weighted loss, so the guarantees of~\cref{sec:theory}, which apply to the weighted loss we train on, also bound the final-iterate risk.

\begin{figure}
    \centering
    \includegraphics[width=0.5\linewidth]{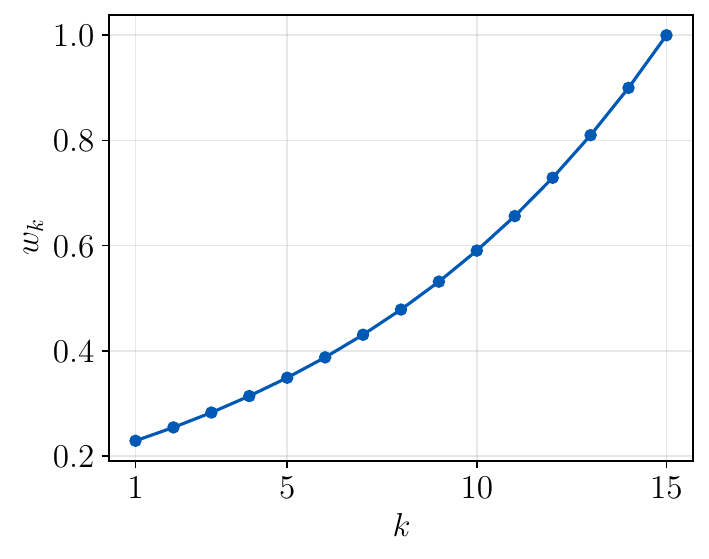}
    \caption{A plot showing the exponential weight decay for our training objective with weights $w_k = 0.9^{K-k}$ for $K=15$.}
    \label{fig:loss_iter_weights}
\end{figure}

Additionally, during the learning procedure, in order to remove numerical issues with any step size component becoming negative, we \emph{square} the step size before plugging it into the first-order method.
We initialize the values so that the square of the input step size is our desired value but we backpropagate the gradients to the square root.
This squaring enforces nonnegative step sizes, the lower bound of the parameter box~$\Theta$ in~\cref{assumption:compact-theta}; we do not need to enforce an upper bound, since we optimize for convergent algorithms, which naturally keeps the step sizes bounded.

\paragraph{Parameter selection}
For each experiment, we set a maximum value of $K$ and learn a separate schedule for each $k = 1, \dots, K$ used as the learning horizon.
The Wasserstein radius~$\varepsilon > 0$ from~\cref{subsec:data-driven} controls the probability of constraint satisfaction; in practice it is treated as a hyperparameter tuned by cross-validation~\cite{gotohCalibrationDistributionallyRobust2021,parkDatadrivenAnalysisFirstOrder2025}.
We cross-validate over $\varepsilon \in \{10^{-2}, 10^{-1}, 1, 5, 10\}$ on a validation set and choose the schedule minimizing the empirical risk.
A separate test set is used for final loss evaluations.
As noted in~\cref{rem:radius_tightness}, this cross-validated value is the operating radius, whereas $\overline{\varepsilon}(\beta)$ in~\cref{thm:finite_sample} is the radius certified a priori.

\paragraph{Offline learning versus online deployment}
The inner conic program and the stochastic gradient method of~\cref{alg:dr-l2o} run once per task and horizon, offline.
The conic program evaluates the DRO risk exactly, so~\eqref{prob:dr-l2o} minimizes a certified quantity rather than the empirical average of~\eqref{prob:l2o}.
At deployment the learned optimizer is the underlying first-order method with fixed scalars~$\theta$, so its per-iteration cost equals that of the baseline method.
Adaptive rules move this tuning cost online rather than avoiding it.
Backtracking line search, for instance, has its own hyperparameters and spends several oracle calls per accepted step, so at an equal oracle budget it completes fewer steps than a fixed schedule.

The size of the conic program~\eqref{prob:dro-pep-dual} is set by the horizon and the mini-batch size, not by the instance dimension, since the lifted data are inner products of iterates and oracle outputs.
The instance dimension ranges from~$25$ variables in~\cref{subsec:logistic} to more than~$10^4$ in~\cref{subsec:lp} without changing the program.
The horizon is the binding cost parameter, since the lifted variables carry~$O(K^2)$ entries per sample and generic conic solvers scale worse than that.
We therefore claim practicality only for moderate horizons, which is the scalability limitation of~\cref{sec:conclusion}.

\section{Further experiment details}\label{appendix:experiments}

All examples are written in Python 3.13 with JAX version 0.9.0; the code is available at \codelink.
The SDPs are solved using the Clarabel solver~\cite{goulartClarabelInteriorpointSolver2026} with primal residual, dual residual, and duality gap tolerances all set to $10^{-5}$.
For the derivatives, we use the diffcp package~\cite{agrawalDifferentiableConvexOptimization2019,agrawalDifferentiatingConeProgram2019} in order to differentiate through the KKT conditions of the inner SDP.
All computations were run on a high-performance computing cluster with 4 CPU cores and 64GB of RAM.

In all plots, shaded bars show the $10^\text{th}$ to $90^\text{th}$ quantile ranges.
For each experiment, we also tabulate the plotted values for every horizon $K$: the loss tables report the mean, the $10^\text{th}$ quantile, the median, and the $90^\text{th}$ quantile of the test loss, and the fraction-solved tables use the same tolerances $\eta$ as the figures.
Furthermore, in all SGD iteration timing tables, our code with JAX uses just-in-time (JIT) techniques for computational efficiency, and we noticed that the first few SGD iterations take significantly longer because of the compilation time.
Since this is an initial startup cost, these initial times become outliers for the full distribution of times so we discard the first 5 iteration times before computing the averages and standard deviations.

In these experimental families, since the optimal values are not always at 0 with objective value 0, we use a separate set of 1000 instances, solve all problem instances to optimality, and compute the maximum distance to optimality from the initial points.
We then add a 10\% buffer to ensure that all sample values in the training set satisfy the interpolation conditions with this computed sample distance.

\subsection{Further details of logistic regression experiment in Section~\ref{subsec:logistic}}\label{appendix:logistic}

We use the german.numer dataset from LIBSVM~\cite{changLIBSVMLibrarySupport2011}, which contains 1000 examples with 24 features.
We standardize each feature over the full dataset and append an intercept column of ones, so that~$n = 25$.
Each problem instance is a uniform random subsample of~$m = 300$ rows, and we reject subsamples that are linearly separable, since the unregularized problem then has no minimizer.
We initialize all~$x^0$ values at~0.
We compute~$L$ as the maximum of~$\lambda_{\max}(A_i^\tpose A_i)/(4m)$ over the training instances and~$R$ as the maximum of~$\|x_i^\star\|$ over the training instances.
For the out-of-distribution set, we multiply the feature matrix of each instance by~$1.25$ and hold all other parameters constant, which increases~$L$ by a factor of~$1.25^2$.
Every out-of-distribution instance then has a smoothness constant between~$0.91$ and~$1.20$, above the training bound~$L = 0.80$, so these tests probe robustness beyond what~\cref{thm:finite_sample} certifies, and the out-of-distribution performance of~\eqref{prob:dr-l2o} shows a strength of our method.

We use a training set of size 1000; the validation set and both in-distribution and out-of-distribution test sets are size 250.
We use a mini-batch of size 20 for the SGD iterations.
We initialize the step sizes with constant value~$1/L$ and the momentum coefficients with the standard FGM sequence for smooth convex functions~\cite{daspremontAccelerationMethods2021}.
The final momentum coefficient~$\rho^{K-1}$ does not affect~$x^K$, so only the first~$K-1$ coefficients are effectively learned.
In Figure~\ref{fig:logreg_losses} we show the test losses of all the learned schedules while in Table~\ref{tab:logreg-times} we provide average times per SGD iteration for select values of $K$.
Tables~\ref{tab:logreg-losses} and~\ref{tab:logreg-frac-solved} list the values plotted in Figures~\ref{fig:logreg_losses} and~\ref{fig:logreg_frac_solved} for every horizon $K$.

\begin{figure}
    \centering
    \includegraphics[width=0.8\linewidth]{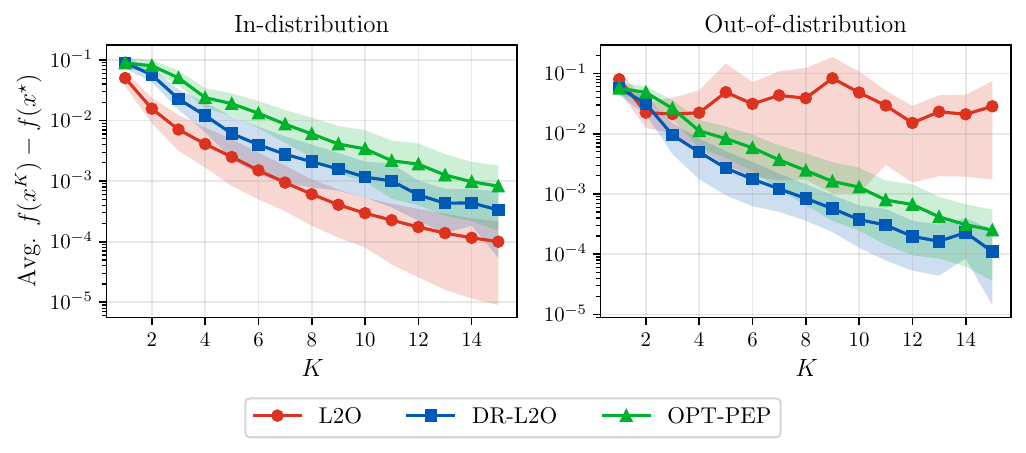}
    \caption{Logistic regression experiment (Section~\ref{subsec:logistic}) test losses against horizon $K$ on the \emph{Left:} in-distribution test set and \emph{Right:} out-of-distribution test set.}
    \label{fig:logreg_losses}
\end{figure}

\subsection{Further details of LASSO experiment in Section~\ref{subsec:lasso}}\label{appendix:lasso}

We use a sparse coding example, or the problem of recovering a sparse vector $\tilde{x}$ from noisy measurements through a sparse dictionary matrix $A\in\reals^{m \times n}$~\cite{chenTheoreticalLinearConvergence2018}.
We form $A$ by sampling each entry $A_{ij} \sim \mathcal{N}(0, 1/m)$ and normalizing each column to have unit $2$-norm; we fix $A$ for both the in-distribution and out-of-distribution sets and compute $L$ as the largest eigenvalue of $A^\tpose A$.
We sample a sparse vector $\tilde{x}_j \in \reals^n$ as $\tilde{x}_j \sim \mathcal{N}(0, \sigma_x^2)$ with probability $p_\text{mask}$ and 0 otherwise.
We then construct $b_j = A\tilde{x}_j + \delta_j$ where $\delta_j \sim \mathcal{N}(0, \sigma_\text{err}^2)$.
In the experiment, we set~$(m, n) = (250, 500)$, $\lambda = 0.4$, $\sigma_x = 2.0$ for the in-distribution set and $\sigma_x=3.0$ for the out-of-distribution set, $\sigma_{\text{err}} = 0.01$, and~$p_{\mathrm{mask}} = 0.1$.

The training set, validation set, and test set sizes are 1000, 250, and 248, respectively.
We use a size 10 mini-batch for SGD and initialize the step size schedules with constant value $1/L$.
We initialize all $x^0$ values at 0.
In Figure~\ref{fig:lasso_losses} we show the test losses of all the learned schedules while in Table~\ref{tab:lasso-times} we provide average times per SGD iteration for select values of $K$.
Tables~\ref{tab:lasso-losses} and~\ref{tab:lasso-frac-solved} list the values plotted in Figures~\ref{fig:lasso_losses} and~\ref{fig:lasso_frac_solved} for every horizon $K$.

\begin{figure}
    \centering
    \includegraphics[width=0.8\linewidth]{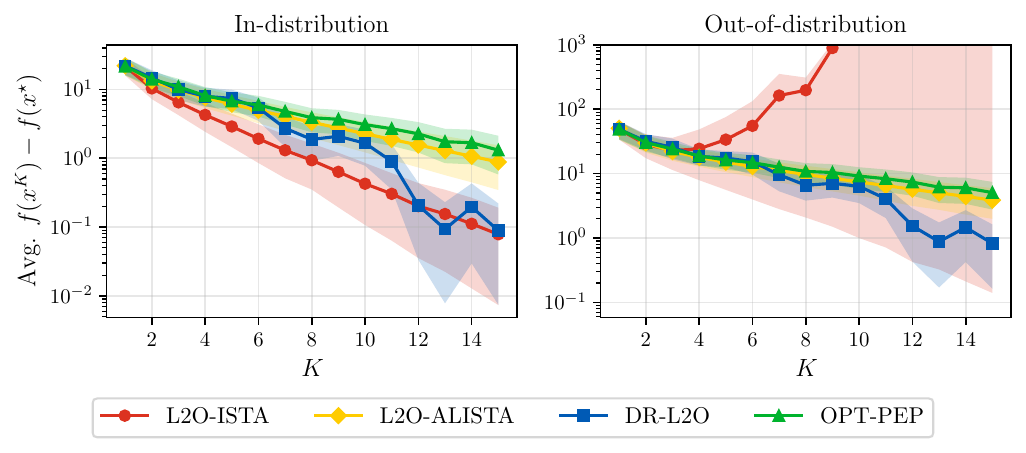}
    \caption{LASSO experiment (Section~\ref{subsec:lasso}) test losses against horizon $K$ on the \emph{Left:} in-distribution test set and \emph{Right:} out-of-distribution test set.}
    \label{fig:lasso_losses}
\end{figure}

\subsection{Further details of TV inpainting experiment in Section~\ref{subsec:lp}}\label{appendix:lp}

For the in-distribution dataset, we use the Olivetti faces dataset~\cite{samaria1994Faces,olivettifaces}, which contains 40 faces with 10 grayscale images each (64 $\times$ 64 pixels) of different angles.
For the corruption, we choose 10\% of pixels of each image uniformly at random to black out (\ie~set to 0).
We randomly pick 28 faces for the training set, 4 faces for the validation set, 8 faces for the test set, and ensure that the images are separated in a stratified manner to prevent data leakage during training.
For the out-of-distribution set, we pick 40 random \emph{color} images from Tiny ImageNet~\cite{tinyimagenet, deng2009imagenet}, where each $U_{i,j} \in \reals^3$ is now a triplet of RGB values.

We reformulate the TV LP in the standard LP form
\begin{equation*}
    \begin{array}{ll}
        \textnormal{minimize} & c^\tpose x \\
        \textnormal{subject to} & Ax = b \\
        & Gx \le h \\
        & \underline{x} \le x \le \overline{x},
    \end{array}
\end{equation*}
equivalent to the formulation used in~\cite[Equation 1]{luCuPDLPjlGPUImplementation2025} and~\cite{blinBatchedFirstOrderMethods2026}, and its saddle-point reformulation.
To obtain this form from the TV LP in Section~\ref{subsec:lp}, we vectorize the image $U$ into $\mathrm{vec}(U) \in \reals^{mn}$ and introduce nonnegative slacks $t \in \reals^{2(m-1)(n-1)}$ that upper-bound each absolute-difference term in the objective.
Let~$D$ stack the horizontal and vertical pixel-difference operators and $E$ select the known-pixel coordinates.
Writing $x = (\mathrm{vec}(U), t)$, the data of the standard LP are
\begin{equation*}
    c = \begin{bmatrix} 0 \\ \ones \end{bmatrix}, \quad
    A = \begin{bmatrix} E & 0 \end{bmatrix}, \quad
    b = \mathrm{vec}(U^\text{orig})_\mathcal{K}, \quad
    G = \begin{bmatrix} D & -I \\ -D & -I \end{bmatrix}, \quad
    h = 0,
\end{equation*}
with box bounds $\underline{x} = 0$ and $\overline{x} = (\ones, +\infty)$ on $(\mathrm{vec}(U), t)$.
The corresponding PDHG iterations are
\begin{align*}
    x^{k+1} &= \prox_{\tau^k f} \big( x^k - \tau^k M^\tpose u^k \big), \\
    \bar{x}^{k+1} &= x^{k+1} + \rho^k (x^{k+1} - x^k), \\
    u^{k+1} &= \prox_{\sigma^k g^*} \big( u^k + \sigma^k M \bar{x}^{k+1} \big).
\end{align*}
For each sample LP, we define the stacked constraint matrix $M^\tpose = \begin{bmatrix} A^\tpose & -G^\tpose \end{bmatrix}$ and compute the maximum value of $\|M\|_2$ across the samples, denoted~$M_\text{max}$.
We initialize the training procedures with $(\tau^k, \rho^k, \sigma^k) = (0.5/M_\text{max}, 1, 0.5/M_\text{max})$.
In order for the Lagrangian duality gap to make sense as a metric, we need to initialize the algorithm from a feasible point, so we choose to initialize $x^0 = 1/2$ and $u^0 = 1$.
In Figure~\ref{fig:pdlp_losses} we show the test losses of all the learned schedules and in Figure~\ref{fig:pdlp_frac_solved} we show the fractions of problems solved.
Lastly, in Table~\ref{tab:pdlp-times} we provide average times per SGD iteration for select values of $K$.
Tables~\ref{tab:pdlp-losses} and~\ref{tab:pdlp-frac-solved} list the values plotted in Figures~\ref{fig:pdlp_losses} and~\ref{fig:pdlp_frac_solved} for every horizon $K$.

\begin{figure}
    \centering
    \includegraphics[width=0.8\linewidth]{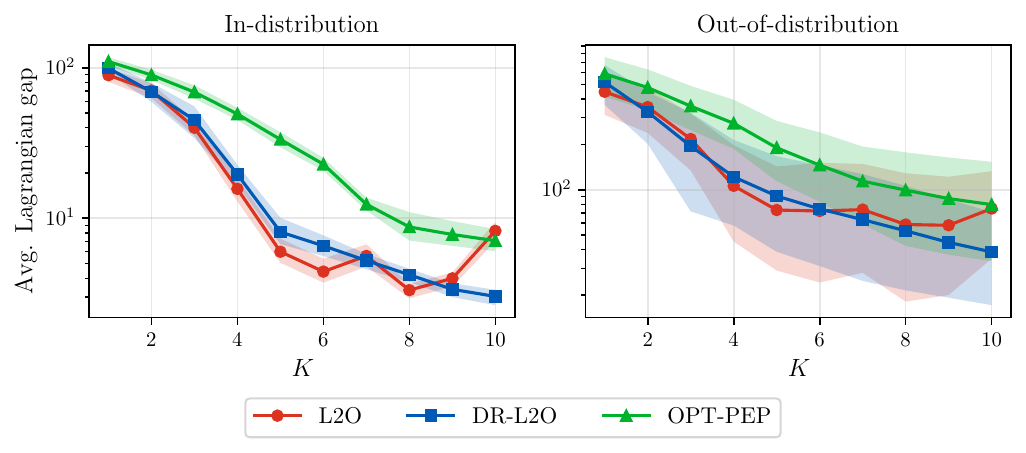}
    \caption{TV inpainting experiment (Section~\ref{subsec:lp}) test losses against horizon $K$ on the \emph{Left:} Olivetti in-distribution test set and \emph{Right:} Tiny ImageNet out-of-distribution test set.}
    \label{fig:pdlp_losses}
\end{figure}

\begin{figure}
    \centering
    \includegraphics[width=0.9\linewidth]{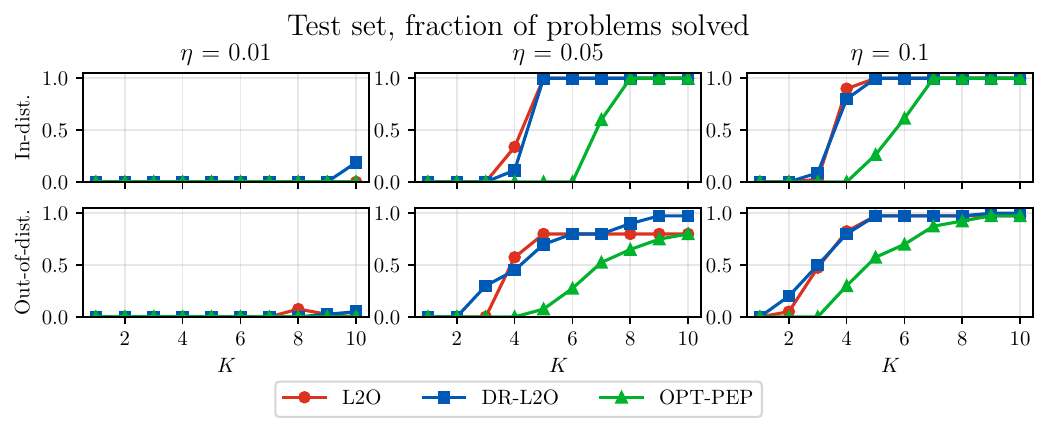}
    \caption{Fractions of TV inpainting problems solved to different relative tolerances against horizon $K$. \emph{Top:} Fractions solved on the Olivetti in-distribution set. \emph{Bottom:} Fractions solved on the Tiny ImageNet out-of-distribution set.}
    \label{fig:pdlp_frac_solved}
\end{figure}

\begin{figure}
    \centering
    \includegraphics[width=\linewidth]{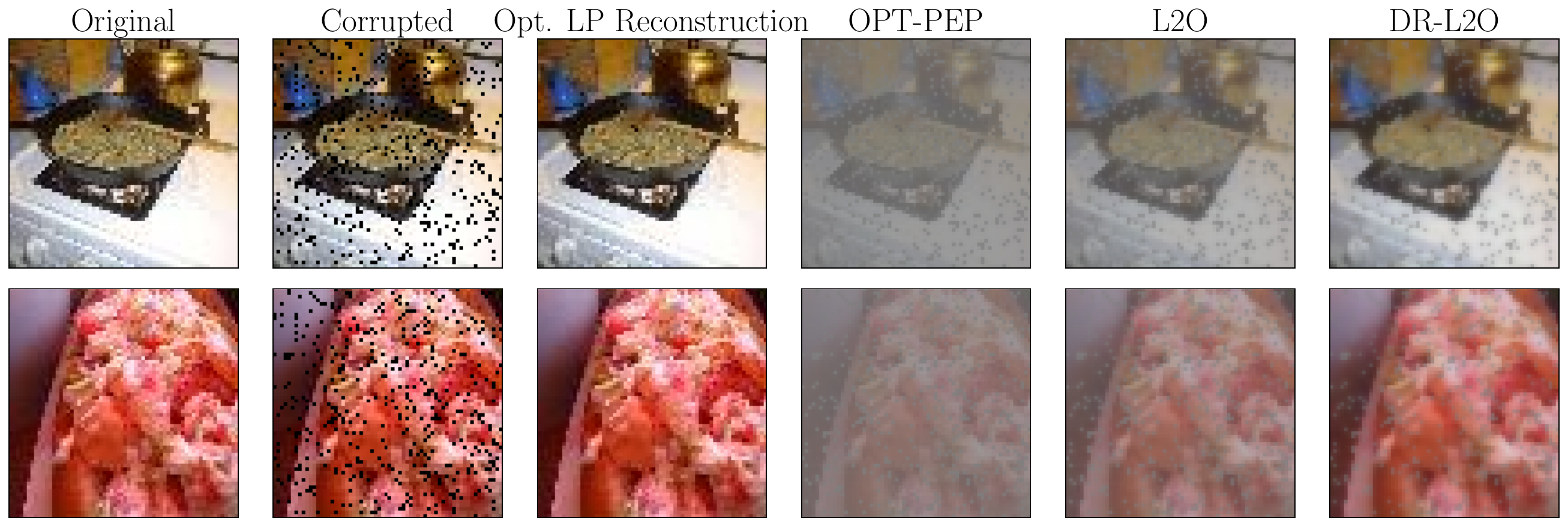}
    \caption{More reconstructions on images from the Tiny ImageNet dataset. The images are ordered in the same way as in Figure~\ref{fig:pdlp_reconstructions}.}
    \label{fig:pdlp_more_reconstructions}
\end{figure}

\subsection{Adaptive comparison}\label{appendix:adaptive}

Since the learned schedules are fixed ahead of time, it is natural to compare them against a step size that adapts to each instance.
Such a comparison is not straightforward to make fair: line searches use extra oracle calls within each iteration, and standard techniques require the performance metric to be the objective value.
For the TV inpainting experiment, our performance metric is the Lagrangian duality gap, so objective-based backtracking and Armijo conditions do not apply, and we omit this experiment.

As the adaptive baseline we use backtracking line search~\cite[Algorithm~9.2]{boydConvexOptimization2004}: each iteration starts from the trial step $t = 1$ and halves $t$ until a sufficient decrease condition holds.
For logistic regression we use the Armijo condition $f(y^k - t\nabla f(y^k)) \le f(y^k) - 0.1\,t\,\|\nabla f(y^k)\|^2$ at the extrapolated point~$y^k$, with the standard FGM momentum sequence.
Each trial reuses the product $A\nabla f(y^k)$, so backtracking uses two matrix-vector products per iteration, the same as the learned schedules, and comparing by horizon $K$ is fair (Figure~\ref{fig:logreg_backtracking}).
For LASSO we use the proximal gradient variant~\cite{beckFirstOrderMethodsOptimization2017}, which accepts $x^+ = \prox_{t\lambda\|\cdot\|_1}(x^k - t\nabla f(x^k))$ if $f(x^+) \le f(x^k) + \nabla f(x^k)^\tpose(x^+ - x^k) + \|x^+ - x^k\|^2/(2t)$, with $f$ the least-squares term.
Here each trial needs a new product with~$A$, so an iteration can cost more than the two products of a learned iteration: at $K=15$, backtracking uses on average $36.5$ in-distribution and $39.2$ out-of-distribution products instead of $30$.
We therefore compare against both the horizon $K$ and the number of matrix-vector products (Figure~\ref{fig:lasso_backtracking}).
Tables~\ref{tab:logreg-losses} and~\ref{tab:lasso-backtracking} list the plotted values.

\begin{figure}
    \centering
    \includegraphics[width=0.8\linewidth]{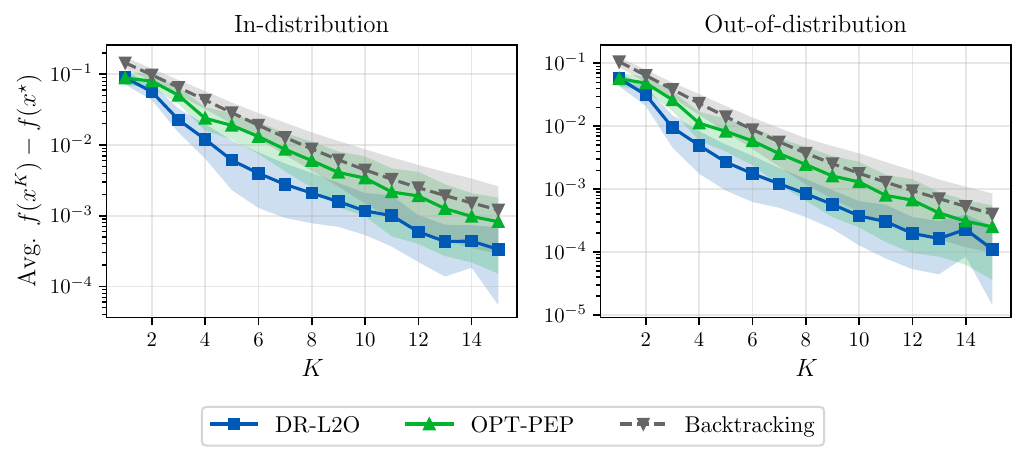}
    \caption{Logistic regression experiment (Section~\ref{subsec:logistic}) test losses of~\eqref{prob:dr-l2o}, \eqref{prob:opt-pep}, and backtracking line search against horizon $K$ on the \emph{Left:} in-distribution test set and \emph{Right:} out-of-distribution test set.}
    \label{fig:logreg_backtracking}
\end{figure}

\begin{figure}
    \centering
    \includegraphics[width=0.8\linewidth]{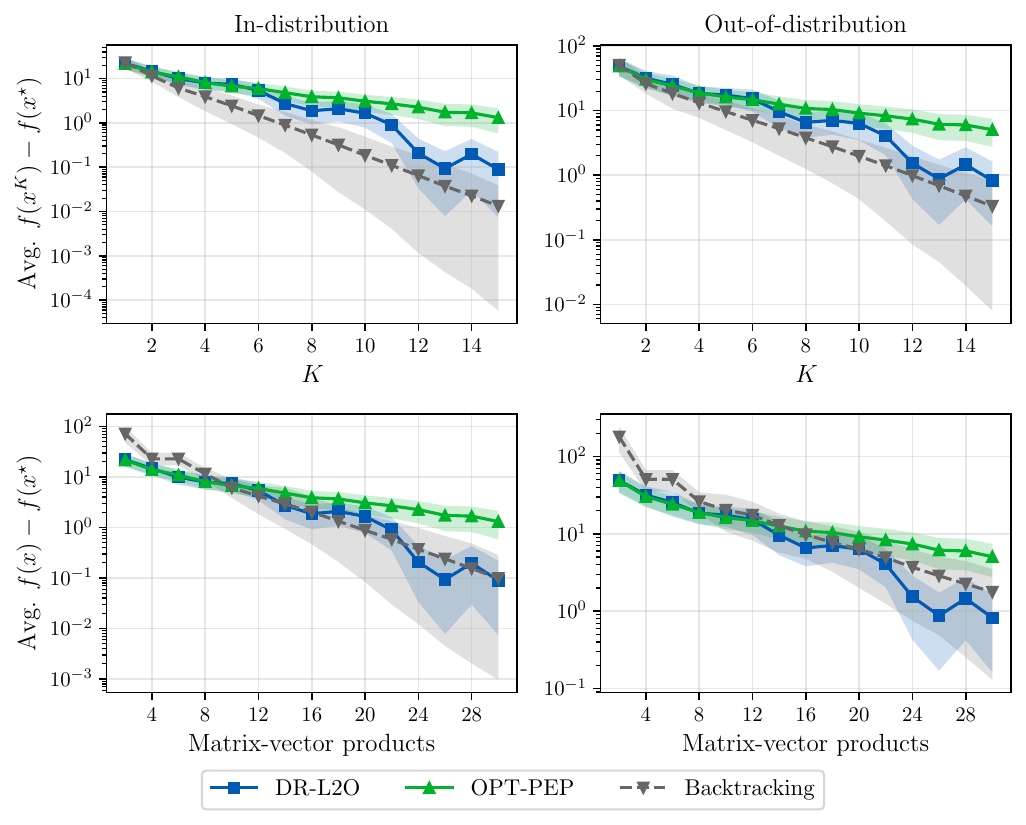}
    \caption{LASSO experiment (Section~\ref{subsec:lasso}) test losses of~\eqref{prob:dr-l2o}, \eqref{prob:opt-pep}, and backtracking line search on the \emph{Left:} in-distribution test set and \emph{Right:} out-of-distribution test set. \emph{Top:} against horizon $K$. \emph{Bottom:} against the number of matrix-vector products.}
    \label{fig:lasso_backtracking}
\end{figure}

\section{Tables}\label{appendix:tables}

\begin{table}[htbp]
    \centering
    \caption{Logistic regression experiment per-SGD-iteration wall time (mean $\pm$ $2\sigma$, in seconds)
   for the validation-best schedule used by the losses figure.}
    \label{tab:logreg-times}
    \small
    \begin{filecontents*}{logreg_times.csv}
Framework,K,Time
L2O,1,$0.001 \pm 0.000$
,5,$0.002 \pm 0.000$
,10,$0.006 \pm 0.001$
,15,$0.009 \pm 0.001$
DR-L2O,1,$0.062 \pm 0.045$
,5,$0.602 \pm 0.157$
,10,$2.635 \pm 0.594$
,15,$7.402 \pm 0.706$
OPT-PEP,1,$0.002 \pm 0.000$
,5,$0.006 \pm 0.001$
,10,$0.018 \pm 0.002$
,15,$0.070 \pm 0.005$
\end{filecontents*}
\pgfplotstabletypeset[
      col sep=comma,
      string type,
      columns/Framework/.style={column name=Framework, column type=l},
      columns/K/.style={column name={$K$}, column type=l},
      columns/Time/.style={column name={Time (s)}, column type=l},
      every head row/.style={before row=\toprule, after row=\midrule},
      every last row/.style={after row=\bottomrule},
      every row no 4/.style={before row=\midrule},
      every row no 8/.style={before row=\midrule},
    ]{logreg_times.csv}
  \end{table}

\begin{table}[htbp]
    \centering
    \caption{LASSO experiment per-SGD-iteration wall time (mean $\pm$ $2\sigma$, in seconds)
   for the validation-best schedule used by the losses figure.}
    \label{tab:lasso-times}
    \small
    \begin{filecontents*}{lasso_times.csv}
Framework,K,Time
L2O-ISTA,1,$0.006 \pm 0.000$
,5,$0.015 \pm 0.001$
,10,$0.025 \pm 0.002$
,15,$0.037 \pm 0.002$
L2O-ALISTA,1,$0.004 \pm 0.000$
,5,$0.011 \pm 0.001$
,10,$0.020 \pm 0.001$
,15,$0.029 \pm 0.001$
DR-L2O,1,$0.144 \pm 0.027$
,5,$0.825 \pm 0.129$
,10,$2.993 \pm 0.327$
,15,$14.697 \pm 8.933$
OPT-PEP,1,$0.015 \pm 0.001$
,5,$0.050 \pm 0.009$
,10,$0.364 \pm 0.041$
,15,$2.008 \pm 0.424$
\end{filecontents*}
\pgfplotstabletypeset[
      col sep=comma,
      string type,
      columns/Framework/.style={column name=Framework, column type=l},
      columns/K/.style={column name={$K$}, column type=l},
      columns/Time/.style={column name={Time (s)}, column type=l},
      every head row/.style={before row=\toprule, after row=\midrule},
      every last row/.style={after row=\bottomrule},
      every row no 4/.style={before row=\midrule},
      every row no 8/.style={before row=\midrule},
      every row no 12/.style={before row=\midrule},
    ]{lasso_times.csv}
  \end{table}

\begin{table}[htbp]
    \centering
    \caption{TV inpainting experiment per-SGD-iteration wall time (mean $\pm$ $2\sigma$, in seconds)
   for the validation-best schedule used by the losses figure.}
    \label{tab:pdlp-times}
    \small
    \begin{filecontents*}{pdlp_times.csv}
Framework,K,Time
L2O,1,$0.119 \pm 0.014$
,5,$0.450 \pm 0.016$
,10,$0.845 \pm 0.019$
DR-L2O,1,$0.410 \pm 0.018$
,5,$3.755 \pm 0.154$
,10,$57.125 \pm 27.471$
OPT-PEP,1,$0.037 \pm 0.012$
,5,$0.371 \pm 0.021$
,10,$3.954 \pm 0.309$
\end{filecontents*}
\pgfplotstabletypeset[
      col sep=comma,
      string type,
      columns/Framework/.style={column name=Framework, column type=l},
      columns/K/.style={column name={$K$}, column type=l},
      columns/Time/.style={column name={Time (s)}, column type=l},
      every head row/.style={before row=\toprule, after row=\midrule},
      every last row/.style={after row=\bottomrule},
      every row no 3/.style={before row=\midrule},
      every row no 6/.style={before row=\midrule},
    ]{pdlp_times.csv}
  \end{table}

\begin{table}[htbp]
    \centering
    \caption{Logistic regression experiment test losses plotted in Figures~\ref{fig:logreg_losses} and~\ref{fig:logreg_backtracking}: mean, $10^\text{th}$ quantile, median, and $90^\text{th}$ quantile of $f(x^K) - f(x^\star)$ over the in-distribution and out-of-distribution test sets for every horizon $K$. Backtracking line search uses two matrix-vector products per iteration, as the learned schedules do, so its values at a budget of $2K$ products equal its values at horizon $K$.}
    \label{tab:logreg-losses}
    \footnotesize
    \setlength{\tabcolsep}{3.5pt}
    \renewcommand{\arraystretch}{0.88}
    \begin{fittable}
    \begin{filecontents*}{logreg_losses.csv}
Framework,K,IDmean,IDq10,IDmed,IDq90,OODmean,OODq10,OODmed,OODq90
L2O,1,5.01e-2,3.64e-2,4.93e-2,6.43e-2,7.98e-2,5.14e-2,7.97e-2,1.10e-1
,2,1.57e-2,9.03e-3,1.50e-2,2.37e-2,2.24e-2,1.25e-2,2.00e-2,3.37e-2
,3,7.09e-3,3.15e-3,6.34e-3,1.23e-2,2.12e-2,9.98e-3,1.58e-2,3.93e-2
,4,4.10e-3,1.67e-3,3.54e-3,7.62e-3,2.22e-2,5.56e-3,1.10e-2,5.24e-2
,5,2.50e-3,8.34e-4,2.10e-3,4.93e-3,4.87e-2,3.04e-3,1.57e-2,1.47e-1
,6,1.50e-3,4.98e-4,1.18e-3,2.93e-3,3.12e-2,1.93e-3,2.21e-2,7.12e-2
,7,9.45e-4,3.19e-4,7.61e-4,1.81e-3,4.32e-2,1.70e-3,2.88e-2,1.08e-1
,8,6.08e-4,1.84e-4,4.41e-4,1.07e-3,3.88e-2,1.72e-3,2.13e-2,1.23e-1
,9,4.06e-4,1.16e-4,2.75e-4,7.22e-4,8.30e-2,1.02e-3,7.00e-2,1.86e-1
,10,2.95e-4,8.01e-5,1.87e-4,5.30e-4,4.82e-2,1.02e-3,4.14e-2,1.08e-1
,11,2.26e-4,4.23e-5,1.19e-4,4.23e-4,2.93e-2,3.05e-3,2.96e-2,5.20e-2
,12,1.74e-4,2.57e-5,8.04e-5,3.45e-4,1.51e-2,1.55e-3,1.36e-2,2.90e-2
,13,1.39e-4,1.60e-5,5.51e-5,2.89e-4,2.31e-2,1.97e-3,1.92e-2,4.34e-2
,14,1.16e-4,1.16e-5,3.57e-5,2.43e-4,2.09e-2,1.94e-3,1.67e-2,4.42e-2
,15,1.00e-4,8.98e-6,2.90e-5,2.16e-4,2.84e-2,1.74e-3,1.83e-2,7.49e-2
DR-L2O,1,8.97e-2,7.18e-2,8.74e-2,1.11e-1,5.73e-2,4.35e-2,5.68e-2,7.10e-2
,2,5.66e-2,4.33e-2,5.50e-2,7.34e-2,3.14e-2,2.09e-2,3.05e-2,4.22e-2
,3,2.27e-2,1.50e-2,2.12e-2,3.26e-2,9.48e-3,4.57e-3,8.89e-3,1.49e-2
,4,1.21e-2,6.44e-3,1.07e-2,1.99e-2,4.95e-3,1.75e-3,4.43e-3,8.29e-3
,5,6.09e-3,2.31e-3,5.06e-3,1.16e-2,2.68e-3,9.49e-4,2.33e-3,5.11e-3
,6,3.98e-3,1.29e-3,3.25e-3,7.87e-3,1.76e-3,6.24e-4,1.44e-3,3.37e-3
,7,2.77e-3,9.32e-4,2.28e-3,5.47e-3,1.21e-3,5.05e-4,1.02e-3,2.17e-3
,8,2.09e-3,7.82e-4,1.65e-3,4.01e-3,8.43e-4,3.59e-4,7.11e-4,1.45e-3
,9,1.58e-3,6.95e-4,1.26e-3,2.91e-3,5.71e-4,2.33e-4,4.89e-4,9.61e-4
,10,1.17e-3,5.34e-4,9.28e-4,2.09e-3,3.75e-4,1.28e-4,3.27e-4,6.44e-4
,11,1.00e-3,3.66e-4,7.46e-4,1.93e-3,3.04e-4,7.86e-5,2.33e-4,5.64e-4
,12,5.91e-4,2.21e-4,4.51e-4,1.02e-3,1.97e-4,5.33e-5,1.47e-4,3.60e-4
,13,4.32e-4,1.38e-4,3.12e-4,7.47e-4,1.64e-4,4.42e-5,1.09e-4,3.09e-4
,14,4.36e-4,1.84e-4,3.78e-4,7.31e-4,2.29e-4,8.45e-5,1.91e-4,3.94e-4
,15,3.33e-4,5.44e-5,1.65e-4,6.90e-4,1.10e-4,1.43e-5,3.83e-5,2.58e-4
OPT-PEP,1,8.85e-2,7.08e-2,8.62e-2,1.09e-1,5.66e-2,4.29e-2,5.61e-2,7.03e-2
,2,8.02e-2,6.42e-2,7.73e-2,1.00e-1,4.84e-2,3.65e-2,4.75e-2,6.11e-2
,3,5.05e-2,3.85e-2,4.84e-2,6.63e-2,2.62e-2,1.75e-2,2.52e-2,3.59e-2
,4,2.39e-2,1.60e-2,2.23e-2,3.42e-2,1.11e-2,5.91e-3,1.06e-2,1.70e-2
,5,1.89e-2,1.20e-2,1.75e-2,2.80e-2,8.27e-3,3.95e-3,7.79e-3,1.31e-2
,6,1.32e-2,7.49e-3,1.18e-2,2.10e-2,5.82e-3,2.40e-3,5.34e-3,9.64e-3
,7,8.68e-3,4.25e-3,7.58e-3,1.49e-2,3.67e-3,1.18e-3,3.17e-3,6.37e-3
,8,6.01e-3,2.51e-3,5.05e-3,1.11e-2,2.45e-3,6.55e-4,2.03e-3,4.71e-3
,9,4.09e-3,1.32e-3,3.33e-3,8.08e-3,1.61e-3,3.59e-4,1.30e-3,3.36e-3
,10,3.40e-3,9.86e-4,2.77e-3,6.89e-3,1.30e-3,2.46e-4,1.03e-3,2.72e-3
,11,2.17e-3,5.11e-4,1.72e-3,4.67e-3,7.93e-4,1.44e-4,5.96e-4,1.68e-3
,12,1.90e-3,3.99e-4,1.49e-3,4.18e-3,6.67e-4,9.60e-5,4.76e-4,1.44e-3
,13,1.26e-3,2.69e-4,9.31e-4,2.78e-3,4.17e-4,8.47e-5,2.64e-4,8.93e-4
,14,9.75e-4,2.19e-4,6.92e-4,2.08e-3,3.09e-4,6.26e-5,1.71e-4,6.70e-4
,15,8.23e-4,1.51e-4,5.43e-4,1.80e-3,2.50e-4,3.60e-5,1.12e-4,5.50e-4
Backtracking,1,1.45e-1,1.21e-1,1.43e-1,1.75e-1,1.06e-1,8.71e-2,1.06e-1,1.25e-1
,2,9.81e-2,8.01e-2,9.52e-2,1.20e-1,6.40e-2,4.95e-2,6.34e-2,7.89e-2
,3,6.53e-2,5.13e-2,6.27e-2,8.33e-2,3.84e-2,2.82e-2,3.75e-2,4.94e-2
,4,4.33e-2,3.23e-2,4.13e-2,5.78e-2,2.32e-2,1.52e-2,2.22e-2,3.23e-2
,5,2.86e-2,2.01e-2,2.68e-2,3.99e-2,1.42e-2,8.12e-3,1.35e-2,2.09e-2
,6,1.91e-2,1.21e-2,1.77e-2,2.83e-2,8.82e-3,4.34e-3,8.18e-3,1.37e-2
,7,1.28e-2,7.21e-3,1.15e-2,2.06e-2,5.64e-3,2.27e-3,5.17e-3,9.35e-3
,8,8.80e-3,4.32e-3,7.68e-3,1.51e-2,3.72e-3,1.20e-3,3.22e-3,6.44e-3
,9,6.17e-3,2.58e-3,5.20e-3,1.14e-2,2.53e-3,6.99e-4,2.11e-3,4.82e-3
,10,4.44e-3,1.48e-3,3.62e-3,8.69e-3,1.78e-3,4.30e-4,1.45e-3,3.70e-3
,11,3.28e-3,9.25e-4,2.70e-3,6.67e-3,1.28e-3,2.87e-4,1.02e-3,2.69e-3
,12,2.49e-3,6.32e-4,2.00e-3,5.22e-3,9.45e-4,1.87e-4,7.32e-4,1.97e-3
,13,1.93e-3,4.51e-4,1.53e-3,4.14e-3,7.04e-4,1.59e-4,5.29e-4,1.42e-3
,14,1.52e-3,3.49e-4,1.16e-3,3.36e-3,5.30e-4,1.19e-4,3.62e-4,1.09e-3
,15,1.21e-3,2.88e-4,8.89e-4,2.60e-3,4.02e-4,9.68e-5,2.53e-4,8.52e-4
\end{filecontents*}
\pgfplotstabletypeset[
      col sep=comma,
      string type,
      columns/Framework/.style={column name=Framework, column type=l},
      columns/K/.style={column name={$K$}, column type=l},
      columns/IDmean/.style={column name=Mean, column type=r},
      columns/IDq10/.style={column name={$10^\text{th}$}, column type=r},
      columns/IDmed/.style={column name=Median, column type=r},
      columns/IDq90/.style={column name={$90^\text{th}$}, column type=r},
      columns/OODmean/.style={column name=Mean, column type=r},
      columns/OODq10/.style={column name={$10^\text{th}$}, column type=r},
      columns/OODmed/.style={column name=Median, column type=r},
      columns/OODq90/.style={column name={$90^\text{th}$}, column type=r},
      every head row/.style={
        before row={\toprule & & \multicolumn{4}{c}{In-distribution} & \multicolumn{4}{c}{Out-of-distribution}\\ \cmidrule(lr){3-6}\cmidrule(lr){7-10}},
        after row=\midrule},
      every last row/.style={after row=\bottomrule},
      every row no 15/.style={before row=\midrule},
      every row no 30/.style={before row=\midrule},
      every row no 45/.style={before row=\midrule},
    ]{logreg_losses.csv}
    \end{fittable}
\end{table}

\begin{table}[htbp]
    \centering
    \caption{Fractions of logistic regression problems solved plotted in Figure~\ref{fig:logreg_frac_solved}, at each relative tolerance $\eta$, on the in-distribution and out-of-distribution test sets for every horizon $K$.}
    \label{tab:logreg-frac-solved}
    \small
    \begin{fittable}
    \begin{filecontents*}{logreg_frac_solved.csv}
Framework,K,IDeta1,IDeta2,IDeta3,OODeta1,OODeta2,OODeta3
L2O,1,0.000,0.000,0.000,0.000,0.000,0.000
,2,0.000,0.000,0.472,0.000,0.000,0.196
,3,0.000,0.004,0.948,0.000,0.000,0.428
,4,0.000,0.076,0.996,0.000,0.000,0.584
,5,0.000,0.300,1.000,0.000,0.000,0.476
,6,0.000,0.604,1.000,0.000,0.072,0.420
,7,0.000,0.832,1.000,0.000,0.088,0.364
,8,0.032,0.936,1.000,0.000,0.084,0.376
,9,0.180,0.972,1.000,0.012,0.116,0.272
,10,0.352,0.972,1.000,0.004,0.116,0.248
,11,0.556,0.980,1.000,0.016,0.080,0.248
,12,0.648,0.992,1.000,0.032,0.096,0.540
,13,0.688,0.996,1.000,0.048,0.088,0.352
,14,0.736,0.996,1.000,0.048,0.084,0.420
,15,0.744,1.000,1.000,0.052,0.084,0.384
DR-L2O,1,0.000,0.000,0.000,0.000,0.000,0.000
,2,0.000,0.000,0.000,0.000,0.000,0.008
,3,0.000,0.000,0.092,0.000,0.000,0.888
,4,0.000,0.000,0.736,0.000,0.056,0.988
,5,0.000,0.036,0.968,0.000,0.292,1.000
,6,0.000,0.136,0.992,0.000,0.492,1.000
,7,0.000,0.284,0.996,0.000,0.728,1.000
,8,0.000,0.412,1.000,0.012,0.896,1.000
,9,0.000,0.584,1.000,0.024,0.968,1.000
,10,0.000,0.772,1.000,0.132,0.984,1.000
,11,0.000,0.816,1.000,0.284,0.984,1.000
,12,0.020,0.948,1.000,0.488,0.996,1.000
,13,0.104,0.968,1.000,0.600,1.000,1.000
,14,0.048,0.976,1.000,0.316,1.000,1.000
,15,0.448,0.972,1.000,0.704,1.000,1.000
OPT-PEP,1,0.000,0.000,0.000,0.000,0.000,0.000
,2,0.000,0.000,0.000,0.000,0.000,0.000
,3,0.000,0.000,0.000,0.000,0.000,0.036
,4,0.000,0.000,0.072,0.000,0.000,0.780
,5,0.000,0.000,0.288,0.000,0.004,0.924
,6,0.000,0.000,0.672,0.000,0.024,0.976
,7,0.000,0.000,0.864,0.000,0.136,1.000
,8,0.000,0.036,0.980,0.000,0.336,1.000
,9,0.000,0.136,0.992,0.008,0.544,1.000
,10,0.000,0.208,0.996,0.036,0.644,1.000
,11,0.000,0.444,1.000,0.120,0.856,1.000
,12,0.008,0.492,1.000,0.184,0.896,1.000
,13,0.016,0.684,1.000,0.316,0.976,1.000
,14,0.008,0.800,1.000,0.448,0.984,1.000
,15,0.096,0.832,1.000,0.552,0.984,1.000
\end{filecontents*}
\pgfplotstabletypeset[
      col sep=comma,
      string type,
      columns/Framework/.style={column name=Framework, column type=l},
      columns/K/.style={column name={$K$}, column type=l},
      columns/IDeta1/.style={column name={$\eta = 10^{-4}$}, column type=r},
      columns/IDeta2/.style={column name={$\eta = 10^{-3}$}, column type=r},
      columns/IDeta3/.style={column name={$\eta = 10^{-2}$}, column type=r},
      columns/OODeta1/.style={column name={$\eta = 10^{-4}$}, column type=r},
      columns/OODeta2/.style={column name={$\eta = 10^{-3}$}, column type=r},
      columns/OODeta3/.style={column name={$\eta = 10^{-2}$}, column type=r},
      every head row/.style={
        before row={\toprule & & \multicolumn{3}{c}{In-distribution} & \multicolumn{3}{c}{Out-of-distribution}\\ \cmidrule(lr){3-5}\cmidrule(lr){6-8}},
        after row=\midrule},
      every last row/.style={after row=\bottomrule},
      every row no 15/.style={before row=\midrule},
      every row no 30/.style={before row=\midrule},
    ]{logreg_frac_solved.csv}
    \end{fittable}
\end{table}

\begin{table}[htbp]
    \centering
    \caption{LASSO experiment test losses plotted in Figure~\ref{fig:lasso_losses}: mean, $10^\text{th}$ quantile, median, and $90^\text{th}$ quantile of $f(x^K) - f(x^\star)$ over the in-distribution and out-of-distribution test sets for every horizon $K$.}
    \label{tab:lasso-losses}
    \footnotesize
    \setlength{\tabcolsep}{3.5pt}
    \renewcommand{\arraystretch}{0.88}
    \begin{fittable}
    \begin{filecontents*}{lasso_losses.csv}
Framework,K,IDmean,IDq10,IDmed,IDq90,OODmean,OODq10,OODmed,OODq90
L2O-ISTA,1,2.17e1,1.60e1,2.10e1,2.85e1,4.87e1,3.36e1,4.91e1,6.37e1
,2,1.03e1,7.23e0,9.73e0,1.44e1,2.85e1,1.73e1,2.79e1,4.00e1
,3,6.49e0,4.23e0,6.08e0,9.29e0,2.27e1,1.14e1,2.02e1,3.59e1
,4,4.26e0,2.43e0,3.92e0,6.39e0,2.43e1,7.90e0,1.77e1,4.84e1
,5,2.89e0,1.45e0,2.57e0,4.45e0,3.37e1,5.58e0,1.62e1,7.55e1
,6,1.92e0,8.53e-1,1.74e0,3.10e0,5.51e1,3.96e0,1.25e1,1.33e2
,7,1.30e0,5.12e-1,1.17e0,2.19e0,1.63e2,2.81e0,1.10e1,3.53e2
,8,9.35e-1,3.51e-1,8.80e-1,1.63e0,1.97e2,2.08e0,6.04e0,3.09e2
,9,6.35e-1,1.93e-1,5.86e-1,1.20e0,8.91e2,1.50e0,4.79e0,1.07e3
,10,4.26e-1,1.07e-1,3.82e-1,8.55e-1,5.44e3,9.98e-1,3.50e0,5.86e3
,11,3.04e-1,6.35e-2,2.62e-1,6.07e-1,1.65e4,7.13e-1,2.46e0,1.35e4
,12,2.02e-1,3.54e-2,1.58e-1,4.35e-1,1.21e5,4.24e-1,1.83e0,9.50e4
,13,1.55e-1,2.24e-2,1.16e-1,3.49e-1,3.06e5,3.25e-1,1.38e0,2.08e5
,14,1.12e-1,1.28e-2,7.99e-2,2.67e-1,7.20e5,2.11e-1,9.84e-1,2.93e5
,15,7.85e-2,7.36e-3,5.18e-2,1.91e-1,6.04e6,1.40e-1,7.82e-1,4.03e6
L2O-ALISTA,1,2.20e1,1.61e1,2.11e1,2.94e1,5.04e1,3.37e1,5.09e1,6.63e1
,2,1.33e1,9.83e0,1.29e1,1.75e1,2.98e1,2.11e1,3.00e1,3.91e1
,3,9.73e0,7.08e0,9.49e0,1.29e1,2.20e1,1.58e1,2.22e1,2.84e1
,4,7.62e0,5.35e0,7.39e0,1.02e1,1.79e1,1.28e1,1.82e1,2.31e1
,5,6.11e0,4.16e0,5.96e0,8.27e0,1.51e1,1.05e1,1.52e1,1.97e1
,6,4.96e0,3.22e0,4.80e0,6.81e0,1.29e1,8.82e0,1.29e1,1.70e1
,7,4.06e0,2.51e0,3.92e0,5.57e0,1.11e1,7.30e0,1.11e1,1.50e1
,8,3.33e0,1.97e0,3.22e0,4.64e0,9.69e0,6.19e0,9.59e0,1.32e1
,9,2.75e0,1.55e0,2.65e0,3.93e0,8.46e0,5.28e0,8.33e0,1.17e1
,10,2.27e0,1.22e0,2.19e0,3.37e0,7.41e0,4.55e0,7.25e0,1.05e1
,11,1.88e0,9.42e-1,1.81e0,2.85e0,6.50e0,3.79e0,6.32e0,9.26e0
,12,1.56e0,7.33e-1,1.51e0,2.44e0,5.71e0,3.15e0,5.54e0,8.30e0
,13,1.29e0,5.68e-1,1.25e0,2.10e0,5.01e0,2.74e0,4.88e0,7.37e0
,14,1.07e0,4.47e-1,1.03e0,1.81e0,4.41e0,2.31e0,4.25e0,6.58e0
,15,8.84e-1,3.46e-1,8.51e-1,1.54e0,3.87e0,1.98e0,3.71e0,5.94e0
DR-L2O,1,2.20e1,1.63e1,2.13e1,2.88e1,4.89e1,3.43e1,4.90e1,6.44e1
,2,1.47e1,1.10e1,1.44e1,1.90e1,3.17e1,2.30e1,3.19e1,4.10e1
,3,9.96e0,7.03e0,9.49e0,1.35e1,2.55e1,1.67e1,2.53e1,3.49e1
,4,7.94e0,5.61e0,7.72e0,1.06e1,1.85e1,1.32e1,1.88e1,2.38e1
,5,7.40e0,5.18e0,7.16e0,9.90e0,1.74e1,1.24e1,1.77e1,2.24e1
,6,5.39e0,3.49e0,5.09e0,7.59e0,1.53e1,9.66e0,1.53e1,2.12e1
,7,2.73e0,1.46e0,2.58e0,4.01e0,9.58e0,5.31e0,9.32e0,1.42e1
,8,1.88e0,9.18e-1,1.79e0,2.87e0,6.58e0,3.80e0,6.44e0,9.40e0
,9,2.09e0,1.08e0,1.99e0,3.13e0,7.06e0,4.24e0,6.92e0,1.01e1
,10,1.65e0,7.79e-1,1.58e0,2.61e0,6.34e0,3.48e0,6.15e0,9.52e0
,11,9.03e-1,3.50e-1,8.69e-1,1.59e0,4.06e0,2.04e0,3.90e0,6.33e0
,12,2.06e-1,3.35e-2,1.60e-1,4.40e-1,1.54e0,4.30e-1,1.42e0,2.88e0
,13,9.33e-2,7.82e-3,6.05e-2,2.30e-1,8.75e-1,1.70e-1,7.39e-1,1.74e0
,14,1.96e-1,2.96e-2,1.51e-1,4.31e-1,1.46e0,4.23e-1,1.33e0,2.70e0
,15,8.89e-2,7.44e-3,5.77e-2,2.20e-1,8.22e-1,1.62e-1,7.05e-1,1.62e0
OPT-PEP,1,2.17e1,1.60e1,2.10e1,2.85e1,4.87e1,3.36e1,4.91e1,6.37e1
,2,1.40e1,1.05e1,1.37e1,1.82e1,3.03e1,2.19e1,3.06e1,3.90e1
,3,1.09e1,8.04e0,1.07e1,1.43e1,2.41e1,1.75e1,2.42e1,3.09e1
,4,8.12e0,5.76e0,7.92e0,1.08e1,1.88e1,1.35e1,1.91e1,2.42e1
,5,6.74e0,4.68e0,6.53e0,9.09e0,1.62e1,1.14e1,1.64e1,2.09e1
,6,5.95e0,4.03e0,5.80e0,8.07e0,1.48e1,1.03e1,1.49e1,1.93e1
,7,4.79e0,3.09e0,4.65e0,6.60e0,1.26e1,8.53e0,1.26e1,1.66e1
,8,3.89e0,2.38e0,3.76e0,5.33e0,1.08e1,7.02e0,1.08e1,1.45e1
,9,3.68e0,2.23e0,3.56e0,5.04e0,1.04e1,6.70e0,1.03e1,1.40e1
,10,3.07e0,1.78e0,2.96e0,4.33e0,9.14e0,5.77e0,9.04e0,1.26e1
,11,2.69e0,1.51e0,2.59e0,3.85e0,8.31e0,5.19e0,8.17e0,1.16e1
,12,2.26e0,1.21e0,2.18e0,3.35e0,7.38e0,4.53e0,7.21e0,1.04e1
,13,1.74e0,8.49e-1,1.67e0,2.68e0,6.15e0,3.49e0,5.96e0,8.82e0
,14,1.68e0,8.15e-1,1.62e0,2.60e0,6.00e0,3.39e0,5.80e0,8.63e0
,15,1.31e0,5.82e-1,1.27e0,2.12e0,5.06e0,2.76e0,4.94e0,7.43e0
\end{filecontents*}
\pgfplotstabletypeset[
      col sep=comma,
      string type,
      columns/Framework/.style={column name=Framework, column type=l},
      columns/K/.style={column name={$K$}, column type=l},
      columns/IDmean/.style={column name=Mean, column type=r},
      columns/IDq10/.style={column name={$10^\text{th}$}, column type=r},
      columns/IDmed/.style={column name=Median, column type=r},
      columns/IDq90/.style={column name={$90^\text{th}$}, column type=r},
      columns/OODmean/.style={column name=Mean, column type=r},
      columns/OODq10/.style={column name={$10^\text{th}$}, column type=r},
      columns/OODmed/.style={column name=Median, column type=r},
      columns/OODq90/.style={column name={$90^\text{th}$}, column type=r},
      every head row/.style={
        before row={\toprule & & \multicolumn{4}{c}{In-distribution} & \multicolumn{4}{c}{Out-of-distribution}\\ \cmidrule(lr){3-6}\cmidrule(lr){7-10}},
        after row=\midrule},
      every last row/.style={after row=\bottomrule},
      every row no 15/.style={before row=\midrule},
      every row no 30/.style={before row=\midrule},
      every row no 45/.style={before row=\midrule},
    ]{lasso_losses.csv}
    \end{fittable}
\end{table}

\begin{table}[htbp]
    \centering
    \caption{LASSO experiment test losses of backtracking line search plotted in Figure~\ref{fig:lasso_backtracking}: mean, $10^\text{th}$ quantile, median, and $90^\text{th}$ quantile of $f(x) - f(x^\star)$ over the in-distribution and out-of-distribution test sets. \emph{Iterations:} the iterate $x^K$ after $K$ iterations. \emph{Products:} the last iterate whose cumulative number of matrix-vector products is within the budget, or $x^0$ if no iteration fits. The learned schedules use two products per iteration, so their values at a budget of $2K$ products are those at horizon $K$ in Table~\ref{tab:lasso-losses}.}
    \label{tab:lasso-backtracking}
    \footnotesize
    \setlength{\tabcolsep}{3.5pt}
    \begin{fittable}
    \begin{filecontents*}{lasso_backtracking.csv}
Budget,x,IDmean,IDq10,IDmed,IDq90,OODmean,OODq10,OODmed,OODq90
Iterations,1,2.29e1,1.69e1,2.23e1,2.99e1,5.07e1,3.54e1,5.09e1,6.65e1
,2,1.15e1,8.52e0,1.13e1,1.51e1,2.61e1,1.84e1,2.62e1,3.44e1
,3,6.12e0,3.84e0,5.63e0,9.20e0,1.85e1,1.07e1,1.87e1,2.60e1
,4,3.90e0,1.84e0,3.68e0,6.26e0,1.31e1,7.72e0,1.33e1,1.82e1
,5,2.43e0,9.43e-1,2.23e0,4.15e0,9.68e0,5.00e0,9.92e0,1.42e1
,6,1.46e0,4.67e-1,1.33e0,2.59e0,7.07e0,3.23e0,7.09e0,1.09e1
,7,8.90e-1,2.09e-1,7.63e-1,1.76e0,5.26e0,2.00e0,5.24e0,8.27e0
,8,5.34e-1,8.01e-2,4.27e-1,1.16e0,3.79e0,1.25e0,3.80e0,6.08e0
,9,3.18e-1,2.74e-2,2.27e-1,7.96e-1,2.76e0,7.52e-1,2.65e0,4.78e0
,10,1.86e-1,1.11e-2,1.11e-1,4.88e-1,1.97e0,4.25e-1,1.82e0,3.52e0
,11,1.10e-1,4.06e-3,5.84e-2,2.90e-1,1.41e0,1.94e-1,1.24e0,2.65e0
,12,6.45e-2,1.16e-3,2.78e-2,1.85e-1,9.88e-1,8.48e-2,8.09e-1,1.97e0
,13,3.75e-2,4.30e-4,1.46e-2,1.17e-1,6.88e-1,4.53e-2,5.14e-1,1.45e0
,14,2.26e-2,1.80e-4,7.14e-3,7.18e-2,4.77e-1,1.96e-2,3.09e-1,1.08e0
,15,1.33e-2,5.76e-5,3.38e-3,3.87e-2,3.31e-1,7.98e-3,1.84e-1,8.43e-1
Products,2,7.18e1,4.82e1,6.94e1,9.97e1,1.80e2,1.13e2,1.83e2,2.43e2
,4,2.29e1,1.69e1,2.23e1,2.99e1,5.07e1,3.54e1,5.09e1,6.65e1
,6,2.29e1,1.69e1,2.23e1,2.99e1,5.07e1,3.54e1,5.09e1,6.65e1
,8,1.15e1,8.52e0,1.13e1,1.51e1,2.61e1,1.84e1,2.62e1,3.44e1
,10,6.12e0,3.84e0,5.63e0,9.20e0,2.03e1,1.07e1,1.91e1,3.20e1
,12,4.19e0,1.84e0,3.76e0,6.78e0,1.74e1,8.14e0,1.79e1,2.60e1
,14,2.93e0,9.43e-1,2.38e0,6.12e0,1.28e1,5.72e0,1.33e1,1.82e1
,16,2.04e0,4.67e-1,1.66e0,4.15e0,9.74e0,4.76e0,9.94e0,1.45e1
,18,1.33e0,2.09e-1,1.13e0,2.59e0,7.73e0,3.23e0,7.42e0,1.25e1
,20,8.84e-1,8.35e-2,7.43e-1,1.79e0,6.32e0,2.00e0,6.41e0,1.06e1
,22,5.83e-1,3.00e-2,4.27e-1,1.33e0,4.93e0,1.25e0,4.99e0,8.27e0
,24,3.70e-1,1.23e-2,2.30e-1,9.85e-1,3.70e0,7.52e-1,3.80e0,6.30e0
,26,2.40e-1,4.52e-3,1.26e-1,6.79e-1,2.87e0,4.90e-1,2.67e0,5.07e0
,28,1.54e-1,2.04e-3,6.32e-2,4.83e-1,2.26e0,2.55e-1,1.87e0,4.44e0
,30,9.88e-2,9.65e-4,3.50e-2,2.86e-1,1.77e0,1.30e-1,1.46e0,3.50e0
\end{filecontents*}
\pgfplotstabletypeset[
      col sep=comma,
      string type,
      columns/Budget/.style={column name=Cost, column type=l},
      columns/x/.style={column name=Budget, column type=l},
      columns/IDmean/.style={column name=Mean, column type=r},
      columns/IDq10/.style={column name={$10^\text{th}$}, column type=r},
      columns/IDmed/.style={column name=Median, column type=r},
      columns/IDq90/.style={column name={$90^\text{th}$}, column type=r},
      columns/OODmean/.style={column name=Mean, column type=r},
      columns/OODq10/.style={column name={$10^\text{th}$}, column type=r},
      columns/OODmed/.style={column name=Median, column type=r},
      columns/OODq90/.style={column name={$90^\text{th}$}, column type=r},
      every head row/.style={
        before row={\toprule & & \multicolumn{4}{c}{In-distribution} & \multicolumn{4}{c}{Out-of-distribution}\\ \cmidrule(lr){3-6}\cmidrule(lr){7-10}},
        after row=\midrule},
      every last row/.style={after row=\bottomrule},
      every row no 15/.style={before row=\midrule},
    ]{lasso_backtracking.csv}
    \end{fittable}
\end{table}

\begin{table}[htbp]
    \centering
    \caption{Fractions of LASSO problems solved plotted in Figure~\ref{fig:lasso_frac_solved}, at each relative tolerance $\eta$, on the in-distribution and out-of-distribution test sets for every horizon $K$.}
    \label{tab:lasso-frac-solved}
    \small
    \renewcommand{\arraystretch}{0.88}
    \begin{fittable}
    \begin{filecontents*}{lasso_frac_solved.csv}
Framework,K,IDeta1,IDeta2,IDeta3,OODeta1,OODeta2,OODeta3
L2O-ISTA,1,0.000,0.000,0.000,0.000,0.000,0.000
,2,0.000,0.000,0.000,0.000,0.000,0.000
,3,0.000,0.000,0.000,0.000,0.000,0.000
,4,0.000,0.000,0.073,0.000,0.000,0.000
,5,0.000,0.028,0.589,0.000,0.000,0.004
,6,0.000,0.286,0.911,0.000,0.004,0.056
,7,0.004,0.690,0.980,0.000,0.012,0.200
,8,0.036,0.899,1.000,0.004,0.044,0.380
,9,0.137,0.996,1.000,0.004,0.168,0.472
,10,0.351,1.000,1.000,0.008,0.352,0.528
,11,0.540,1.000,1.000,0.012,0.460,0.600
,12,0.790,1.000,1.000,0.068,0.532,0.600
,13,0.883,1.000,1.000,0.104,0.600,0.648
,14,0.960,1.000,1.000,0.160,0.680,0.696
,15,0.996,1.000,1.000,0.300,0.640,0.652
L2O-ALISTA,1,0.000,0.000,0.000,0.000,0.000,0.000
,2,0.000,0.000,0.000,0.000,0.000,0.000
,3,0.000,0.000,0.000,0.000,0.000,0.000
,4,0.000,0.000,0.000,0.000,0.000,0.000
,5,0.000,0.000,0.000,0.000,0.000,0.000
,6,0.000,0.000,0.004,0.000,0.000,0.000
,7,0.000,0.000,0.056,0.000,0.000,0.000
,8,0.000,0.000,0.246,0.000,0.000,0.000
,9,0.000,0.016,0.621,0.000,0.000,0.004
,10,0.000,0.077,0.855,0.000,0.000,0.020
,11,0.000,0.210,0.968,0.000,0.004,0.064
,12,0.000,0.440,0.996,0.000,0.004,0.180
,13,0.000,0.653,1.000,0.000,0.012,0.340
,14,0.008,0.843,1.000,0.000,0.028,0.548
,15,0.036,0.940,1.000,0.000,0.060,0.720
DR-L2O,1,0.000,0.000,0.000,0.000,0.000,0.000
,2,0.000,0.000,0.000,0.000,0.000,0.000
,3,0.000,0.000,0.000,0.000,0.000,0.000
,4,0.000,0.000,0.000,0.000,0.000,0.000
,5,0.000,0.000,0.000,0.000,0.000,0.000
,6,0.000,0.000,0.008,0.000,0.000,0.000
,7,0.000,0.024,0.621,0.000,0.000,0.004
,8,0.000,0.222,0.968,0.000,0.004,0.064
,9,0.000,0.137,0.911,0.000,0.000,0.032
,10,0.000,0.359,0.984,0.000,0.004,0.120
,11,0.036,0.935,1.000,0.000,0.052,0.684
,12,0.794,1.000,1.000,0.060,0.864,1.000
,13,0.980,1.000,1.000,0.252,0.988,1.000
,14,0.798,1.000,1.000,0.072,0.884,1.000
,15,0.992,1.000,1.000,0.268,0.988,1.000
OPT-PEP,1,0.000,0.000,0.000,0.000,0.000,0.000
,2,0.000,0.000,0.000,0.000,0.000,0.000
,3,0.000,0.000,0.000,0.000,0.000,0.000
,4,0.000,0.000,0.000,0.000,0.000,0.000
,5,0.000,0.000,0.000,0.000,0.000,0.000
,6,0.000,0.000,0.000,0.000,0.000,0.000
,7,0.000,0.000,0.012,0.000,0.000,0.000
,8,0.000,0.000,0.081,0.000,0.000,0.000
,9,0.000,0.000,0.137,0.000,0.000,0.000
,10,0.000,0.004,0.411,0.000,0.000,0.000
,11,0.000,0.016,0.645,0.000,0.000,0.004
,12,0.000,0.077,0.855,0.000,0.000,0.020
,13,0.000,0.302,0.984,0.000,0.004,0.104
,14,0.000,0.331,0.984,0.000,0.004,0.128
,15,0.000,0.641,1.000,0.000,0.012,0.324
\end{filecontents*}
\pgfplotstabletypeset[
      col sep=comma,
      string type,
      columns/Framework/.style={column name=Framework, column type=l},
      columns/K/.style={column name={$K$}, column type=l},
      columns/IDeta1/.style={column name={$\eta = 10^{-2}$}, column type=r},
      columns/IDeta2/.style={column name={$\eta = 5 \cdot 10^{-2}$}, column type=r},
      columns/IDeta3/.style={column name={$\eta = 10^{-1}$}, column type=r},
      columns/OODeta1/.style={column name={$\eta = 10^{-2}$}, column type=r},
      columns/OODeta2/.style={column name={$\eta = 5 \cdot 10^{-2}$}, column type=r},
      columns/OODeta3/.style={column name={$\eta = 10^{-1}$}, column type=r},
      every head row/.style={
        before row={\toprule & & \multicolumn{3}{c}{In-distribution} & \multicolumn{3}{c}{Out-of-distribution}\\ \cmidrule(lr){3-5}\cmidrule(lr){6-8}},
        after row=\midrule},
      every last row/.style={after row=\bottomrule},
      every row no 15/.style={before row=\midrule},
      every row no 30/.style={before row=\midrule},
      every row no 45/.style={before row=\midrule},
    ]{lasso_frac_solved.csv}
    \end{fittable}
\end{table}

\begin{table}[htbp]
    \centering
    \caption{TV inpainting experiment test losses plotted in Figure~\ref{fig:pdlp_losses}: mean, $10^\text{th}$ quantile, median, and $90^\text{th}$ quantile of the Lagrangian duality gap over the Olivetti in-distribution and Tiny ImageNet out-of-distribution test sets for every horizon $K$.}
    \label{tab:pdlp-losses}
    \footnotesize
    \setlength{\tabcolsep}{3.5pt}
    \begin{fittable}
    \begin{filecontents*}{pdlp_losses.csv}
Framework,K,IDmean,IDq10,IDmed,IDq90,OODmean,OODq10,OODmed,OODq90
L2O,1,8.93e1,8.07e1,8.66e1,9.92e1,4.47e2,3.14e2,4.23e2,5.84e2
,2,7.05e1,6.29e1,6.86e1,7.86e1,3.54e2,2.38e2,3.30e2,4.78e2
,3,4.00e1,3.55e1,3.90e1,4.47e1,2.17e2,1.34e2,1.97e2,3.20e2
,4,1.57e1,1.29e1,1.57e1,1.84e1,1.06e2,4.53e1,8.53e1,1.97e2
,5,6.00e0,5.06e0,5.96e0,7.33e0,7.33e1,2.91e1,5.66e1,1.43e2
,6,4.42e0,3.73e0,4.40e0,5.39e0,7.22e1,2.42e1,5.15e1,1.51e2
,7,5.62e0,4.79e0,5.62e0,6.72e0,7.37e1,2.81e1,5.56e1,1.48e2
,8,3.33e0,2.94e0,3.37e0,3.66e0,5.87e1,1.81e1,3.95e1,1.29e2
,9,3.99e0,3.51e0,4.02e0,4.37e0,5.80e1,2.00e1,4.15e1,1.22e2
,10,8.25e0,7.06e0,8.31e0,9.52e0,7.49e1,3.49e1,6.19e1,1.33e2
DR-L2O,1,9.93e1,9.21e1,9.74e1,1.08e2,5.17e2,3.66e2,4.96e2,6.72e2
,2,6.93e1,5.94e1,6.81e1,7.94e1,3.29e2,2.02e2,2.85e2,4.47e2
,3,4.49e1,3.41e1,4.51e1,5.54e1,1.95e2,7.18e1,1.49e2,3.23e2
,4,1.96e1,1.66e1,1.95e1,2.26e1,1.22e2,5.79e1,1.00e2,2.14e2
,5,8.11e0,6.72e0,8.02e0,1.01e1,9.08e1,3.87e1,7.16e1,1.67e2
,6,6.56e0,5.61e0,6.56e0,7.70e0,7.45e1,3.11e1,5.84e1,1.45e2
,7,5.24e0,4.59e0,5.28e0,5.80e0,6.33e1,2.48e1,4.81e1,1.27e2
,8,4.21e0,3.76e0,4.24e0,4.61e0,5.33e1,2.15e1,4.08e1,1.06e2
,9,3.37e0,2.99e0,3.41e0,3.70e0,4.47e1,1.92e1,3.48e1,8.63e1
,10,3.03e0,2.65e0,3.06e0,3.35e0,3.86e1,1.71e1,3.11e1,7.11e1
OPT-PEP,1,1.10e2,1.03e2,1.08e2,1.18e2,5.88e2,4.16e2,5.45e2,7.58e2
,2,8.96e1,8.32e1,8.79e1,9.69e1,4.77e2,3.37e2,4.50e2,6.29e2
,3,6.89e1,6.23e1,6.69e1,7.64e1,3.59e2,2.49e2,3.46e2,4.89e2
,4,4.93e1,4.50e1,4.82e1,5.42e1,2.76e2,1.87e2,2.60e2,3.96e2
,5,3.34e1,2.96e1,3.26e1,3.74e1,1.90e2,1.14e2,1.71e2,2.86e2
,6,2.29e1,2.07e1,2.23e1,2.52e1,1.46e2,8.31e1,1.23e2,2.40e2
,7,1.24e1,1.13e1,1.24e1,1.37e1,1.14e2,5.91e1,8.84e1,1.93e2
,8,8.76e0,7.11e0,8.60e0,1.10e1,9.93e1,4.23e1,8.04e1,1.77e2
,9,7.80e0,6.56e0,7.74e0,9.58e0,8.73e1,3.71e1,6.88e1,1.64e2
,10,7.10e0,6.04e0,7.10e0,8.48e0,7.96e1,3.35e1,6.26e1,1.53e2
\end{filecontents*}
\pgfplotstabletypeset[
      col sep=comma,
      string type,
      columns/Framework/.style={column name=Framework, column type=l},
      columns/K/.style={column name={$K$}, column type=l},
      columns/IDmean/.style={column name=Mean, column type=r},
      columns/IDq10/.style={column name={$10^\text{th}$}, column type=r},
      columns/IDmed/.style={column name=Median, column type=r},
      columns/IDq90/.style={column name={$90^\text{th}$}, column type=r},
      columns/OODmean/.style={column name=Mean, column type=r},
      columns/OODq10/.style={column name={$10^\text{th}$}, column type=r},
      columns/OODmed/.style={column name=Median, column type=r},
      columns/OODq90/.style={column name={$90^\text{th}$}, column type=r},
      every head row/.style={
        before row={\toprule & & \multicolumn{4}{c}{In-distribution} & \multicolumn{4}{c}{Out-of-distribution}\\ \cmidrule(lr){3-6}\cmidrule(lr){7-10}},
        after row=\midrule},
      every last row/.style={after row=\bottomrule},
      every row no 10/.style={before row=\midrule},
      every row no 20/.style={before row=\midrule},
    ]{pdlp_losses.csv}
    \end{fittable}
\end{table}

\begin{table}[htbp]
    \centering
    \caption{Fractions of TV inpainting problems solved plotted in Figure~\ref{fig:pdlp_frac_solved}, at each relative tolerance $\eta$, on the Olivetti in-distribution and Tiny ImageNet out-of-distribution sets for every horizon $K$.}
    \label{tab:pdlp-frac-solved}
    \small
    \begin{fittable}
    \begin{filecontents*}{pdlp_frac_solved.csv}
Framework,K,IDeta1,IDeta2,IDeta3,OODeta1,OODeta2,OODeta3
L2O,1,0.000,0.000,0.000,0.000,0.000,0.000
,2,0.000,0.000,0.000,0.000,0.000,0.050
,3,0.000,0.000,0.025,0.000,0.000,0.475
,4,0.000,0.338,0.900,0.000,0.575,0.825
,5,0.000,1.000,1.000,0.000,0.800,0.975
,6,0.000,1.000,1.000,0.000,0.800,0.975
,7,0.000,1.000,1.000,0.000,0.800,0.975
,8,0.000,1.000,1.000,0.075,0.800,0.975
,9,0.000,1.000,1.000,0.025,0.800,0.975
,10,0.000,1.000,1.000,0.000,0.800,0.975
DR-L2O,1,0.000,0.000,0.000,0.000,0.000,0.000
,2,0.000,0.000,0.000,0.000,0.000,0.200
,3,0.000,0.000,0.087,0.000,0.300,0.500
,4,0.000,0.113,0.800,0.000,0.450,0.800
,5,0.000,1.000,1.000,0.000,0.700,0.975
,6,0.000,1.000,1.000,0.000,0.800,0.975
,7,0.000,1.000,1.000,0.000,0.800,0.975
,8,0.000,1.000,1.000,0.000,0.900,0.975
,9,0.000,1.000,1.000,0.025,0.975,1.000
,10,0.188,1.000,1.000,0.050,0.975,1.000
OPT-PEP,1,0.000,0.000,0.000,0.000,0.000,0.000
,2,0.000,0.000,0.000,0.000,0.000,0.000
,3,0.000,0.000,0.000,0.000,0.000,0.000
,4,0.000,0.000,0.000,0.000,0.000,0.300
,5,0.000,0.000,0.263,0.000,0.075,0.575
,6,0.000,0.000,0.613,0.000,0.275,0.700
,7,0.000,0.600,1.000,0.000,0.525,0.875
,8,0.000,1.000,1.000,0.000,0.650,0.925
,9,0.000,1.000,1.000,0.000,0.750,0.975
,10,0.000,1.000,1.000,0.000,0.800,0.975
\end{filecontents*}
\pgfplotstabletypeset[
      col sep=comma,
      string type,
      columns/Framework/.style={column name=Framework, column type=l},
      columns/K/.style={column name={$K$}, column type=l},
      columns/IDeta1/.style={column name={$\eta = 10^{-2}$}, column type=r},
      columns/IDeta2/.style={column name={$\eta = 5 \cdot 10^{-2}$}, column type=r},
      columns/IDeta3/.style={column name={$\eta = 10^{-1}$}, column type=r},
      columns/OODeta1/.style={column name={$\eta = 10^{-2}$}, column type=r},
      columns/OODeta2/.style={column name={$\eta = 5 \cdot 10^{-2}$}, column type=r},
      columns/OODeta3/.style={column name={$\eta = 10^{-1}$}, column type=r},
      every head row/.style={
        before row={\toprule & & \multicolumn{3}{c}{In-distribution} & \multicolumn{3}{c}{Out-of-distribution}\\ \cmidrule(lr){3-5}\cmidrule(lr){6-8}},
        after row=\midrule},
      every last row/.style={after row=\bottomrule},
      every row no 10/.style={before row=\midrule},
      every row no 20/.style={before row=\midrule},
    ]{pdlp_frac_solved.csv}
    \end{fittable}
\end{table}

\end{document}